\documentclass{article}
\usepackage{scalef nt}
\usepackage{etoolbox}
\usepackage{comment}
\newtoggle{iclr}
\togglefalse{iclr}
\newtoggle{workshop}
\togglefalse{workshop}
\newcommand{\iclr}[1]{\iftoggle{iclr}{#1}{}}
\newcommand{\arxiv}[1]{\iftoggle{iclr}{}{#1}}

\newcommand{\loose}{\looseness=-1}

\iclr{
\usepackage[dvipsnames]{xcolor} %
}

\iclr{

\usepackage{iclr2025_conference}
\usepackage{times}
}

\arxiv{
\usepackage[letterpaper, left=1in, right=1in, top=1in,bottom=1in]{geometry}
  \usepackage{parskip}
  \usepackage[colorlinks=true, linkcolor=blue!70!black, citecolor=blue!70!black,urlcolor=black,breaklinks=true]{hyperref}
  \usepackage[dvipsnames]{xcolor}
}

\iclr{
  \usepackage{parskip}
    \usepackage[colorlinks=true, linkcolor=blue!70!black, citecolor=blue!70!black,urlcolor=black,breaklinks=true]{hyperref}
}

\PassOptionsToPackage{hypertexnames=false}{hyperref}  %

\usepackage{amsmath}
\usepackage{microtype}
\usepackage{hhline}

\usepackage{amsthm}
\usepackage{mathtools}
\usepackage{bbm}
\usepackage{amsfonts}
\usepackage{amssymb}
\usepackage[nameinlink,capitalize]{cleveref}

\makeatletter
\newcommand{\neutralize}[1]{\expandafter\let\csname c@#1\endcsname\count@}
\makeatother

   \newenvironment{lemmod}[2]
  {\renewcommand{\thelemma}{\ref*{#1}#2}%
   \neutralize{lemma}\phantomsection
   \begin{lemma}}
  {\end{lemma}}

\newenvironment{thmmod}[2]
  {\renewcommand{\thetheorem}{\ref*{#1}#2}%
   \neutralize{theorem}\phantomsection
   \begin{theorem}}
  {\end{theorem}}

\usepackage{algorithm}

\arxiv{
\usepackage{natbib}
\bibliographystyle{plainnat}
\bibpunct{(}{)}{;}{a}{,}{,}
}

\usepackage{xpatch}

\usepackage{thm-restate}
\usepackage{autonum}
\declaretheorem[name=Theorem,parent=section]{theorem}
\declaretheorem[name=Lemma,parent=section]{lemma}
\declaretheorem[name=Assumption, parent=section]{assumption}
\declaretheorem[name=Definition, parent=section]{definition}
\declaretheorem[name=Condition, parent=section]{condition}
\declaretheorem[name=Corollary, parent=section]{corollary}

\declaretheorem[qed=$\triangleleft$,name=Example,style=definition, parent=section]{example}
\declaretheorem[name=Remark, parent=section]{remark}
\declaretheorem[name=Proposition, parent=section]{proposition}

\usepackage{crossreftools}
\newcommand{\creftitle}[1]{\cref{#1}}

\makeatletter
  \renewenvironment{proof}[1][Proof]%
  {%
   \par\noindent{\bfseries\upshape {#1.}\ }%
  }%
  {\qed\newline}
  \makeatother

\xpatchcmd{\proof}{\itshape}{\normalfont\proofnameformat}{}{}
\newcommand{\proofnameformat}{\bfseries}

\newcommand{\pref}[1]{\cref{#1}}
\newcommand{\pfref}[1]{Proof of \pref{#1}}

\renewcommand{\eqref}[1]{\texorpdfstring{\hyperref[#1]{(\ref*{#1})}}{(\ref*{#1})}}
\crefformat{equation}{#2Eq.\,(#1)#3}
\Crefformat{equation}{#2Eq.\,(#1)#3}

\Crefformat{figure}{#2Figure~#1#3}
\Crefformat{assumption}{#2Assumption~#1#3}

\Crefname{assumption}{Assumption}{Assumptions}

\crefname{fact}{Fact}{Facts}

\Crefformat{figure}{#2Figure #1#3}
\Crefformat{assumption}{#2Assumption #1#3}

\usepackage{xparse}

\ExplSyntaxOn
\DeclareDocumentCommand{\XDeclarePairedDelimiter}{mm}
 {
  \__egreg_delimiter_clear_keys: %
  \keys_set:nn { egreg/delimiters } { #2 }
  \use:x %
   {
    \exp_not:n {\NewDocumentCommand{#1}{sO{}m} }
     {
      \exp_not:n { \IfBooleanTF{##1} }
       {
        \exp_not:N \egreg_paired_delimiter_expand:nnnn
         { \exp_not:V \l_egreg_delimiter_left_tl }
         { \exp_not:V \l_egreg_delimiter_right_tl }
         { \exp_not:n { ##3 } }
         { \exp_not:V \l_egreg_delimiter_subscript_tl }
       }
       {
        \exp_not:N \egreg_paired_delimiter_fixed:nnnnn 
         { \exp_not:n { ##2 } }
         { \exp_not:V \l_egreg_delimiter_left_tl }
         { \exp_not:V \l_egreg_delimiter_right_tl }
         { \exp_not:n { ##3 } }
         { \exp_not:V \l_egreg_delimiter_subscript_tl }
       }
     }
   }
 }

\keys_define:nn { egreg/delimiters }
 {
  left      .tl_set:N = \l_egreg_delimiter_left_tl,
  right     .tl_set:N = \l_egreg_delimiter_right_tl,
  subscript .tl_set:N = \l_egreg_delimiter_subscript_tl,
 }

\cs_new_protected:Npn \__egreg_delimiter_clear_keys:
 {
  \keys_set:nn { egreg/delimiters } { left=.,right=.,subscript={} }
 }

\cs_new_protected:Npn \egreg_paired_delimiter_expand:nnnn #1 #2 #3 #4
 {%
  \mathopen{}
  \mathclose\c_group_begin_token
   \left#1
   #3
   \group_insert_after:N \c_group_end_token
   \right#2
   \tl_if_empty:nF {#4} { \c_math_subscript_token {#4} }
 }
\cs_new_protected:Npn \egreg_paired_delimiter_fixed:nnnnn #1 #2 #3 #4 #5
 {
  \mathopen{#1#2}#4\mathclose{#1#3}
  \tl_if_empty:nF {#5} { \c_math_subscript_token {#5} }
 }
\ExplSyntaxOff

\XDeclarePairedDelimiter{\supnorm}{
  left=\lVert,
  right=\rVert,
  subscript=\infty
  }

\usepackage[utf8]{inputenc} %
\usepackage[T1]{fontenc}    %
\usepackage{url}            %
\usepackage{booktabs}       %
\usepackage{amsfonts}       %
\usepackage{nicefrac}       %
\usepackage{microtype}      %

\usepackage{makecell}

\usepackage{enumitem}

\iclr{
\setlist[enumerate]{leftmargin=*}
\setlist[itemize]{leftmargin=*}
}

\usepackage{breakcites}

\arxiv{\usepackage{subfigure}}
\usepackage{mathrsfs}

\usepackage{algorithm}
\usepackage{verbatim}
\usepackage[noend]{algpseudocode}

\usepackage{multicol}

\usepackage{colortbl}

\usepackage{setspace}

\usepackage{transparent}

\usepackage{upgreek}

\usepackage{inconsolata}
\usepackage[scaled=.90]{helvet}
\usepackage{xspace}

\usepackage{pifont}
\usepackage{bm}

\usepackage{tablefootnote}

\usepackage[most]{tcolorbox}

\tcbset {
  base/.style={
    arc=0mm, 
    bottomtitle=0.5mm,
    boxrule=0mm,
    colbacktitle=!10!white, 
    coltitle=black, 
    fonttitle=\bfseries, 
    left=2.5mm,
    leftrule=1mm,
    right=3.5mm,
    title={#1},
    toptitle=0.75mm, 
  }
}

\newtcolorbox{mainbox}[1][]{
  colframe=blue!40!black,
  colback=blue!2!white,
  enhanced,
  attach boxed title to top text left={yshift=-2mm},
  boxed title style={size=small,colframe=blue!40!black,colback=blue!40!black} 
  title={\textbf{#1}}
}

\DeclareFontFamily{U}{jkpmia}{}
\DeclareFontShape{U}{jkpmia}{m}{it}{<->s*jkpmia}{}
\DeclareFontShape{U}{jkpmia}{bx}{it}{<->s*jkpbmia}{}
\DeclareMathAlphabet{\mathfrak}{U}{jkpmia}{m}{it}
\SetMathAlphabet{\mathfrak}{bold}{U}{jkpmia}{bx}{it}

\DeclarePairedDelimiter{\abs}{\lvert}{\rvert} %
\DeclarePairedDelimiter{\brk}{[}{]}
\DeclarePairedDelimiter{\crl}{\{}{\}}
\DeclarePairedDelimiter{\prn}{(}{)}
\DeclarePairedDelimiter{\nrm}{\|}{\|}

\DeclarePairedDelimiter{\floor}{\lfloor}{\rfloor}

\let\Pr\undefined

\DeclareMathOperator{\En}{\mathbb{E}}

\DeclareMathOperator{\Pr}{Pr}

\DeclareMathOperator*{\argmin}{arg\,min} %
\DeclareMathOperator*{\argmax}{arg\,max}

\newcommand{\mb}[1]{\boldsymbol{#1}}
\newcommand{\wt}[1]{\widetilde{#1}}
\newcommand{\wh}[1]{\widehat{#1}}
\newcommand{\wb}[1]{\widebar{#1}}

\def\ddefloop#1{\ifx\ddefloop#1\else\ddef{#1}\expandafter\ddefloop\fi}
\def\ddef#1{\expandafter\def\csname bb#1\endcsname{\ensuremath{\mathbb{#1}}}}
\ddefloop ABCDEFGHIJKLMNOPQRSTUVWXYZ\ddefloop
\def\ddefloop#1{\ifx\ddefloop#1\else\ddef{#1}\expandafter\ddefloop\fi}
\def\ddef#1{\expandafter\def\csname b#1\endcsname{\ensuremath{\mathbf{#1}}}}
\ddefloop ABCDEFGHIJKLMNOPQRSTUVWXYZ\ddefloop
\def\ddef#1{\expandafter\def\csname sf#1\endcsname{\ensuremath{\mathsf{#1}}}}
\ddefloop ABCDEFGHIJKLMNOPQRSTUVWXYZ\ddefloop
\def\ddef#1{\expandafter\def\csname c#1\endcsname{\ensuremath{\mathcal{#1}}}}
\ddefloop ABCDEFGHIJKLMNOPQRSTUVWXYZ\ddefloop
\def\ddef#1{\expandafter\def\csname h#1\endcsname{\ensuremath{\widehat{#1}}}}
\ddefloop ABCDEFGHIJKLMNOPQRSTUVWXYZ\ddefloop
\def\ddef#1{\expandafter\def\csname hc#1\endcsname{\ensuremath{\widehat{\mathcal{#1}}}}}
\ddefloop ABCDEFGHIJKLMNOPQRSTUVWXYZ\ddefloop
\def\ddef#1{\expandafter\def\csname t#1\endcsname{\ensuremath{\widetilde{#1}}}}
\ddefloop ABCDEFGHIJKLMNOPQRSTUVWXYZ\ddefloop
\def\ddef#1{\expandafter\def\csname tc#1\endcsname{\ensuremath{\widetilde{\mathcal{#1}}}}}
\ddefloop ABCDEFGHIJKLMNOPQRSTUVWXYZ\ddefloop
\def\ddefloop#1{\ifx\ddefloop#1\else\ddef{#1}\expandafter\ddefloop\fi}
\def\ddef#1{\expandafter\def\csname scr#1\endcsname{\ensuremath{\mathscr{#1}}}}
\ddefloop ABCDEFGHIJKLMNOPQRSTUVWXYZ\ddefloop

\newcommand{\eps}{\epsilon}
\newcommand{\veps}{\varepsilon}

\newcommand{\ldef}{\vcentcolon=}
\newcommand{\rdef}{=\vcentcolon}

\newcommand{\xpo}{\texttt{XPO}\xspace}

\newcommand{\rself}{r_{\texttt{self}}}

\newcommand{\algofont}[1]{{\scalefont{1.0}{\texttt{{#1}}}}}

\newcommand{\mlsharp}{\text{maximum-likelihood sharpening}\xspace}
\newcommand{\Mlsharp}{\text{Maximum-likelihood sharpening}\xspace}

\newcommand{\pibase}{\piref}
\newcommand{\piref}{\pi_{\texttt{base}}}
\newcommand{\pirefh}[1][h]{\pi_{\texttt{base},#1}}

\newcommand{\pihatbon}{\pihat^{\texttt{BoN}}}
\newcommand{\pihatada}{\pihat^{\texttt{AdaBoN}}}

\newcommand{\ybon}{y^{\texttt{BoN}}}
\newcommand{\yada}{y^{\texttt{AdaBoN}}}

\newcommand{\hatytau}[1][\tau]{\wh{y}_{#1}}
\newcommand{\taumu}{N_{\mu}}

\newcommand{\framework}{sample-and-evaluate\xspace}

\newcommand{\rhs}{right-hand side\xspace}

\newcommand{\gametwentyfour}{\texttt{GameOf24}\xspace}
\newcommand{\prontoqa}{\texttt{ProntoQA}\xspace}
\newcommand{\mmlu}{\texttt{MMLU}\xspace}

\newcommand{\gsm}{\texttt{GSM8k}\xspace}
\newcommand{\mathdataset}{\texttt{MATH}\xspace}
\newcommand{\gptthree}{\texttt{gpt-3.5-turbo-instruct}\xspace}

\newcommand{\phithreemini}{\texttt{Phi3-Mini}\xspace}
\newcommand{\phithreesmall}{\texttt{Phi3-Small}\xspace}
\newcommand{\phithreemedium}{\texttt{Phi3-Medium}\xspace}
\newcommand{\phithreefivemini}{\texttt{Phi3.5-Mini}\xspace}
\newcommand{\llamathree}{\texttt{Llama3.2-3B-Instruct}\xspace}
\newcommand{\llamagame}{\texttt{llama2-7b-game24-policy-hf}\xspace}
\newcommand{\mistral}{\texttt{Mistral-7B-Instruct-v0.3}\xspace}

\newcommand{\cdist}{\mu} %

\newcommand{\Closs}{C_{\texttt{loss}}}
\newcommand{\Ccon}{C_{\texttt{conc}}}
\newcommand{\Cconc}{C_{\texttt{conc}}}
\newcommand{\Cstar}{C_{\texttt{cov}}}
\newcommand{\Cstarb}{\wb{C}_{\texttt{cov}}}
\newcommand{\Cstargp}{C_{\texttt{cov},\gamma,p}}
\newcommand{\Cstarg}{C_{\texttt{cov},\gamma}}

\newcommand{\cDpref}{\cD_{\texttt{pref}}}

\newcommand{\pin}{\pi^{\texttt{BoN}}_N}
\newcommand{\piada}{\pi^{\texttt{AdaBoN}}_\mu}

\newcommand{\Nstar}{N^{\star}}

\newcommand{\Ystar}{\mb{y}^{\star}}
\newcommand{\Ygamma}[1][\gamma]{\mb{y}^{\star}_{#1}}
\newcommand{\Ypi}[1][\pi]{\mb{y}^{#1}}
\newcommand{\Ypigamma}[1][\pi]{\mb{y}^{#1}_{\gamma}}

\newcommand{\YIgamma}[1][\gamma]{\mb{y}^{\cI}_{#1}}

\newcommand{\dirac}{\updelta}

\newcommand{\deltafail}{\rho}
\newcommand{\deltaf}{\deltafail}
\newcommand{\yhat}{\wh{y}}
\newcommand{\Cprob}{C_{\texttt{prob}}}

\newcommand{\bestofn}{best-of-$N$\xspace}
\newcommand{\BestofN}{Best-of-$N$\xspace}

\newcommand{\bestofnalg}{\texttt{SFT-Sharpening}\xspace}
\newcommand{\sftalg}{\texttt{SFT-Sharpening}\xspace}
\newcommand{\rlhfalg}{\algofont{RLHF-Sharpening}\xspace}

\newcommand{\Cpp}[2]{\cC_{\nicefrac{#1}{#2};\beta}}

\newcommand{\rstar}{r^{\star}}

\newcommand{\rhat}{\wh{r}}

\newcommand{\pistarb}{\pi^{\star}_{\beta}}

  \newcommand{\ystar}{y^{\star}}

\newcommand{\vepsstat}{\veps_{\texttt{stat}}}

\newcommand{\Unif}{\mathsf{Unif}}

\newcommand{\pibar}{\wb{\pi}}

\renewcommand{\emptyset}{\varnothing}

\newcommand{\M}[1]{^{{\scriptscriptstyle M}}}  %

\newcommand{\pistar}{\pi^{\star}}

\newcommand{\pihat}{\wh{\pi}}

\newcommand{\MV}{\mathcal{V}}
\newcommand{\NP}{\mathsf{NP}}

\newcommand{\approxleq}{\lesssim}
\newcommand{\approxgeq}{\gtrsim}

\newcommand{\bigoh}{O}
\newcommand{\bigoht}{\wt{O}}

\newcommand{\indic}{\mathbb{I}}

\renewcommand{\Pr}{\bbP}

\newcommand{\poly}{\mathrm{poly}}

\newcommand{\Dkl}[2]{D_{\mathsf{KL}}\prn*{#1\,\|\,#2}}

\newcommand{\Dhels}[2]{D^{2}_{\mathsf{H}}\prn*{#1,#2}}

\newcommand{\supp}{\mathrm{supp}}

\newcommand{\Pstar}{P^{\star}}

\newcommand{\mathand}{\quad\text{and}\quad}

\def\multiset#1#2{\ensuremath{\left(\kern-.3em\left(\genfrac{}{}{0pt}{}{#1}{#2}\right)\kern-.3em\right)}}

\newcommand{\Rmax}{R_{\mathsf{max}}}
\newcommand{\Vmax}{V_{\mathsf{max}}}
\newcommand{\NN}{\mathbb{N}}
\newcommand{\MX}{\mathcal{X}}
\newcommand{\MY}{\mathcal{Y}}
\newcommand{\BP}{\mathbb{P}}

\newcommand{\EE}{\mathbb{E}}
\newcommand{\MD}{\mathcal{D}}
\DeclareMathOperator*{\E}{\mathbb{E}}
\newcommand{\RR}{\mathbb{R}}
\DeclareMathOperator{\Tr}{Tr}
\newcommand{\norm}[1]{\lVert #1 \rVert}
\newcommand{\SEC}{\mathsf{SEC}}

\newcommand{\phit}{\widetilde \phi}
\newcommand{\thetat}{\widetilde \theta}
\newcommand{\MF}{\mathcal{F}}
\newcommand{\piz}{\pi_{\mathsf{zero}}}
\newcommand{\ty}{\widetilde y}

\newcommand{\epdisc}{\epsilon_{\mathsf{disc}}}

\renewcommand{\emptyset}{\varnothing}

\newcommand{\delfail}{\deltafail}
\newcommand{\gammargin}{\gamma_{\mathsf{margin}}}

\input{widebar}

\usepackage{graphicx}

\iclr{\usepackage{subfigure}}

 \usepackage[suppress]{color-edits}
 \addauthor{df}{ForestGreen}
 \addauthor{ab}{red}
 \addauthor{ak}{BurntOrange}
  \addauthor{ah}{olive}
  \addauthor{ms}{magenta}
  \addauthor{dr}{purple}
  \addauthor{jta}{red}
  \newcommand{\dfc}[1]{\dfcomment{#1}}

 \newcommand{\dhruv}[1]{\drcomment{#1}}

\makeatletter
\let\OldStatex\Statex
\renewcommand{\Statex}[1][3]{%
  \setlength\@tempdima{\algorithmicindent}%
  \OldStatex\hskip\dimexpr#1\@tempdima\relax}
\makeatother

 \usepackage{accents}

\let\oldparagraph\paragraph

\renewcommand{\paragraph}[1]{\oldparagraph{#1.}}

\iclr{
\title{Self-Improvement in Language Models: \\The Sharpening
  Mechanism}
}
\arxiv{
  \title{\huge Self-Improvement in Language Models:\\ The Sharpening
  Mechanism}
}

\arxiv{
  \author{
    \begin{tabular}{c c c c}
      \makecell{Audrey Huang\thanks{Equal contribution.} \\
      {\footnotesize\href{mailto:audreyh5@illinois.edu}{\texttt{audreyh5@illinois.edu}}}
      }
      &
              \makecell{Adam Block\footnotemark[1] \\
      {\footnotesize\href{mailto:blockadam@microsoft.com}{\texttt{blockadam@microsoft.com}}}
      }
      &
              \makecell{Dylan J. Foster\footnotemark[1] \\
      {\footnotesize \href{mailto:dylanfoster@microsoft.com}{\texttt{dylanfoster@microsoft.com}}}
      }
      &
        \makecell{Dhruv Rohatgi \\
      {\footnotesize \href{mailto:drohatgi@mit.edu}{\texttt{drohatgi@mit.edu}}}
      }
    \end{tabular}
    \and
        \begin{tabular}{c c c c}
      \makecell{Cyril Zhang \\
      {\footnotesize \href{mailto:cyrilzhang@microsoft.com}{\texttt{cyrilzhang@microsoft.com}}}
      }
      &
              \makecell{Max Simchowitz \\
      {\footnotesize \href{mailto:msimchow@andrew.cmu.edu}{\texttt{msimchow@andrew.cmu.edu}}}
      }
      &
              \makecell{Jordan T. Ash \\
      {\footnotesize \href{mailto:ash.jordan@microsoft.com}{\texttt{ash.jordan@microsoft.com}}}
      }
      &
        \makecell{Akshay Krishnamurthy \\
      {\footnotesize \href{mailto:akshaykr@microsoft.com}{\texttt{akshaykr@microsoft.com}}}
      }
\end{tabular}
  }
  }

\date{}

\usepackage{autonum}

  \addtocontents{toc}{\protect\setcounter{tocdepth}{0}}

\begin{document}
\maketitle

\begin{abstract}
  Recent work in language modeling has raised the possibility of
\emph{self-improvement}, where a language models evaluates and refines its own
generations to achieve higher performance without external feedback. 
It is impossible for
this self-improvement to create information that is not already in the
model, so
why should we expect that this will lead to
improved capabilities? 

We offer a new perspective on the capabilities of self-improvement
through a lens we refer to as \emph{sharpening}. Motivated by the observation
that language models are often better at verifying response quality than they are at generating correct
responses, we formalize self-improvement
        as using the model itself as a verifier during
        post-training in order to ``sharpen'' the model to one
        placing large mass on high-quality sequences, thereby amortizing the expensive
        inference-time computation of generating good sequences. We begin by introducing a new statistical framework
        for sharpening in which the learner aims to sharpen a
        pre-trained base policy via sample access, and establish
        fundamental limits. Then,
        we analyze two natural families of self-improvement algorithms based on SFT and
RLHF. We find that (i)
the SFT-based approach is minimax optimal whenever the
initial model has sufficient coverage, but (ii) the RLHF-based approach
can improve over SFT-based self-improvement by leveraging
online exploration, bypassing the need for
coverage. Finally, we empirically validate the sharpening mechanism via inference-time and amortization experiments. We view these findings as a starting point toward a foundational
understanding that can guide the design and evaluation of self-improvement
algorithms. \loose

\end{abstract}

\section{Introduction}
\label{sec:intro}
Contemporary language models are remarkably proficient on a wide range of natural language tasks \citep{brown2020language,ouyang2022training,touvron2023llama,achiam2023gpt,anil2023palm},  
but inherit shortcomings of the data on which they were trained. 
A fundamental challenge is to achieve better performance than what is directly induced by the distribution of available, human-generated training data. To this end, recent
work
\citep{huang2022large,wang2022self,bai2022constitutional,pang2023language,yuan2024self}
has raised the possibility of ``self-improvement,'' where a
model---typically through forms of self-play or self-training in which the
model critiques its own generations---learns to improve on its own,
without external feedback. This phenomenon is somewhat counterintuitive; at first glance it would seem to
disagree with the well-known data-processing inequality
\citep{cover1999elements}, which implies that no form of self-training
should be able to create information not already in the
model. This motivates the question of why we should expect such supervision-free interventions will lead to
stronger reasoning and planning capabilities.\loose%

  A dominant hypothesis for why improvement without external feedback might be possible is that 
  models contain ``hidden knowledge''~\citep{hinton2015distilling} that 
  is difficult to access. Self-improvement, rather than creating knowledge 
  from nothing, is a means of extracting and distilling this knowledge into a more
  accessible form, and thus is a computational
  phenomenon rather than a statistical one. 
  While there is a growing body of empirical evidence for this
  hidden-knowledge hypothesis
  \citep{furlanello2018born,gotmare2019closer,dong2019distillation,abnar2020transferring,allen2020towards},
  particularly in the context of self-distillation, a fundamental understanding of self-improvement remains missing. Concretely, where in the model is this hidden knowledge, and when and how can it be extracted? \loose

  \subsection{Our Perspective: The Sharpening Mechanism}

\begin{figure}[t]
    \centering
    \subfigure[]{
        \includegraphics[width=0.32\textwidth]{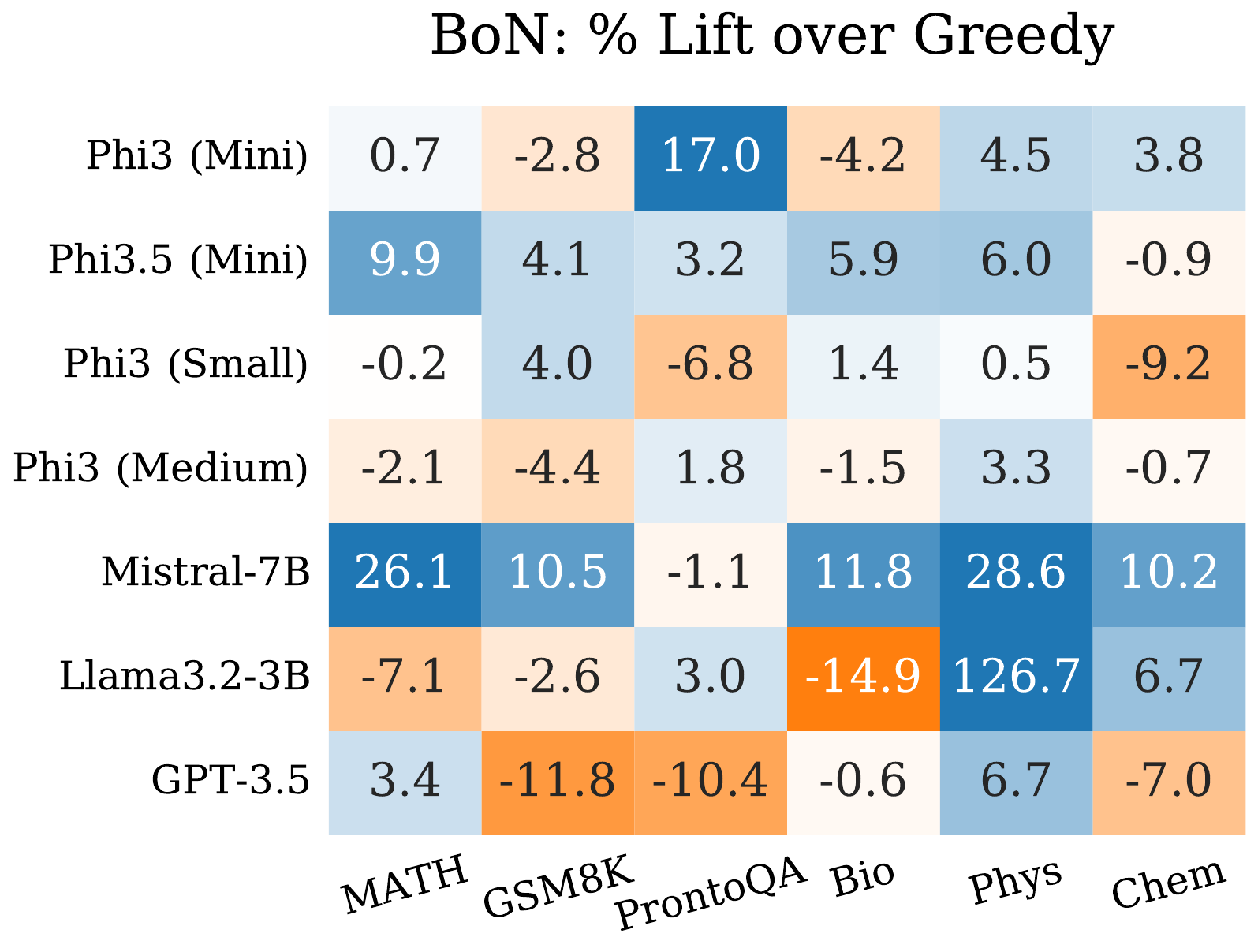}
        \label{sfig:bon-improvement} 
    }
    \hfill \subfigure[]{
        \includegraphics[width=0.28\textwidth]{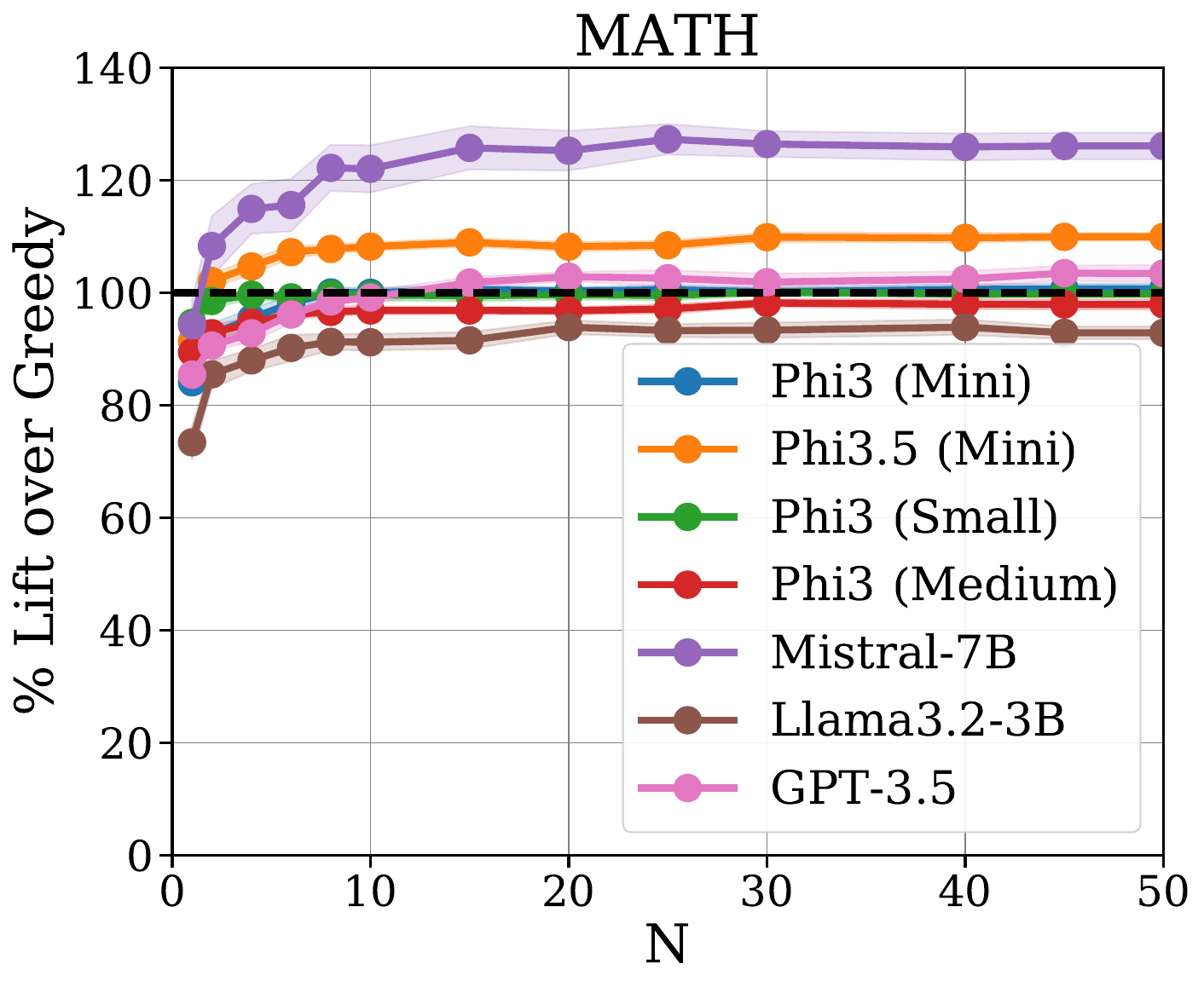}
        \label{sfig:bon-math-improvement}
    }
    \hfill
    \subfigure[]{
        \includegraphics[width=0.28\textwidth]{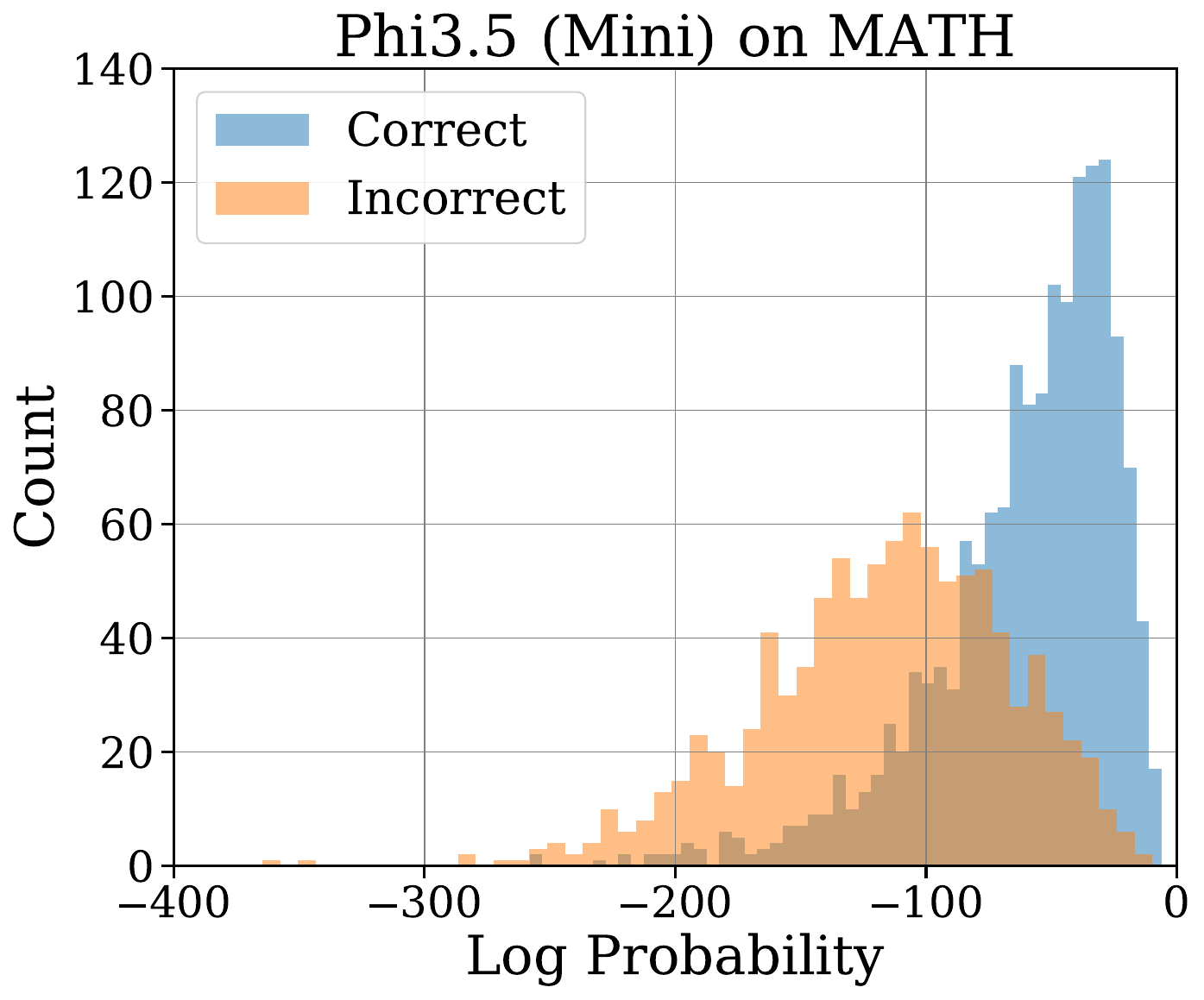}
        \label{sfig:logprob-sitribution}
    }
    \vspace{-0.25cm}
    \caption{Validation of maximum-likelihood sharpening, via
      Best-of-$N$ (BoN) sampling, at inference time. (a) Percent
      accuracy improvement over greedy decoding for BoN sharpening
      with $N=50$ on 6 tasks and 7 models, colored by performance.
      (b) Perecent accuracy improvement over greedy for BoN sharpening as a
      function of $N$ for 7 different models on the MATH dataset.  (c)
      Distribution over sequence-level log probabilities for sampled
      completions ($N=1$) from \phithreefivemini on the MATH dataset,
      conditioned on whether or not the completion is correct.
      Correct completions are noticeably in higher likelihood than incorrect completions, demonstrating the utility of inference-time sharpening.}\label{fig:validation}\vspace{-0.25cm}
\end{figure}

In this paper, we posit a potential source of hidden knowledge, and
offer a formal perspective on how
to extract it. %
Our starting point is the widely observed phenomenon that language models are often better at verifying whether
responses are correct than they are at generating correct
responses \citep{huang2022large,wang2022self,bai2022constitutional,pang2023language,yuan2024self}. %
This gap may be explained by the theory of computational complexity, which suggests that generating high-quality
responses can be less computationally tractable than verification
\citep{cook1971complexity,levin1973universal,karp1972reducibility}. In
autoregressive language modeling, computing the most likely response for a given
prompt is $\NP$-hard in the worst case (\cref{sec:hardness}), whereas the model's likelihood for
a given response can be easily evaluated.\looseness=-1

We view self-improvement as any attempt to narrow this gap, i.e., use the model as its own
verifier to improve generation and \emph{sharpen} the model toward
  high-quality responses.
Formally, consider a learner with access to a base model
$\piref:\cX\to\Delta(\cY)$ representing a conditional distribution
that maps a prompt $x\in\cX$ to a distribution over responses
(i.e., $\piref(y\mid{}x)$ is the probability that the model generates
the response $y$ given the prompt $x$).\iclr{\footnote{Our \arxiv{most }general results are
agnostic to the structure of $\cX$, $\cY$, and $\piref$, but an
important special case for language modeling is the autoregressive setting where
$\cY=\cV^{H}$ for a vocabulary space $\cV$ and sequence
  length $H$, and where $\piref$ has the autoregressive structure
$\piref(y_{1:H}\mid{}x)=\prod_{h=1}^{H}\pirefh(y_h\mid{}y_{1:h-1},x)$
for $y=y_{1:H}\in\cY$.\loose}} 
We posit that $\piref$ has already been trained in some manner
  (e.g., through next-token prediction or additional post-training steps such as SFT or RLHF), 
with the key feature being that $\piref$ is a good verifier, as measured by some \emph{self-reward} function $\rself (y \mid x; \piref)$ measuring model certainty.
The self-reward
function is derived purely from the base model $\piref$, without external supervision or feedback. Examples include
normalized and/or regularized sequence likelihood \citep{meister2020if}, models-as-judges \citep{zheng2024judging,yuan2024self,wu2024meta,wang2024self}, and model
confidence \citep{wang2024chain}.\loose

\begin{tcolorbox}[enhanced,title=Sharpening,
    colframe=blue!40!black,
    colback=blue!2!white,
    fonttitle=\bfseries,
  attach boxed title to top text left={xshift=30mm,yshift=-2.5mm},
  boxed title style={size=small,colframe=blue!40!black,colback=blue!40!black}]
    We refer to \textbf{sharpening} as any process that tilts $\piref$
  toward responses that are more certain in the sense that they
  enjoy greater self-reward $\rself$.
  That is, a sharpened model $\pihat$ is one that
  (approximately) maximizes the self-reward:
  \loose
  \begin{align}
    \label{eq:sharpening}
    \pihat(x)\approx \argmax_{y\in\cY}\rself(y\mid{}x; \piref).
  \end{align}
\end{tcolorbox}

\akdelete{Note that, in \cref{eq:sharpening}, $y$ denotes an entire response, rather than a single token.
Sharpening may be implemented at inference-time, or \textbf{amortized} via
self-training (\Cref{sec:algorithms}). Popular decoding
strategies such as greedy, low-temperature sampling, and beam-search can all be
viewed as instances of the former (albeit at the
token-level).\footnote{More sophisticated decoding strategies like
normalized/regularized sequence likelihood
\citep{meister2020if} or chain-of-thought decoding
\citep{wang2024chain} \iclr{\abedit{also} admit an interpretation as sharpening;
see \cref{sec:additional_related}.}\arxiv{use various metrics of model
``confidence'' to guide sampling in the
  absence of external verifiers; these too admit an informal
  interpretation as ``sharpening''. See \cref{sec:additional_related}.}\loose}
The latter captures many existing self-training schemes
\citep{huang2022large,wang2022self,bai2022constitutional,pang2023language,yuan2024self},
and is the main focus of this paper; we use the term \emph{sharpening} without further qualification to
  refer to the latter.\loose}

\akedit{An important special case for sharpening is in
  language/autoregressive modeling. Here, we have $\cY = \cV^H$ for a
  vocabulary space $\cV$ and sequence length $H$, and $\piref$ has the
  autoregressive structure $\piref(y_{1:H}\mid{}x) = \prod_{h=1}^H
  \pirefh(y_h \mid{}y_{1:h-1},x)$ for $y = y_{1:H} \in
  \cY$. Sharpening in this setting pertains to entire responses, i.e.,
  the optimization over responses in \Cref{eq:sharpening} is at the \emph{sequence level}. In contrast, popular decoding strategies such as greedy, low-temperature sampling, and beam search operate at the token-level; nevertheless, they can be viewed as heuristics for \emph{inference-time sharpening}.\footnote{More sophisticated decoding strategies like
normalized/regularized sequence likelihood
\citep{meister2020if} or chain-of-thought decoding
\citep{wang2024chain} \iclr{\abedit{also} admit an interpretation as sharpening;
see \cref{sec:additional_related}.}\arxiv{use various metrics of model
``confidence'' to guide sampling in the
  absence of external verifiers; these too admit an informal
  interpretation as ``sharpening''. See \cref{sec:additional_related}.}\loose} The combinatorial response space can make sharpening computationally demanding and so, an appealing alternative to inference-time sharpening is \emph{amortization via self-training} (\Cref{sec:algorithms}). The latter captures many existing self-training schemes
\citep{huang2022large,wang2022self,bai2022constitutional,pang2023language,yuan2024self},
and is the main focus of this paper; we use the term \emph{sharpening} without further qualification to
  refer to the latter.\loose}

We refer to the \textbf{sharpening mechanism} as the phenomenon where responses from a model with the highest certainty (in the sense
of large self-reward $\rself$) exhibit the greatest performance on a
task of interest. Though it is unclear a-priori whether there are self-rewards related to task performance, the successes of self-improvement in prior works \citep{huang2022large,wang2022self,bai2022constitutional,pang2023language,yuan2024self} give strong positive evidence.
These works suggest that, in many settings, models do have hidden knowledge: the model's own self-reward correlates with response quality, but it is computationally challenging to generate high self-rewarding---and thus high quality---responses. It is the role of (algorithmic) sharpening to leverage these verifications to improve the quality of generations, despite computational difficulty.\loose

\newcommand{\rbase}{r_{\mathrm{base}}}

\subsection{Contributions}

We initiate the theoretical study of self-improvement via the
sharpening mechanism. We disentangle the choice of self-reward from the algorithms used to optimize it, and aim to
understand: (i) When and how does self-training
achieve sharpening? (ii) What are the fundamental
limits for self-training algorithms?\loose
\arxiv{\dfc{above seems like the most clear description of what we do, but it
  would be nice if we could also weave in experiments.}}

\paragraph{Algorithms for sharpening (\cref{sec:algorithms})}
The starting point for our work is to consider
two natural families of self-improvement algorithms based on
supervised fine-tuning (SFT) and reinforcement learning (RL/RLHF),
respectively, \bestofnalg and \rlhfalg. Both algorithms \textbf{amortize} the
sharpening objective \eqref{eq:sharpening} into a dedicated
post-training/fine-tuning phase:
\begin{itemize}
\item \bestofnalg filters responses where the self-reward $\rself(y\mid{}x;\piref)$
  is large and fine-tunes on the resulting dataset, invoking common
  SFT pipelines
\arxiv{\citep{amini2024variational,sessa2024bond,gui2024bonbon,pace2024west}}\iclr{\citep{amini2024variational,sessa2024bond}}.\loose
  \item \rlhfalg{} directly applies reinforcement learning techniques (e.g.,
  PPO~\citep{schulman2017proximal} or DPO~\citep{rafailov2024direct}) to optimize the self-reward function $\rself(y\mid{}x;\piref)$.
\end{itemize}
In the remainder of the paper, we introduce a theoretical framework to analyze the performance of these
  algorithms\iclr{.}\arxiv{, and validate our findings empirically.} Our main contributions are as follows.

\paragraph{Maximum-likelihood sharpening objective
  (\Cref{sec:autoreg_sharpening})} As a concrete proposal for one
source of hidden knowledge, we focus on self-rewards defined by the model's sequence-level log-probabilities:
\begin{align}
  \label{eq:ml_sharpening}
    \rself(y \mid x; \piref) := \log \piref(y \mid x)
\end{align}
This is a 
stylized self-reward function, which offers perhaps the simplest objective for
self-improvement in the absence of external feedback (i.e., purely
supervision-free), yet also connects self-improvement to a rich body of theoretical computer science literature on computational trade-offs for optimization (inference) versus sampling (\Cref{sec:additional_related}). 
\arxiv{\akedit{We view \Cref{eq:ml_sharpening} as a \emph{clean and minimal objective} that reveals the interplay between hidden knowledge, computational bottlenecks, and self-improvement in generative model.}}
In spite of its simplicity,\arxiv{ we show empirically
that} \mlsharp is already sufficient to achieve non-trivial performance gains \akedit{over greedy decoding on a range of 
reasoning tasks with several language models;}\akdelete{reasoning tasks such as \gametwentyfour, \gsm, and \mathdataset \abedit{over greedy decoding};} cf. \cref{fig:validation}. We believe it can serve as a starting point toward
understanding forms of self-improvement that use more
sophisticated self-rewarding\arxiv{ or judging but are less amenable to theoretical analysis}
\citep{huang2022large,wang2022self,\arxiv{bai2022constitutional,}pang2023language,yuan2024self}. \loose

\paragraph{A statistical framework for sharpening
(\cref{sec:sample,sec:lower})} Though the goal of sharpening is
computational in nature, we recast self-training according to the
\mlsharp objective \cref{eq:ml_sharpening} as
a \textbf{statistical} problem where we aim to produce a model
approximating \eqref{eq:sharpening} using a polynomial number of (i)
sample prompts $x \sim \mu$, (ii) sampling queries of the form $y
\sim \piref(x)$, and (iii) likelihood evaluations of the form $\piref(y \mid x)$. 
\iclr{Evaluating the efficiency of the algorithm through the number of such queries, this abstraction
offers a natural way to evaluate the performance of self-improvement/sharpening
algorithms and establish fundamental limits; we use our framework to prove new
lower bounds that highlight the importance of
the base model's coverage.
}
\arxiv{Evaluating the efficiency of the algorithm through the number of such queries, this abstraction
offers a natural way to evaluate the performance of self-improvement/sharpening
algorithms and establish fundamental limits and minimax optimality,
similar to the role of information-based complexity in optimization
\citep{nemirovski1983problem,traub1988information,raginsky2011information,agarwal2012information}, 
statistical query complexity in computational learning theory
\citep{blum1994weakly,kearns1998efficient,feldman2012complete,feldman2017general},
and query complexity more broadly. We use our framework to prove new
lower bounds and fundamental limits which highlight the importance of \abedit{the base model's}
\emph{coverage} (that is, probability mass placed on high-quality responses). 
}

\paragraph{Analysis of sharpening algorithms (\cref{sec:theoretical_analysis})}
Within our statistical framework for \arxiv{maximum-likelihood }sharpening, we show that
\bestofnalg and \rlhfalg provably converge to sharpened models,
establishing several results:\iclr{ \textbf{(i) \bestofnalg is
    minimax optimal}, and learns a sharpened model whenever $\piref$ has
  sufficient coverage (we also show
  that a novel variant based on adaptive sampling can sidestep
  the minimax lower bound); \textbf{(ii) \rlhfalg benefits from
    on-policy exploration}, and can bypass
  the need for coverage---improving over \bestofnalg.\loose
  }
\arxiv{\begin{itemize}
\item \textbf{Optimality of \bestofnalg.} We show that \bestofnalg
  succeeds at learning a sharpened model whenever $\piref$ has
  sufficient coverage, and is minimax
  optimal in a worst-case sense. Perhaps surprisingly, we show
  that a novel variant based on adaptive sampling can bypass
  this lower bound.
\item \textbf{Benefits of \rlhfalg.} We show that \rlhfalg
  also succeeds at learning a sharpened model and achieves similar performance to \bestofnalg when $\piref$ has
  sufficient coverage. However, we show that this algorithm can bypass
  the need for coverage---improving over \bestofnalg---by leveraging
  deliberate \emph{exploration} of the response space.\loose
\end{itemize}
}
\paragraph{Empirical investigation (\Cref{sec:experiments})}  We
empirically explore the extent to which our theoretical framework can aid language models in a variety of tasks.  We first consider three choices of self-reward, including \mlsharp, and sharpen via a practical approximation, inference-time best-of-N sampling: given a prompt $x \in \cX$, we draw $N$ responses $y_1,\ldots,y_N\sim\piref(\cdot\mid{}x)$ and return the response $\widehat{y} = \argmax_{y_i} \rself(y_i \mid{}x)$; this is
equivalent to \iftoggle{workshop}{\cite{stiennon2020learning,gao2023scaling,yang2024asymptotics}}{\citet{stiennon2020learning,gao2023scaling,yang2024asymptotics}} \abedit{and is a popular approach in modern deployments}.\footnote{We mention in passing that
  inference-time \bestofn sampling enjoys provable guarantees for
  maximizing the \mlsharp objective when $N$ is sufficiently
  large. See \cref{sec:inference} for details.}
We consider an extensive list of model-dataset pairs and find that sharpening, even with the stylized maximum-likelihood self-reward, often improves performance over greedy decoding. 
  We then implement one of our algorithms, \bestofnalg, on a subset of
  these model-dataset pairs and observe a significant positive effect
  on performance, indicating that sharpening can indeed be amortized.  An overview of our inference-time experiments can be found in \Cref{fig:validation}.

\dfc{Feels like we should perhaps advertise the fact that we are doing
  this in a learning-theoretic way that abstracts away the role of
  optimization/model and allows for general function classes. For
  arxiv I think we should definitely add a paragraph about this at least.
  }

\subsection{Related Work}

Our work is most directly related to a growing body of empirical research that
studies self-training for language models in a
supervision-free setting with no external
feedback~\citep{huang2022large,wang2022self,bai2022constitutional,pang2023language,yuan2024self}.
The
specific algorithms for self-improvement/sharpening we study can be
viewed as applications of standard alignment algorithms
\citep{amini2024variational,sessa2024bond,\arxiv{gui2024bonbon,pace2024west,}christiano2017deep,bai2022training,ouyang2022training,rafailov2024direct}
with a specific choice of reward function. 
However, the maximum
likelihood sharpening objective \eqref{eq:ml_sharpening} used for our
theoretical results has been relatively unexplored within the
alignment and self-improvement literature.\loose

\dfc{Add yuda's paper once it's online.}

      On the theoretical side, current understanding of self-training is limited. One line of work, focusing on the
      \emph{self-distillation} objective \citep{hinton2015distilling}
      for classification and regression, aims to provide convergence
      guarantees for self-training in stylized setups such as linear
      models
      \citep{mobahi2020self,frei2022self,das2023understanding,das2024retraining,pareek2024understanding},
      with 
      \iftoggle{workshop}{\cite{allen2020towards}}{\citet{allen2020towards}} giving guarantees for feedforward
      neural networks and \citet{boix2024towards} proposing a general
      PAC-style framework. 
      To the best of our knowledge, our work is the first to study
      self-training in a general framework that subsumes language
      modeling. \iclr{See \cref{sec:additional_related} for a more extensive discussion of
        related work.\loose}

\arxiv{See \cref{sec:additional_related} for a more extensive discussion of
related work.} 

\arxiv{
\paragraph{Notation}
  For an integer $n\in\bbN$, we let $[n]$ denote the set
  $\{1,\dots,n\}$. For a set $\cX$, we let $\Delta(\cX)$ denote the
  set of all probability distributions over $\cX$. We adopt
    standard big-oh notation, and write $f=\bigoht(g)$ to denote that
    $f = \bigoh(g\cdot{}\max\crl*{1,\mathrm{polylog}(g)})$,
    $a\approxleq{}b$ as shorthand for $a=\bigoh(b)$, and $a \asymp{} b$ as shorthand for $a = \Theta(b)$. \loose
    }

\section{Sharpening Algorithms for Self-Improvement}
\label{sec:algorithms}
\newcommand{\givebase}{}

This section introduces the two families of self-improvement algorithms
for sharpening
that we study. {Going forward, we omit the dependence of
  $\rself$ on $\pibase$ when it is clear from context. We use the notation $\argmax_{\pi \in \Pi}$ or $\argmin_{\pi \in \Pi}$
  to denote exact optimization over a user-specified model class
  $\Pi$ for theoretical results
  \citep{agarwal2019reinforcement,foster2023foundations}; empirically,
  these operations can be implemented by training a neural network to low loss.}

\subsection{Self-Improvement through SFT: \bestofnalg}

\label{sec:bestofn}
\bestofnalg filters responses for which the self-reward
$\rself(y\mid{}x)$ 
is large, and applies standard supervised fine-tuning on the resulting dataset
\citep{amini2024variational,sessa2024bond,gui2024bonbon,pace2024west}. This can be viewed as 
amortizing inference-time sharpening via the  effective-but-costly  \bestofn sampling approach
\citep{brown2024large,snell2024scaling,wu2024empirical}. 
Concretely, suppose we have a collection of prompts
$x_1,\ldots,x_n$. For each prompt, we sample $N$ responses
$y_{i,1},\ldots,y_{i,N}\sim\piref(\cdot\mid{}x_i)$, then compute the
\bestofn response
$\ybon_i=\argmax_{j\in\brk{N}}\crl*{\rself(y_{i,j}\mid{}x_i\givebase)}$,
scoring via the model's self-reward function. We
 compute the sharpened model via supervised fine-tuning on the \bestofn responses:
\begin{align}
  \pihatbon = \argmax_{\pi \in \Pi}\sum_{i=1}^{n}\log \pi(\ybon_i\mid{}x_i\givebase).
\end{align}
\arxiv{\bestofnalg}\iclr{This} is a simple, flexible self-training scheme, and
converges to a sharpened model as $n, N\to \infty$.\arxiv{ In
\cref{sec:adaptive}, we consider a variant of
  \bestofnalg based on \emph{adaptive sampling}, which adjusts the
  number of sampled responses adaptively for better performance.}\loose

\subsection{Self-Improvement through RLHF: \rlhfalg}
\label{sec:rlhf}
A drawback of the \bestofnalg algorithm is that it may ignore useful information contained in the self-reward function $\rself(y\mid{}x\givebase)$. Fixing a regularization parameter $\beta > 0$ throughout, 
 our second class of algorithms solve a KL-regularized reinforcement
 learning problem in the spirit of RLHF and other alignment methods
\citep{christiano2017deep,\arxiv{bai2022training,ouyang2022training,}rafailov2024direct}.   Defining
$\En_{\pi}\brk{\cdot}=\En_{x\sim\cdist,y\sim\pi(\cdot\mid{}x)}\brk{\cdot}$
  and 
  $\Dkl{\pi}{\piref}=\En_{\pi}\brk[\big]{\log\frac{\pi(y\mid{}x)}{\piref(y\mid{}x)}}$,
  we choose
\begin{align}
  \label{eq:rlhf}
  \pihat \approx \argmax_{\pi\in\Pi}\crl*{
  \En_{\pi}\brk*{\rself(y\mid{}x\givebase)}
  -\beta\Dkl{\pi}{\piref}
  }.
\end{align}
  The exact optimizer $\pistarb = \argmax_{\pi\in\Pi}\crl*{
  \En_{\pi}\brk*{\rself(y\mid{}x\givebase)}
  -\beta\Dkl{\pi}{\piref}
}$ for this objective has the form \iclr{$
  \pistarb(y\mid{}x)\propto\piref(y\mid{}x)\cdot\exp\prn*{\beta^{-1}\rself(y\mid{}x\givebase)}$,}
\arxiv{\begin{align}
  \pistarb(y\mid{}x)\propto\piref(y\mid{}x)\cdot\exp\prn*{\beta^{-1}\rself(y\mid{}x\givebase)},
\end{align}}
which converges to the solution to the sharpening objective in
\cref{eq:sharpening} as $\beta\to{}0$. Thus, \cref{eq:rlhf} can be seen to
encourage sharpening.\loose

There are many choices for what RLHF/alignment algorithm one might use to solve \eqref{eq:rlhf}. 
For our theoretical results, we implement \cref{eq:rlhf} using an
approach inspired by DPO and its reward-based variants \citep{rafailov2024direct,gao2024rebel}. Given a dataset  $\cD=\crl*{(x,y,y')}$ of $n$ examples sampled via
$x\sim\cdist$ and $y,y'\sim\piref(y\mid{}x)$, we consider the algorithm
that solves
\begin{small}
  \begin{align}
    \label{eq:dpo}
    \pihat\in \argmin_{\pi\in\Pi}\sum_{(x,y,y')\in\cD}\prn*{
    \beta\log\frac{\pi(y\mid{}x)}{\piref(y\mid{}x)}-\beta\log\frac{\pi(y'\mid{}x)}{\piref(y'\mid{}x)}
    - \prn*{\rself(y\mid{}x\givebase)-\rself(y'\mid{}x\givebase)}
    }^2.
  \end{align}%
\end{small}%
In the sequel (\cref{sec:theoretical_analysis}), we show that this approach leads to comparable
guarantees to \bestofnalg, but that a more sophisticated DPO variant
that incorporates \emph{online exploration} \citep{xie2024exploratory} can offer provable
benefits.

\section{A Statistical Framework for Sharpening}
\label{sec:theoretical_framework}

This section introduces the theoretical framework within which we will
analyze the \bestofnalg and \rlhfalg algorithms. We first introduce
the \mlsharp objective as a stylized self-reward function, then
introduce our statistical framework for sharpening.
\iclr{We write $f=\bigoht(g)$ to denote 
    $f = \bigoh(g\cdot{}\max\crl*{1,\mathrm{polylog}(g)})$ and
    $a\approxleq{}b$ as shorthand for $a=\bigoh(b)$.\loose}
\subsection{Maximum-Likelihood Sharpening}\label{sec:autoreg_sharpening}

Our theoretical results focus on the \mlsharp objective given
by
\begin{align}
  \label{eq:ml_sharpening2}
    \rself(y \mid x) := \log \piref(y \mid x),
\end{align}
which we aim to maximize using conditional samples
$y\sim\piref(\cdot\mid{}x)$ from the base model. This is a simple and stylized self-reward function, but we will show
that it enjoys a rich theory. In
particular, we can restate the problem of sharpening with this self-reward through the lens of \emph{amortization}.
\loose
\begin{mainbox}
  \emph{Can we efficiently \textbf{amortize maximum likelihood
      inference (optimization)}
  for a conditional distribution $\piref(y\mid{}x)$ given access to a
  \textbf{sampling oracle} that can sample $y\sim\piref(\cdot \mid{}x)$?\loose}
\end{mainbox}
\dfedit{The tacit assumption in this framing is that the
  maximum-likelihood response constitutes a useful form of hidden
  knowledge.} \Mlsharp connects
the study of self-improvement to a large body of research in
theoretical computer science demonstrating computational reductions between
optimization (inference) and sampling (generation)
\citep{kirkpatrick1983optimization,lovasz2006fast,singh2014entropy,ma2019sampling,talwar2019computational}. %
Our sharpening framework offers a new learning-theoretic perspective by focusing on the problem of amortizing this
type of reduction\akdelete{ \abedit{and is inspired by the classical query complexity framework\citep{nemirovski1983problem,traub1988information,raginsky2011information,agarwal2012information}}}. \msdelete{}

We evaluate the quality of an
approximately sharpened model as follows. Let
\[\Ystar(x)\ldef{}\argmax_{y\in\cY}\log \piref(y\mid{}x);\] we interpret
$\Ystar(x)\subset\cY$ as a set  to accommodate non-unique
maximizers, and will write $\ystar(x)$ to indicate a
unique maximizer when it exists (i.e., when
$\Ystar(x)=\crl{\ystar(x)}$).

                                                     \begin{definition}[Sharpened model]
                                                       \label{def:sharpening}
    We say that a model $\pihat$ is $(\eps,\delta)$-sharpened
    relative to $\piref$ if
    \begin{align}
      \bbP_{x\sim\cdist}\brk*{\pihat\prn*{\Ystar(x)\mid{}x}\geq{}1-\delta} \geq{} 1-\eps.
    \end{align}
  \end{definition}
That is, an $(\eps,\delta)$-sharpened model places at least
$1-\delta$ mass on arg-max responses on all but an $\eps$-fraction
of prompts under $\mu$. For small $\delta$ and $\eps$, we are
guaranteed that $\pihat$ is a high-quality generator: sampling from the model will produce an arg-max
response with high probability for most prompts.

  \paragraph{\Mlsharp for autoregressive models}
  Though our most general results are
  agnostic to the structure of $\cX$, $\cY$, and $\piref$, our primary
  motivation is
the autoregressive setting in which
$\cY=\cV^{H}$ for a \emph{vocabulary space} $\cV$ and sequence
  length $H$, and where $\piref$ has the autoregressive structure
$\piref(y_{1:H}\mid{}x)=\prod_{h=1}^{H}\pirefh(y_h\mid{}y_{1:h-1},x)$
for $y=y_{1:H}\in\cY$.
We observe
that when the response
$y=(y_1,\ldots,y_H)\in\cY=\cV^{H}$ is a sequence of tokens, the
\mlsharp objective \eqref{eq:ml_sharpening} sharpens
toward the \emph{sequence-level} arg-max response:
\begin{align}
  \label{eq:inference_time}
\argmax_{y_{1:H}}\log\piref(y_{1:H}\mid{}x).
\end{align}
\dfedit{Although somewhat stylized, \cref{eq:inference_time}
 is a non-trivial (in general, computationally intractable; see \Cref{sec:hardness})
 solution concept. We view the sequence-level arg-max as a form of hidden
 knowledge that cannot necessarily be uncovered through naive
 sampling or greedy decoding.
  \loose}

  \loose

  \paragraph{Role of $\delta$ for autoregressive models}

  As can be verified through simple examples, beam-search and greedy
  tokenwise decoding do not return an exact (or even approximate)
  solution to \eqref{eq:inference_time} in general. There is one
  notable exception: If the model has already been sharpened to
  $\delta<1/2$ and the arg-max sequence is unique, then greedy
  decoding will succeed.
  \loose
  \begin{proposition}[Greedy decoding succeeds for sharpened policies]
    \label{prop:greedy}
    Let $\pi=\pi_{1:H}$ be an autoregressive model defined over
    response space $\cY=\cV^{H}$. For a given prompt $x\in\cX$, if
    $\Ystar(x)=\crl*{\ystar(x)}$ is a singleton and
    $\pi(\ystar(x)\mid{}x)>1/2$, then the greedy decoding strategy
    that selects
    \iclr{$\yhat_h=\argmax_{y_h\in\cV}\pi_h(y_h\mid{}\yhat_1,\ldots,\yhat_{h-1},x)$}
    \arxiv{\begin{align}
      \yhat_h=\argmax_{y_h\in\cV}\pi_h(y_h\mid{}\yhat_1,\ldots,\yhat_{h-1},x)
    \end{align}}
    guarantees that $\yhat=\ystar(x)$. This result is tight, in the
    sense that there exist $\pi$ with $\pi(\ystar(x)\mid{}x)\leq{}1/2$
    for which greedy decoding fails to recover $\ystar(x)$. \loose
  \end{proposition}
  This means that if we start from an un-sharpened model, it can suffice
  to focus on sharpening to $\delta<1/2$.

\subsection{Sample Complexity Framework}
\label{sec:sample}
As described, sharpening in the sense of \cref{def:sharpening} is a
purely computational problem, which makes it difficult to evaluate the
quality and optimality of
self-improvement algorithms. To address this, we introduce a novel
statistical
framework for sharpening, inspired by the success of oracle complexity in optimization
\citep{nemirovski1983problem,traub1988information,raginsky2011information,agarwal2012information}
and statistical query complexity in computational learning theory \citep{blum1994weakly,kearns1998efficient,feldman2012complete,feldman2017general}.

\begin{definition}[Sample-and-evaluate framework]\label{def:oracle-model}
  In the \textbf{\framework} framework, the algorithm designer does
  not have explicit access to the base model $\piref$. Instead, they
access $\piref$ only through \emph{sample-and-evaluate queries}: The learner is allowed to sample $n$ prompts $x \sim
\cdist$. For each prompt $x$, they can sample $N$
responses $y_1,y_2,\dots y_N \sim \piref(\cdot \mid x)$ and observe
the likelihood $\piref(y_i\mid{}x)$ for each such response. The
efficiency, or \emph{sample complexity}, of the algorithm is measured through the total number of
sample-and-evaluate queries
$m\ldef{}n\cdot{}N$.
\end{definition}
This framework can be seen to capture algorithms like \bestofnalg and
\rlhfalg (implemented with DPO\dfdelete{ or PPO}), which only
access the base model $\piref$ through i) sampling responses via
$y\sim\piref(\cdot\mid{}x)$ \textbf{(generation)}, and ii) evaluating the likelihood
$\piref(y\mid{}x)$ \textbf{(verification)} for these responses. We view the sample complexity
$m=n\cdot{}N$ as a natural statistical abstraction for the
computational complexity of self-improvement (a clear parallel to oracle complexity for optimization algorithms), one which is amenable to
information-theoretic lower bounds.\footnote{Concretely, the sample
  complexity $m=n\cdot{}N$ is a lower bound on the running time
  of any algorithm that operates in the \framework framework.} We will
aim to show that, under appropriate assumptions, \bestofnalg and
\rlhfalg can learn an $(\eps,\delta)$-sharpened model with 
sample complexity
\[
  m = \poly(\eps^{-1},\delta^{-1},\Cprob)
\]
where $\Cprob$ is a potentially problem-dependent constant.

\subsection{Fundamental Limits}
\label{sec:lower}

Before diving into our analysis of \bestofnalg and \rlhfalg in the
\framework framework, let us take a brief detour to give a sense for
how sample complexity guarantees for sharpening should scale. To this end, we will prove a lower bound or fundamental limit on
the sample complexity of any algorithm in the \framework framework.

Intuitively, the performance of any sharpening algorithm based on
sampling should depend on how well the base model $\piref$
covers the arg-max response $\ystar(x)$. To capture this, we define
the following \emph{coverage coefficient}:\footnote{This quantity can
  be interpreted as a special case of the $L_1$-concentrability
  coefficient \citep{farahmand2010error,xie2020q,zanette2021provable,amortila2024scalable}
  studied in the theory of offline reinforcement learning.}\loose
\begin{align}
  \label{eq:cstar}
  \Cstar = \En_{x\sim\cdist}\brk*{\frac{1}{\piref(\Ystar(x)\mid{}x)}}.
\end{align}
More generally, for a model $\pi$, we define
$\Ypi(x)=\argmax_{y\in\cY}\pi(y\mid{}x)$ and 
$  \Cstar(\pi) =
\En_{x\sim\cdist}\brk*{\frac{1}{\pi(\Ypi(x)\mid{}x)}}$. 

Our main lower
bound shows that for worst-case choice of $\Pi$, the coverage coefficient acts as a lower bound on the sample
complexity of any sharpening algorithm.
\begin{theorem}[Lower bound for sharpening]
  \label{thm:lower}
  Fix an integer $d \ge 1$ and parameters $\epsilon \in (0,1)$ and $C
  \ge 1$. There exists a class of models $\Pi$ such that (i) $\log |\Pi|
  \asymp d (1+\log(C \epsilon^{-1}))$, (ii) $\sup_{\pi \in \Pi}
  \Cstar(\pi) \lesssim C$, and (iii) $\Ypi(x)$ is a singleton for all
  $\pi\in\Pi$, $x\in\cX$. Any sharpening algorithm $\pihat$ that
  achieves $\En\brk*{\Pr_{x \sim \cdist}[\pihat(\Ypi[\piref](x)\mid{}x) > 1/2]}  \ge 1
  - \epsilon$ for all $\piref \in \Pi$ must collect a total number of
  samples $m = n\cdot{}N$ at least\loose
  \begin{align}
    m \gtrsim
    \frac{ C\log |\Pi|}{\epsilon^{2}\cdot{}(1+\log(C \epsilon^{-1}))}.
  \end{align}  
\end{theorem}
This result shows that the complexity of any $(\eps,1/2-\delta)$-sharpening algorithm (for $\delta>0$) in the
\framework framework must depend polynomially on the coverage
coefficient $\Cstar$, as well as the accuracy\arxiv{ parameter} $\eps$. The
lower bound also depends on the expressivity of $\piref$,
as captured by the model class complexity term $\log\abs{\Pi}$. We will show in \abedit{the sequel} that it is possible to match this lower bound. \dfedit{Note
  that this result also implies a lower bound for the general
    sharpening problem (i.e., general $\rself$), since \mlsharp is a
  special case.}

\begin{remark}[Relaxed notions of sharpening and coverage]
The notion of coverage in \cref{eq:cstar} is somewhat stringent, since
it requires that $\piref$ place large mass on $\Ystar(x)$ on average.
\iclr{In
\cref{sec:proof_preliminaries}, we introduce a more general and permissive
notion of \emph{approximate sharpening}
(\cref{def:sharpening_general}) which
leads to weaker coverage requirements, and use this to
give generalized versions of our main results.\loose
}
\arxiv{In
\cref{sec:proof_preliminaries}, we introduce a more general and permissive
notion of approximate sharpening (\cref{def:sharpening_general}),
which allows the model to sharpen toward approximate arg-max responses
(in the sense that
$\log\piref(y\mid{}x)\geq{}(1-\gamma)\max_{y\in\cY}\log\piref(y\mid{}x)$
for an approximation parameter $\gamma>0$). This notion of sharpening
leads to significantly weaker coverage requirements, and we state
generalized versions of all our main results which accommodate this in
the appendix.
}
\end{remark}

\dfedit{We close this section by noting that numerous
  recent works---focusing on inference-time computation---show that
standard language models exhibit favorable coverage \abedit{with
  respect to desirable responses}
\citep{brown2024large,snell2024scaling,wu2024empirical}. \akedit{We
  replicate these findings in our experimental setup
  in~\Cref{app:experiments}.} These works suggest that, despite the
exponentially large response space, the coverage coefficient $\Cstar$
may be small in standard language modeling tasks.}

\section{Analysis of Sharpening Algorithms}
\label{sec:theoretical_analysis}

Equipped with the sample complexity framework from
\cref{sec:theoretical_framework}, we now prove that the \sftalg and
\rlhfalg families of algorithms provably learn a sharpened model for
the maximum likelihood sharpening
objective under natural
statistical assumptions.

\dfedit{Throughout this section, we treat the model class $\Pi$ as a
  fixed, user-specified parameter. Our results---in the tradition of
  statistical learning theory---allow for general classes $\Pi$, and
  are agnostic to the structure beyond standard generalization arguments.
  }

\subsection{Analysis of \bestofnalg}
\label{sec:bestofn_theory}

Recall that when we specialize to the maximum-likelihood sharpening self-reward, the
\bestofnalg algorithm takes the form \iclr{$\pihatbon = \argmax_{\pi \in \Pi}\sum_{i=1}^{n}\log\piref (\ybon_i\mid{}x_i\givebase)$,}
\arxiv{\[
    \pihatbon = \argmax_{\pi \in \Pi}\sum_{i=1}^{n}\log\piref (\ybon_i\mid{}x_i\givebase),
\]}
where
$\ybon_i=\argmax_{j\in\brk{N}}\crl*{\log\piref(y_{i,j}\mid{}x_i\givebase)}$
for $y_{i,1},\ldots,y_{i,N}\sim\piref(\cdot\mid{}x_i)$.

To analyze \bestofnalg, we first make a realizability
  assumption. Let $\pin(x)$ be the distribution of the
random variable
$\ybon_N(x)\sim\argmax\crl*{\log\piref(y_i\mid{}x)\mid{}y_1,\ldots,y_N\sim\piref(x)}$.
\begin{assumption}\label{assumption:bon-realizability}
  The model class $\Pi$ satisfies $\pin\in\Pi$.
\end{assumption}
Our main guarantee for \bestofnalg is as follows.\loose
\iclr{
  \begin{theorem}[Sample complexity of \bestofnalg]
  \label{thm:bestofn}
  Let $\eps,\delta,\deltafail\in(0,1)$ be given, and suppose we
  set $n=c\cdot\frac{\log(\abs{\Pi}\deltaf^{-1})}{\delta\eps}$ and
  $\Nstar=c\cdot\frac{\Cstar \log(2\delta^{-1})}{\eps}$ for an
  appropriate constant $c>0$. Then with probability at least
  $1-\deltafail$, \bestofnalg produces a model $\pihat$ such that that $  \bbP_{x\sim\cdist}\brk*{\pihat(\Ystar(x)\mid{}x)\leq{}1-\delta}
  \leq \eps$, and has total sample complexity\footnote{We focus on
    finite classes for simplicity, following a convention in
    reinforcement learning theory
    \citep{agarwal2019reinforcement,foster2023foundations}, but our
    results extend to infinite classes through standard
     arguments.\loose}
\begin{align}
  \label{eq:bestofn_sample}
  m = O\left(\frac{\Cstar\log(\abs{\Pi}\delfail^{-1})\log(\delta^{-1})}{\delta\eps^2}\right).
\end{align}
\end{theorem}}
\arxiv{
  \begin{theorem}[Sample complexity of \bestofnalg]
  \label{thm:bestofn}
  Let $\delfail,\delta\in(0,1)$ be given, and suppose we set
  $N=\Nstar\log(2\delta^{-1})$ for a parameter $\Nstar\in\bbN$. If
  \cref{assumption:bon-realizability} holds, then
  for any $n\in\bbN$, \bestofnalg produces a model $\pihat$ such that with probability at least $1-\delfail$, 
  \begin{align}
    \bbP_{x\sim\cdist}\brk*{\pihat(\Ystar(x)\mid{}x)\leq{}1-\delta}
    \approxleq{} \frac{1}{\delta}\cdot{}\frac{\log(\abs{\Pi}\delfail^{-1})}{n} + \frac{\Cstar}{\Nstar}.
  \end{align}
In particular, given $(\eps,\delta)$, by setting
$n=c\cdot\frac{\log\abs{\Pi}}{\delta\eps}$ and $\Nstar=c\cdot\frac{\Cstar}{\eps}$ for an appropriate constant $c>0$,
we are guaranteed that $  \bbP_{x\sim\cdist}\brk*{\pihat(\Ystar(x)\mid{}x)\leq{}1-\delta}
  \leq \eps$, 
and have total sample complexity
\begin{align}
  \label{eq:bestofn_sample}
  m = O\left(\frac{\Cstar\log(\abs{\Pi}\delfail^{-1})\log(\delta^{-1})}{\delta\eps^2}\right).
\end{align}
\end{theorem}}
This result shows that \bestofnalg, via \cref{eq:bestofn_sample}, is minimax
  optimal in the \framework framework when $\delta$ is constant. In
particular, the sample
complexity bound in \cref{eq:bestofn_sample} matches the lower bound
in \cref{thm:lower} up to polynomial dependence on $\delta$ and
logarithmic factors. Whether the $1/\delta$ factor in
\cref{eq:bestofn_sample} can be removed is an interesting technical question,
but may not be practically consequential because---as discussed in \cref{sec:sample}---the regime $\delta<1/2$ is
most meaningful for autoregressive language modeling.

\begin{remark}[On realizability and coverage]\label{rmk:realizability}
  Realizability assumptions such as
  \cref{assumption:bon-realizability} (which asserts that the class $\Pi$ is powerful
  enough to model the distribution of the \bestofn responses) are
  standard in learning theory
  \citep{agarwal2019reinforcement,\arxiv{lattimore2020bandit,}foster2023foundations},
  though certainly non-trivial (see \cref{sec:hardness} for a natural
  example where they may not hold). The coverage assumption, while
  also standard, when combined with the hypothesis that
  high-likelihood responses are desirable, suggests that $\pibase$ generates high-quality responses with
  reasonable probability. In general, doing so may require leveraging non-trivial \emph{serial} computation at inference time via procedures such as Chain-of-Thought~\citep{wei2022chain}. Although recent work shows that such serial computation \emph{cannot} be amortized~\citep{li2024chain,malach2023auto}, \sftalg instead amortizes the \emph{parallel} computation of \bestofn sampling, and thus has different representational considerations. 
\end{remark}

\paragraph{Benefits of adaptive sampling}
\bestofnalg is optimal in the \framework
framework, but we show in \cref{sec:adaptive} that a variant which selects
the number of responses adaptively based on the prompt $x$ can bypass
this lower bound, improving the $\eps$-dependence in
\cref{eq:bestofn_sample} from $\frac{1}{\eps^2}$ to $\frac{1}{\eps}$.\loose

\iclr{
\paragraph{Empirical validation} In \Cref{app:experiments}, we empirically investigate the benefits of \bestofn on a variety of model-dataset pairs.  Our results are summarized in \Cref{tab:performance} and
  \cref{fig:training_curve_phi,fig:training_curve_mistral}, and broadly show that the benefits incurred through the inference-time sharpening described above can be, to a certain extent, amortized into training time.
}

\subsection{Analysis of \rlhfalg}
\label{sec:rlhf_theory}

We now turn our attention to theoretical guarantees for the \rlhfalg
algorithm family, which uses tools from reinforcement learning to optimize
the self-reward function. When specialized to \mlsharp, the  RL
objective used by \rlhfalg takes the form \iclr{$  \pihat \approx \argmax_{\pi\in\Pi}\crl*{
  \En_{\pi}\brk*{\log\piref(y\mid{}x)}
  -\beta\Dkl{\pi}{\piref}
  }$}
\arxiv{\begin{align}
  \label{eq:ml_rlhf}
  \pihat \approx \argmax_{\pi\in\Pi}\crl*{
  \En_{\pi}\brk*{\log\piref(y\mid{}x)}
  -\beta\Dkl{\pi}{\piref}
  }
\end{align}}
for $\beta>0$. The exact optimizer $\pistarb = \argmax_{\pi\in\Pi}\crl*{
  \En_{\pi}\brk*{\log\piref(y\mid{}x)}
  -\beta\Dkl{\pi}{\piref}
}$ for this objective has the form $
\pistarb(y\mid{}x)\propto\piref^{1+\beta^{-1}}(y\mid{}x)$,  which converges to a sharpened model (per \cref{def:sharpening}) as $\beta\to{}0$.\loose

The key challenge we encounter in this section is the mismatch between the RL reward $\log \pibase(y \mid x)$ and the sharpening desideratum $\pihat(\Ystar(x) \mid x)$. For example, suppose a unique argmax---say, $y^\star(x)$---and second-to-argmax---say, $y'(x)$---are nearly as likely under $\pibase$. Then the RL reward $\E_{\pihat}[\log \pibase(y \mid x)]$ must be optimized to extremely high precision before $\pihat$ can be guaranteed to distinguish the two. To quantify this effect, we introduce a 
\emph{margin condition}. 
\begin{assumption}[Margin]
  \label{ass:hard_margin}
  For a margin parameter $\gammargin>0$, the base model $\piref$ satisfies
  \[
    \max_{y\in\cY}\piref(y\mid{}x) \geq{}(1+\gammargin)\cdot\piref(y'\mid{}x)\quad\forall{}y'\notin\Ystar(x),\quad\forall{}x\in\supp(\mu).
  \]%
  \iclr{\vspace{-5pt}}
\end{assumption}
\sftalg does not suffer from the pathology in the example above, because once $y^\star(x)$ and $y'(x)$ are drawn in a batch of $N$ responses, we have $\ybon_i = y^\star(x_i)$ regardless of margin. However, as we shall show in \Cref{sec:benefits_exploration}, the \rlhfalg algorithm is amenable to online exploration, which may improve dependence on other problem parameters.

\subsubsection{Guarantees for \rlhfalg with Direct Preference Optimization}

The first of our theoretical results for \rlhfalg takes an
offline reinforcement learning approach, whereby we implement
\cref{eq:rlhf} using a reward-based variant of Direct Preference
Optimization (DPO) \citep{rafailov2024direct,gao2024rebel}. Let $\cDpref=\crl*{(x,y,y')}$ be a dataset of $n$ examples sampled via
$x\sim\cdist$, $y,y'\sim\piref(y\mid{}x)$. For a parameter $\beta>0$, we \arxiv{consider the algorithm
that solves}\iclr{solve $\pihat\in \argmin_{\pi\in\Pi}$\loose}
\begin{small}
  \begin{align}
    \label{eq:dpo_ml}
    \arxiv{\pihat\in \argmin_{\pi\in\Pi}}\sum_{(x,y,y')\in\cDpref}\prn*{
    \beta\log\frac{\pi(y\mid{}x)}{\piref(y\mid{}x)}-\beta\log\frac{\pi(y'\mid{}x)}{\piref(y'\mid{}x)}
    - \prn*{\log\piref(y\mid{}x)-\log\piref(y'\mid{}x)}
    }^2.
  \end{align}%
\end{small}%

\paragraph{Assumptions}
Per 
\iftoggle{workshop}{\cite{rafailov2024direct}}{\citet{rafailov2024direct}}, the solution to \cref{eq:dpo_ml} coincides
with that of \cref{eq:ml_sharpening} asymptotically. To provide
finite-sample guarantees, we make a number of statistical
assumptions. First, we make a natural realizability assumption (e.g., \iftoggle{workshop}{\cite{zhu2023principled,xie2024exploratory}}{\citet{zhu2023principled,xie2024exploratory}}).
\begin{assumption}[Realizability]
  \label{ass:realizability_dpo}
  The model class $\Pi$ satisfies $\pistarb\in\Pi$.\footnote{See \Cref{rmk:realizability} for a discussion of this assumption.}
\end{assumption}
Next, we define two concentrability coefficients for a
model $\pi$:
\begin{align}
  \label{eq:conc}
  \cC_{\pi} = \En_{\pi}\brk*{\frac{\pi(y\mid{}x)}{\piref(y\mid{}x)}},
  \mathand
\Cpp{\pi}{\pi'}\ldef{}\En_{\pi}\brk*{\prn*{\frac{\pi(y\mid{}x)}{\pi'(y\mid{}x)}}^{\beta}}.
\end{align}
The following result shows that both coefficients are bounded for the
KL-regularized model $\pistarb$.
\begin{lemma}
  \label{lem:conc_bound}
  The model $\pistarb$ satisfies $\cC_{\pistarb} \leq \Cstar$ and $\Cpp{\piref}{\pistarb} \leq{} \abs*{\cY}$.
\end{lemma}
Motivated by this result, we assume the coefficients in \cref{eq:conc}
are bounded for all $\pi\in\Pi$.
\begin{assumption}[Concentrability]
  \label{ass:conc-closs}
  All $\pi\in\Pi$ satisfy
  $\cC_{\pi} \leq \Cconc$
  for a parameter $\Ccon \geq \Cstar$, and
  $\Cpp{\piref}{\pi} \leq \Closs$
  for a parameter $\Closs \geq \abs{\cY}$.
\end{assumption}
By \cref{lem:conc_bound}, this assumption is consistent with
\cref{ass:realizability_dpo} for reasonable bounds on $\Cconc$ and
$\Closs$; note that our sample complexity bounds will only incur logarithmic dependence on $\Closs$.

\paragraph{Main result}
Our sample complexity guarantee for \rlhfalg (via \cref{eq:dpo_ml})
is as follows.\loose
      \begin{theorem}
        \label{thm:dpo}
        Let $\eps,\delta,\deltaf\in(0,1)$ be given. Set
        $\beta\approxleq\gammargin\delta\eps$, and suppose that
        \cref{ass:realizability_dpo,ass:conc-closs,ass:hard_margin} hold with parameters
        $\Cconc$, $\Closs$, and $\gammargin>0$. For an appropriate
        choice for $n$, the DPO algorithm (\cref{eq:dpo_ml}) ensures that with
        probability at least $1-\deltafail$, $                    \bbP_{x\sim\cdist}\brk*{\pihat(\Ystar(x)\mid{}x)\leq{}1-\delta}
        \leq \eps$, and has sample complexity
       \begin{align}
          m = \bigoht\prn*{\frac{\Cconc\log^{3}(\Closs\abs{\Pi}\deltafail^{-1})}{\gammargin^2\delta^2\eps^2}}.
        \end{align}
      \end{theorem}
Compared to the guarantee for \bestofnalg, \rlhfalg learns
a sharpened model with 
the same dependence on the accuracy $\eps$\abedit{, but a worse dependence on $\delta$; as we primarily consider $\delta$ constant (cf. \Cref{prop:greedy}), we view this as relatively unimportant. We further remark that \rlhfalg uses $N=2$ responses per
prompt, while \sftalg uses many ($N\approx \Cstar/\eps$) responses but fewer prompts. Other \arxiv{notable }differences include:\loose}
\begin{itemize}
  \item \rlhfalg requires the margin condition in
  \cref{ass:hard_margin}, and has sample complexity scaling with
  $\gammargin^{-1}$. We believe this dependence is natural for
  algorithms based on reinforcement learning, as it relates suboptimality with respect to the reward
  function $\rself(y\mid{}x)=\log\piref(y\mid{}x)$ (i.e.,
  $\En_{x\sim\mu}\brk*{\max_{y\in\cY}\log\piref(y\mid{}x)-\En_{y\sim\pihat(x)}\brk*{\log\piref(y\mid{}x)}}\leq\eps$,
  the objective minimized by reinforcement learning)
  to approximate sharpening error $\bbP_{x\sim\cdist}\brk*{\pihat(\Ystar(x)\mid{}x)\leq{}1-\delta}$. However, it is not clear if the precise dependence we pay is necessary. %
\item \rlhfalg requires a bound on the uniform coverage parameter
  $\Cconc$, which is larger than the parameter $\Cstar$ required by
  \bestofnalg in general. We expect that this assumption can be
  removed by incorporating pessimism in the vein of
  \citep{liu2024provably,huang2024correcting}. Also,
  \rlhfalg requires a bound on the parameter $\Closs$. This grants
  control over the range of the reward function
  $\log\piref(y\mid{}x)$, which can otherwise be unbounded. Since the
  dependence on $\Closs$ is only logarithmic, we view this as fairly
  mild. Overall, the guarantee in \cref{thm:dpo} may be somewhat
  pessimistic; it would be interesting if the result can be improved %
  to match the sample complexity
  of \bestofnalg.
\end{itemize}
 \arxiv{\dfc{Expand into longer remark on tradeoffs of $N$ vs $n$?}}

\subsubsection{Benefits of Exploration}\label{sec:benefits_exploration}
The sample complexity guarantees in \cref{thm:dpo} scale with the
coverage parameter $\Cstar = \En[1/\piref(\Ystar(x)|x)]$, which in
general is unavoidable in the sample-and-evaluate framework via our lower bound,~\cref{thm:lower}. Although $\Cstar$ is a problem-dependent parameter, in the worst case it can be as large as $\abs{\cY}$ (which is exponential in sequence length for autoregressive models). 
Fortunately, unlike
\bestofnalg, the \rlhfalg objective \eqref{eq:rlhf} is amenable to RL algorithms employing
active exploration, leading to improved sample complexity when the class $\Pi$ has additional structure. \loose

Our below guarantees for \rlhfalg replace the assumption of bounded coverage
with boundedness of a structural parameter for the model class
$\Pi$ known as the ``sequential extrapolation coefficient'' (SEC)
\citep{xie2023role,xie2024exploratory}, which we denote by $\SEC(\Pi)$. The formal definition is 
deferred to \cref{sec:proofs_exploration}. Conceptually,
$\SEC(\Pi)$ may thought of as a generalization of the eluder dimension
\citep{russo2013eluder,jin2021bellman}.  
It can always be bounded by the coverability coefficient of the model
class \citep{xie2024exploratory} and can be as large as $\Ccon$ in the worst case, so that bounds based on the SEC reflect improvements that are possible in favorable instances.

Beyond boundedness of the SEC, we require a bound on the range of the log-probabilities of $\piref$.\loose
\begin{assumption}[Bounded log-probabilities]\label{assumption:xpo-bdr}
For all $\pi \in \Pi$, $(x,y) \in \MX\times \MY$, $\abs[\big]{\log \frac{1}{\piref(y|x)}} \leq \Rmax$.\loose
\end{assumption}
\dfedit{We expect that the dependence on $\Rmax$ in our result can be
  replaced with $\log(\Closs)$ (\cref{ass:conc-closs}), but we omit
  this extension to simplify presentation.
  }

  We appeal to (a slight modification of)  \xpo, an iterative  language model alignment algorithm  due to
\cite{xie2024exploratory}. \xpo{} is based on the objective in
\cref{eq:dpo_ml}, but unlike DPO,  incorporates a bonus term to
encourage exploration to leverage \textbf{online} interaction. See \cref{sec:proofs_exploration} for a
detailed overview.

\arxiv{\paragraph{Main result}}

\arxiv{The main guarantee for \rlhfalg with \xpo is as follows.}{}

\begin{theorem}[Informal version of \cref{thm:xpo-sharpening-apx}]\label{thm:xpo-sharpening}
Suppose that \Cref{ass:hard_margin,assumption:xpo-bdr} hold with parameters $\gammargin,\Rmax>0$, and that
\Cref{ass:realizability_dpo} holds with $\beta 
= \gammargin/(2\log(2|\MY|/\delta))$. For any $m\in\bbN$ and $\deltafail\in(0,1)$, \xpo
(\cref{alg:xpo}), when configured appropriately, produces an
$(\eps,\delta)$-sharpened model $\pihat \in \Pi$ with probability at
least $1-\deltafail$, and uses sample complexity\footnote{Technically, \cref{alg:xpo} operates in a slight generalization of the sample-and-evaluate framework (\cref{def:oracle-model}), where the algorithm is allowed to query $\piref(y\mid{} x)$ for arbitrary $x,y$. We expect that our lower bound (\cref{thm:lower}) can be extended to to show that dependence on $\Cstar$ is necessary in the worst case even in this more general framework. \cref{alg:xpo} is fundamentally using additional instance-dependent structure (via the SEC) to avoid dependence on the coverage parameter, $\Cstar$.} \[m = \bigoht\prn*{\frac{\SEC(\Pi)\cdot{}\log(\abs{\Pi}\deltafail^{-1})}{\gammargin^2\delta^2\eps^2}}.\]

\end{theorem}
The takeaway from \cref{thm:xpo-sharpening} is that there is no
dependence on the coverage coefficient for $\piref$. Instead,
the rate depends on the complexity of exploration, as governed by the
sequential extrapolation coefficient $\SEC(\Pi)$. \arxiv{We emphasize that while we
present guarantees for \xpo under the \arxiv{sequential extrapolation
coefficient}\iclr{SEC} for concreteness, we}\iclr{We} expect similar guarantees can
derived for other active exploration algorithms and complexity
measures \citep{jiang2017contextual,foster2021statistical,jin2021bellman,xie2023role}.\loose

\arxiv{
\paragraph{Example: Linearly parameterized models}
As a stylized example of a model class $\Pi$ where active exploration
dramatically improves the sample complexity of sharpening, we consider
the class $\Pi_{\phi,B}$ of linear softmax models. This class
consists of models of the form \iclr{$\pi_\theta(y\mid{}x) \propto \exp(\langle
  \phi(x,y),\theta\rangle)$,}
\arxiv{\begin{align}
\pi_\theta(y\mid{}x) \propto \exp(\langle
  \phi(x,y),\theta\rangle),\label{eq:softmax}
       \end{align}
       }
where $\theta\in\bbR^{d}$ is a parameter vector with $\norm{\theta}_2
\leq B$, and $\phi(x,y)\in\bbR^{d}$ is a known feature map with
$\nrm*{\phi(x,y)}\leq{}1$. The sequential extrapolation coefficient
for this class can be bounded as $\SEC(\Pi)=\wt O(d)$, and the optimal
KL-regularized model $\pistarb$ is a linear softmax model (i.e.,
$\pistarb\in\Pi$) whenever the base model $\piref$ is itself a linear
softmax model. This leads to the following result.

\begin{theorem}\label{thm:xpo-softmax}
  Fix $\epsilon,\delta,\deltafail \in (0,1)$ and $B>0$. Suppose that (i) $\piref = \pi_{\theta^\star}$ is a
  linear softmax model with $\norm{\theta^\star}_2 \leq
  \frac{\gammargin B}{3\log(2|\MY|/\delta)}$; (ii) $\piref$ satisfies
  \cref{ass:hard_margin} with parameter $\gammargin$. \cref{alg:xpo},
  with \arxiv{base model $\piref$, }reward function $r(x,y) :=
  \log \piref(x,y)$, and model class $\Pi_{\phi,B}$, returns an
  $(\epsilon,\delta)$-sharpened model with
  \arxiv{probability}\iclr{prob.}\arxiv{ at least} $1-\rho$, and with sample complexity $m=\poly(\epsilon^{-1},\delta^{-1},\gammargin^{-1}, d, B, \log(|\MY|/\deltafail))$.\loose
\end{theorem}
Importantly, \cref{thm:xpo-softmax} has no dependence on the coverage
parameter $\Cstar$, scaling only with the dimension $d$ of the
softmax model class. \loose

For a quantitative comparison,\arxiv{ we note that even for the simple
special case of the linear softmax
model class,} it is straightforward to construct examples of models $\piref$ where
$\Cstar = \EE[1/\piref(\ystar(x)|x)] \asymp |\MY| \asymp \exp(\Omega(d))$, yet
\cref{ass:hard_margin} is satisfied with $\gammargin=\Omega(1)$. For such models, \bestofnalg will incur
$\exp(\Omega(d))$ sample complexity; see \cref{ex:softmax} for details. Hence, \cref{thm:xpo-softmax} represents an \emph{exponential} improvement, obtained by exploiting the structure of the self-reward function in
a way that goes beyond \bestofnalg.\loose

\begin{remark}[Non-triviality]
  \cref{thm:xpo-softmax} is quite stylized in the sense that if the parameter vector $\theta^\star$ of $\piref$ is known, then it
is trivial to directly compute the parameter vector for the sharpened
model $\pistarb$ (which corresponds to rescaling $\theta^{\star}$). \cref{alg:xpo} is interesting and
non-trivial nonetheless because it \emph{does not have explicit knowledge of
  $\theta^\star$}, as it operates in the sample-and-evaluate oracle
model (\cref{def:oracle-model}). \arxiv{Moreover, the guarantee generalizes to
any model class $\Pi$ for which $\SEC(\Pi)$ can be bounded; see \cref{thm:xpo-sharpening-apx} for the formal statement.\loose}
\end{remark}
}

\arxiv{
\section{Experiments}
\label{sec:experiments}

  In this section we explore the sharpening mechanism empirically.
  We consider inference-time experiments that demonstrate that self-improvement through sharpening is possible, as well as training-time experiments that successfully amortize the cost of self-improvement, thereby avoiding computational overhead at inference time.  We first describe the general experimental setup, then turn to the results of our experiments.

\subsection{Experimental Setup}

We experiment with sharpening using the following models, all of which (except for \gptthree) are available on \url{https://huggingface.co}; we provide HuggingFace model identifiers below.
\begin{enumerate}
\item Phi models: We use several models from the Phi family~\citep{abdin2024phi}, specifically \phithreemini (``microsoft/Phi-3-mini-4k-instruct''), \phithreesmall (``microsoft/Phi-3-small-8k-instruct''), \phithreemedium (``microsoft/Phi-3-medium-4k-instruct''), and \phithreefivemini (``microsoft/Phi-3.5-mini-instruct''). 
\item \llamathree (``meta-llama/Llama-3.2-3B-Instruct'')~\citep{dubey2024llama}.
\item \mistral (``mistralai/Mistral-7B-Instruct-v0.3'')~\citep{jiang2023mistral}.
\item \gptthree~\citep{brown2020language}: We access this model via the OpenAI API.
\item \llamagame (``OhCherryFire/llama2-7b-game24-policy-hf''): We use the model of \iftoggle{workshop}{\cite{wan2024alphazero}}{\citet{wan2024alphazero}}, which is a Llama-2 model finetuned on the \gametwentyfour task \citep{yao2024tree}. We use this model only for experiments with \gametwentyfour. 
\end{enumerate}

We consider the following tasks:

\begin{enumerate}
      \item \gsm: We use the above models to generate responses to prompts from the GSM-8k dataset \citep{cobbe2021training} where the goal is to generate a correct answer to an elementary school math question. For inference-time experiments, we take the first 256 examples from the test set in the ``main'' subset.\footnote{\url{https://huggingface.co/datasets/openai/gsm8k}.} %
    \item \mathdataset: We use the above models to generate responses to prompts from the \mathdataset dataset \citep{hendrycks2021measuring}, which consists of more difficult math questions.  For inference-time experiments, we consider ``all'' subsets and take the first 256 examples of the test set where the solution matches the regular expression \verb|(\d*)|.\footnote{\url{https://huggingface.co/datasets/lighteval/MATH}.}  %
    \item \prontoqa: We use the above models to generate responses to prompts from the \prontoqa dataset \citep{saparov2023language}, which consists of chain-of-thought-style reasoning questions with boolean answers.  For inference-time experiments, we take the first 256 examples from the training set.\footnote{\url{https://huggingface.co/datasets/longface/prontoqa-train}.} %
    \item \mmlu: We use the above models to generate responses to prompts from three subsets of the \mmlu dataset \citep{hendrycks2020measuring}, specifically \texttt{college\_biology} (Bio), \texttt{college\_physics} (Phys), and \texttt{college\_chemistry} (Chem), all of which consist of multiple choice questions.\footnote{\url{https://huggingface.co/datasets/cais/mmlu}.} For inference-time experiments, we take the first 256 examples of the test set for each subset. 
    \item \gametwentyfour: We use only the model of \iftoggle{workshop}{\cite{wan2024alphazero}}{\citet{wan2024alphazero}} (i.e., \llamagame), on the \gametwentyfour task \citep{yao2024tree}.  The prompts are four numbers and the goal is to combine the numbers with standard arithmetic operations to reach the number `24.'  For inference-time experiments, we use both the train and test splits of the dataset.\footnote{\url{https://github.com/princeton-nlp/tree-of-thought-llm/tree/master/src/tot/data/24}}  %
    \end{enumerate}

    All of our experiments were run on 40G NVIDIA A100 GPUs, 192G AMD MI300X GPUs, or through the OpenAI API.

    \subsection{Validation of Inference-Time Sharpening}\label{ssec:inference}

    We first validate the sharpening mechanism (i.e., the phenomenon that responses from a model with high self-reward $\rself$ enjoy high performance on downstream tasks) through inference-time experiments, focusing on the maximum likelihood self-reward. For each (model, task) pair, we sample $N$ generations per prompt with temperature 1 and return the best of the $N$ generations according to the \mlsharp self-reward function $\rself(y\mid{}x)=\log\piref(y\mid{}x)$; we compare against greedy decoding as a baseline, whose accuracy is displayed in \Cref{sfig:bon-greedy}.

\paragraph{Implementation details}
For all models and datasets except for \gametwentyfour, we used 1-shot prompting to ensure that models conform to the desired output format and to elicit chain of thought reasoning (for \gametwentyfour we do not provide a demonstration in the prompt). We set the maximum length of decoding to be $512$ tokens. We used 10 seeds for all (model, task) pairs with a maximum value of $N=50$ in \BestofN sampling. We simulated $N$ responses for $N<50$ by subsamplng the 50 generated samples. For \BestofN sampling, we always use temperature $1.0$. Since greedy decoding is a deterministic strategy, there is no need to average over multiple seeds for each (model, task) pair. In all experiments, we collect both the responses and their log-likelihoods under the \emph{reference model} (i.e., the original model from which samples were generated). 

\paragraph{Results}
We display our findings in \Cref{fig:validation}(a) and in \Cref{fig:BoN-inference-granular}; because we only consider a single model for \gametwentyfour, we separate the results for this task into \Cref{fig:game24_inf}. We visualize performance---measured through normalized accuracy improvement over greedy decoding. We also visualize log-likelihoods (under $\piref$) of the selected responses in \cref{fig:BoN-inference-logprobs}. We find that:
\begin{enumerate}
\item Across all (model, task) pairs, inference-time \BestofN
  sharpening (using $\rself(y\mid{}x)=\log\piref(y\mid{}x)$) improves
  over na{\"i}ve sampling with temperature 1.0.
\item For all datasets, \BestofN sharpening improves upon standard \emph{greedy decoding}, for at least one model.
  \item Analogously, for every model, there is at
  least one dataset for which \BestofN sharpening improves over greedy
  decoding.
\end{enumerate}

We further explore the relationship between sequence-level log-probabilities and generation quality in \Cref{fig:distributions}, where we plot the empirical distributions of responses sampled with temperature 1 from the base model for a variety of model-dataset pairs, conditioned on whether or not the response is correct. We find that \emph{the distribution of log probabilities conditioned on correctness stochastically dominates the distribution conditioned on incorrectness} in each (model, task) pair evaluated, which provides more evidence that maximum likelihood sharpening represent a reasonable self-improvement target.

We mention several other observations from the experiments. First, in most cases, performance and log-likelihood saturate at relatively small values of $N$, typically around 10 or 20. This suggests that significant improvements can be obtained with relatively low computational overhead. Second, in some cases, performance can degrade as $N$ increases. We found that this happens for two reasons: (1) the performance of the reference model is poor and so $\rself$ does not provide a good signal (e.g., with \llamathree) and (2) the \BestofN criteria selects for short responses, which have higher log-likelihood but cannot leverage the computational and representational benefits of chain-of-thought, thereby yielding worse performance (e.g., with \gptthree on \gsm).

    \begin{figure}[tp]
    \centering
    \subfigure[]{
    \includegraphics[width=0.45\textwidth]{figs/BoN-math-improvement.pdf}
    }
    \hfill
    \subfigure[]{
    \includegraphics[width=0.45\textwidth]{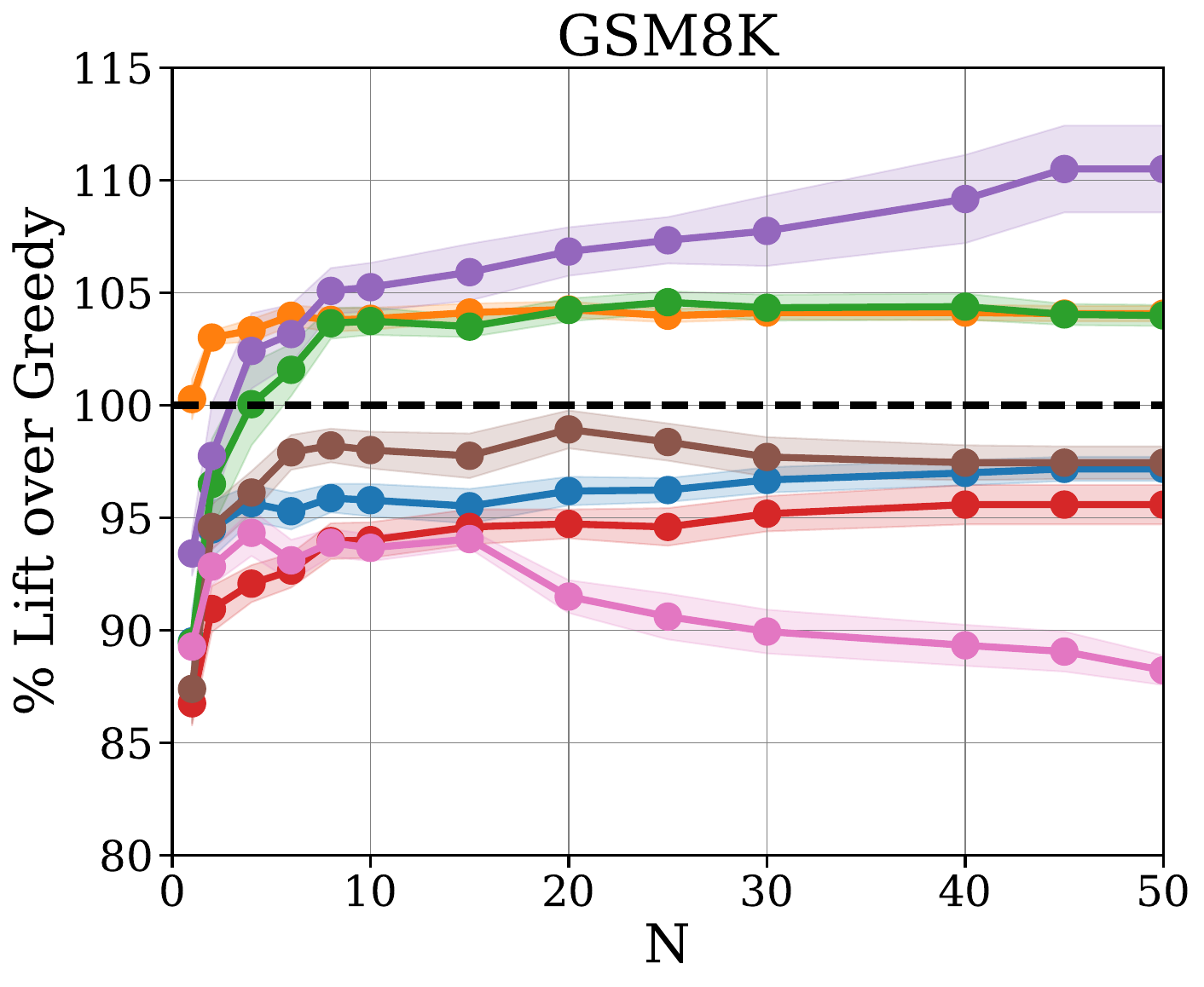}
    }
    \\
    \subfigure[]{
    \includegraphics[width=0.45\textwidth]{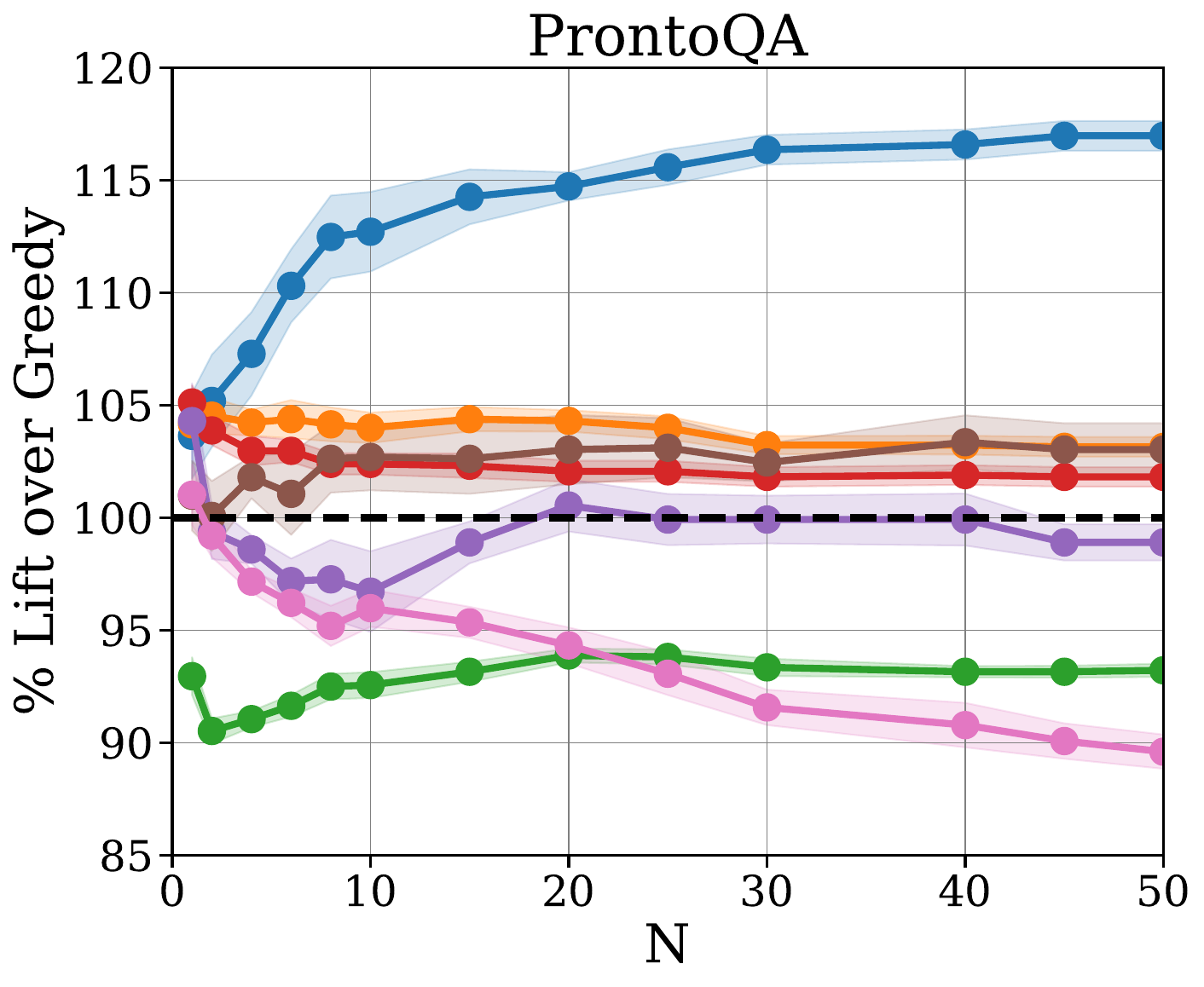}
    }
    \hfill
    \subfigure[]{
    \includegraphics[width=0.45\textwidth]{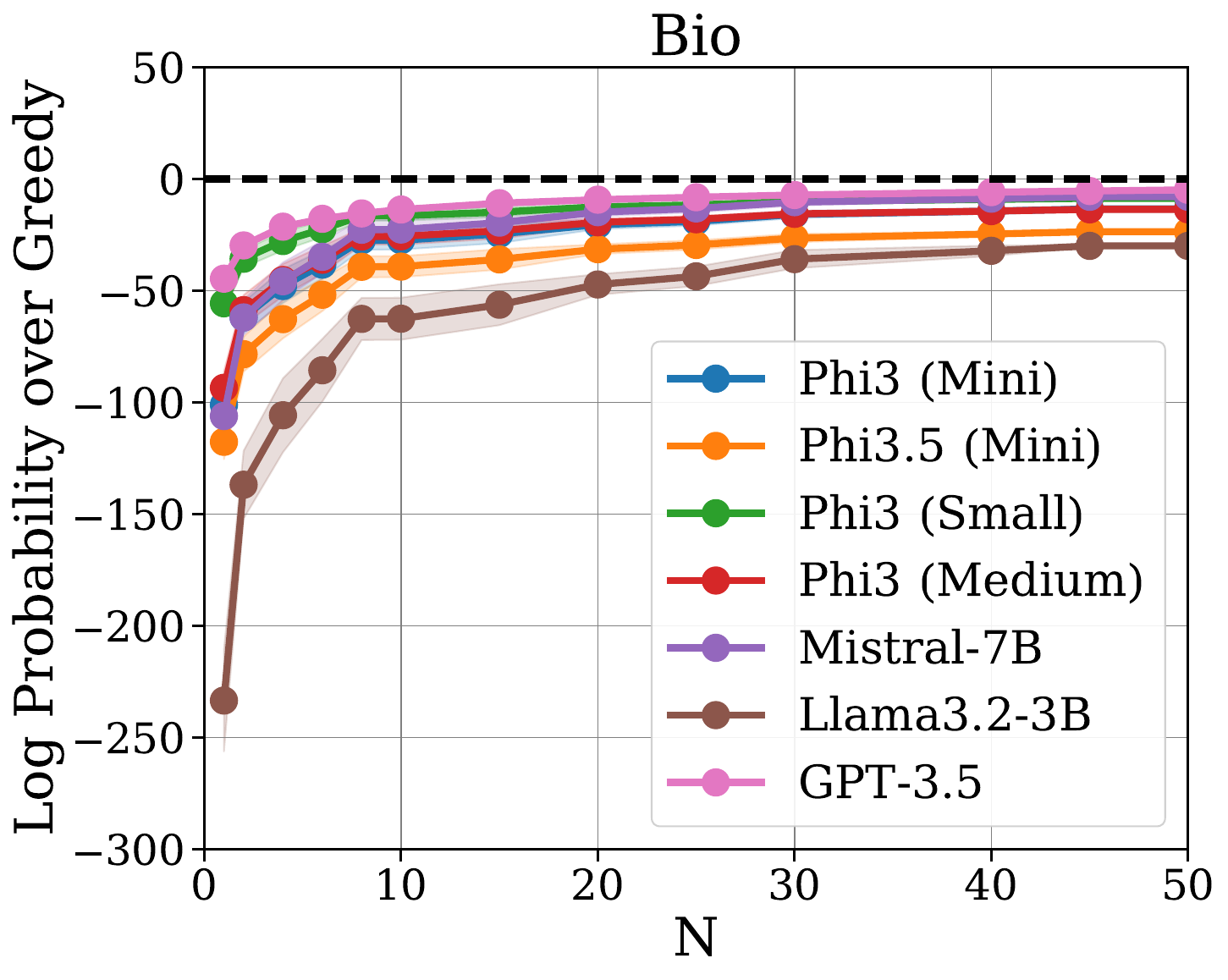}
    }
    \\
    \subfigure[]{
    \includegraphics[width=0.45\textwidth]{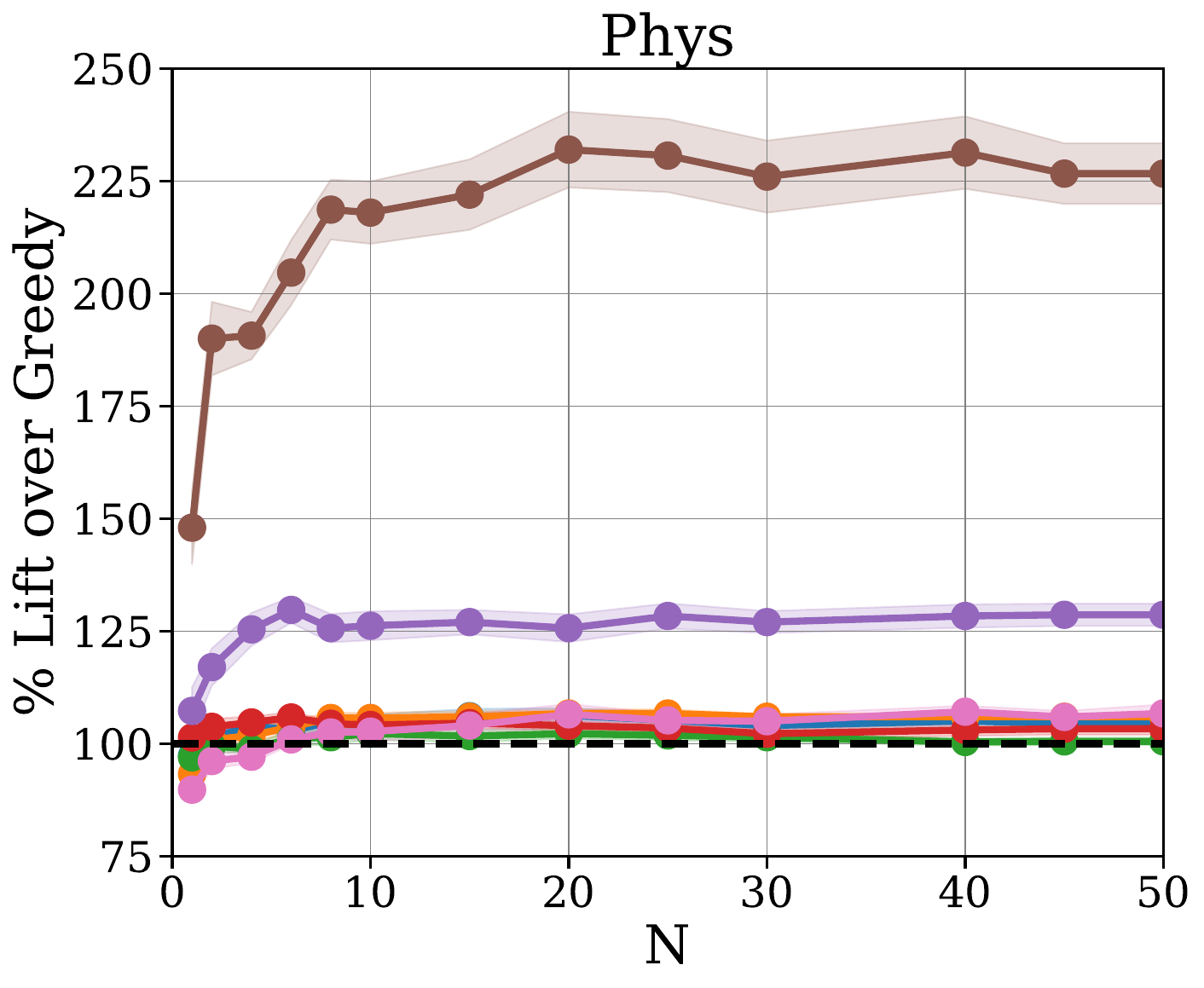}
    }
    \hfill
    \subfigure[]{
    \includegraphics[width=0.45\textwidth]{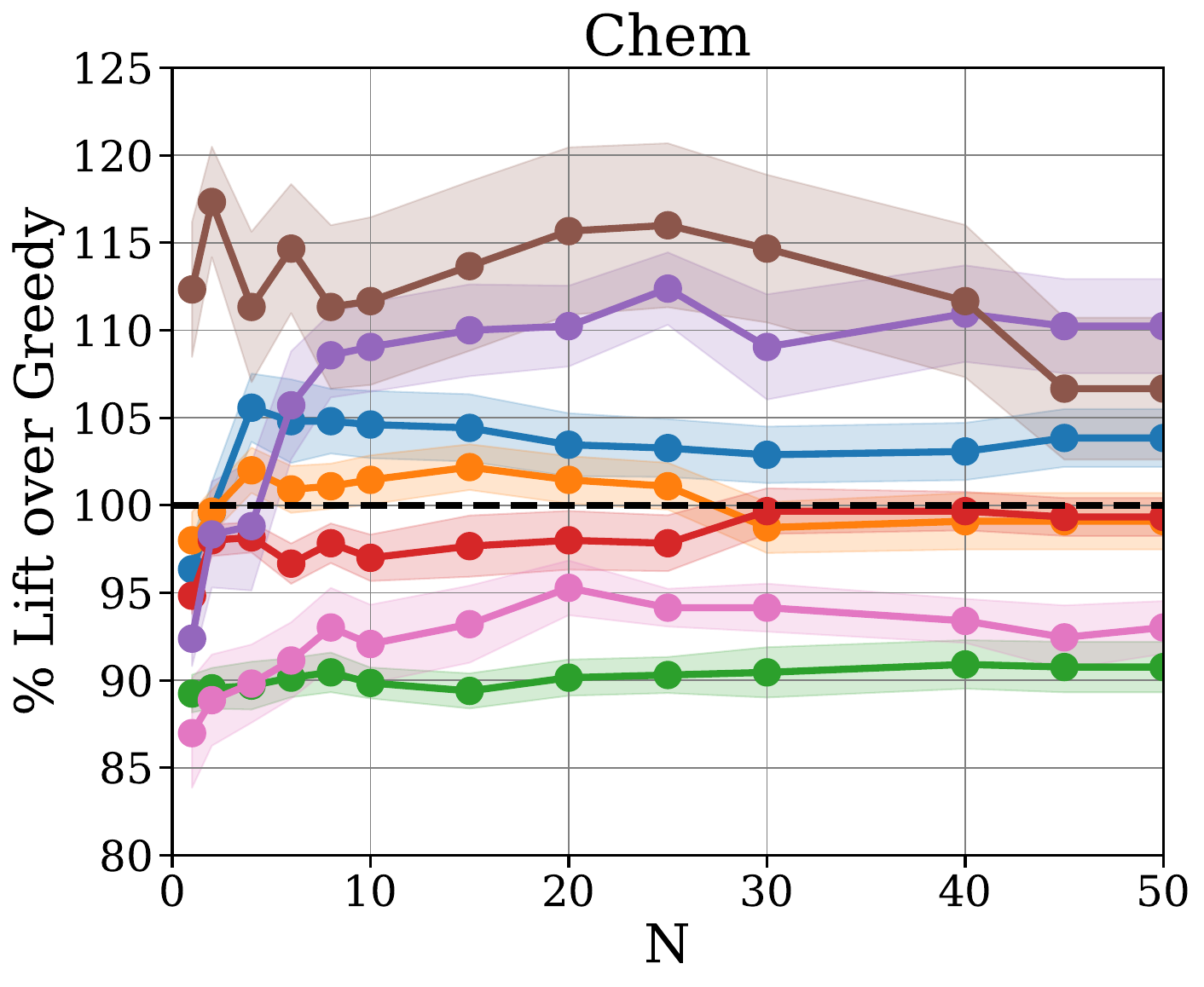}
    }
    \caption{Percent lift in accuracy of inference-time BoN-sharpening over greedy decoding in each task as $N$ is varied.  For many task-model pairs, the accuracy improves as $N$ increases, demonstrating the efficacy of maximum likelihood sharpening.}
    \label{fig:BoN-inference-granular}
\end{figure}

\subsection{Inference-Time Sharpening with other Self-Reward Functions}
Although we focus on $\rself(y \mid{}x) = \log\piref(y \mid{}x)$ throughout the paper, the sharpening framework is significantly more general.\footnote{Although we do not experiment with model-as-judge approaches~\citep[e.g.,][]{huang2022large}, which obtain self-reward by re-prompting, we note that they can be cast in the sharpening framework by simply defining the self reward $\rself(y \mid{} x)$ to be the output of the model when prompted with a verification/scoring prompt, the original prompt $x$ and the candidate response $y$.} As such, we experiment with inference-time sharpening for other choices for $\rself$:\loose
\begin{enumerate}
  \item Length-normalized log-likelihood: $\rself(y \mid{}x) = \frac{1}{|y|} \log\piref(y \mid{}x)$ where $|y|$ is the length, in tokens, of the response.
  \item Majority (self-consistency): All datasets except \gametwentyfour have multiple-choice, boolean, or numerical answers. Although we allow responses to contain chain-of-thought tokens, we can extract the answer from each response and use the most-frequently-occuring answer. This can be seen as a sample-based approximation to the following self-reward function: $\rself(y\mid{}x) = \sum_{y': y'_{\texttt{ans}}=y_{\texttt{ans}}} \piref(y'\mid{}x)$, where $y_{\texttt{ans}}$ are the ``answer'' tokens in the full response $y$. 
\end{enumerate}
Finally, as a skyline we consider the \emph{coverage} criterion~\citep{brown2024large}, where we simply check if any of the sampled responses corresponds to the correct answer. This criterion is a skyline and does not fit into sharpening framework, as it uses knowledge of the ground truth (external) task reward function. 

Results are displayed in \Cref{fig:other_rewards}. For length-normalized log-likelihood (a) and majority (b), we see qualitatively similar behavior to (unnormalized) log-likelihood: inference-time sharpening via these self-reward functions offers improvement over both vanilla (temperature 1.0) sampling and greedy decoding. In both cases, the improvements are generally larger than those obtained with log-likelihood. Finally, examining the coverage criteria, we see that with $N=50$ samples, all of the models almost always produce a correct answer on all tasks, raising the possibility of other self-reward functions that further improve performance.

\begin{figure}[t]
    \centering
    \subfigure[]{
        \includegraphics[width=0.45\textwidth]{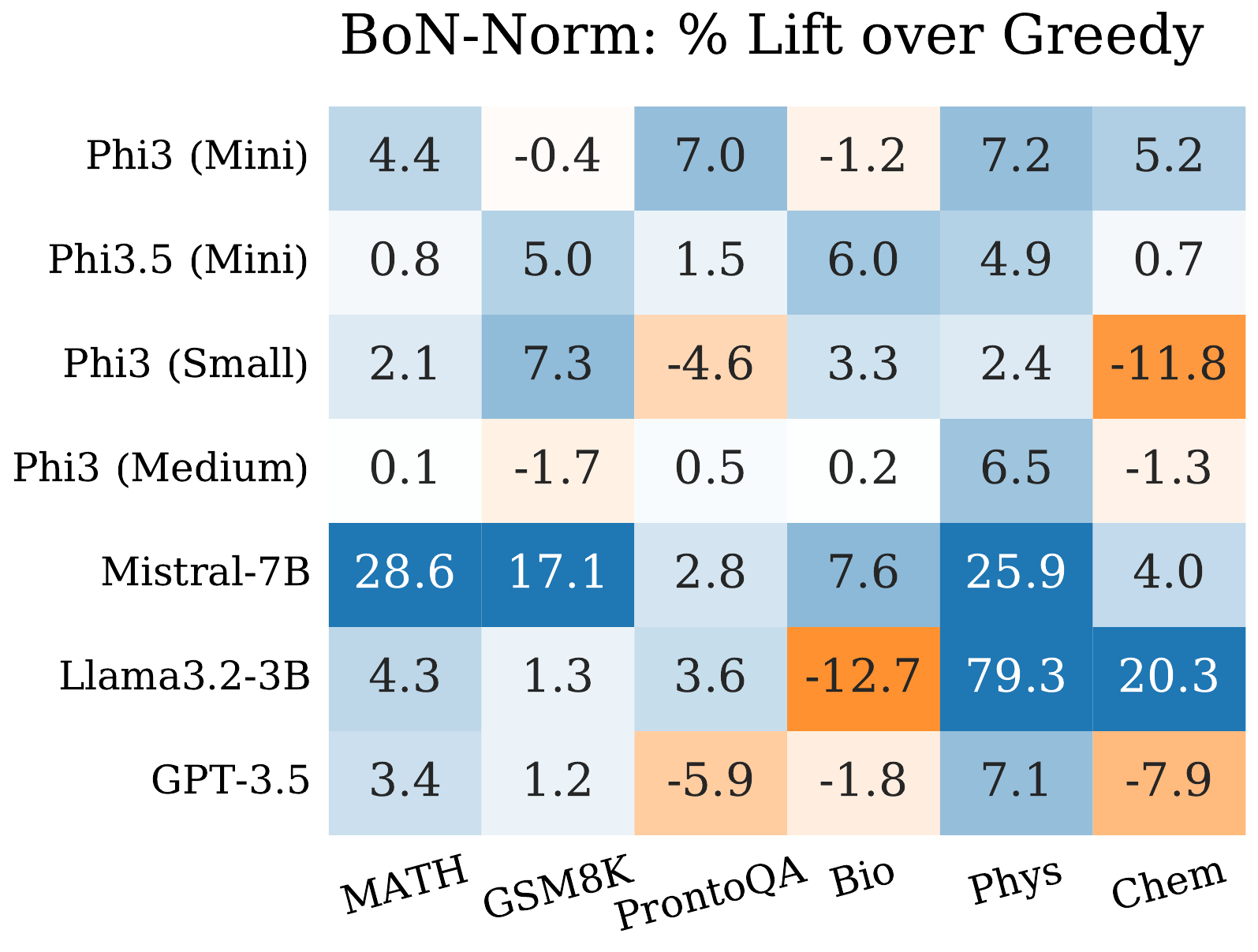}
        \label{sfig:bon-normalized-improvement} 
    }
    \hfill \subfigure[]{
        \includegraphics[width=0.45\textwidth]{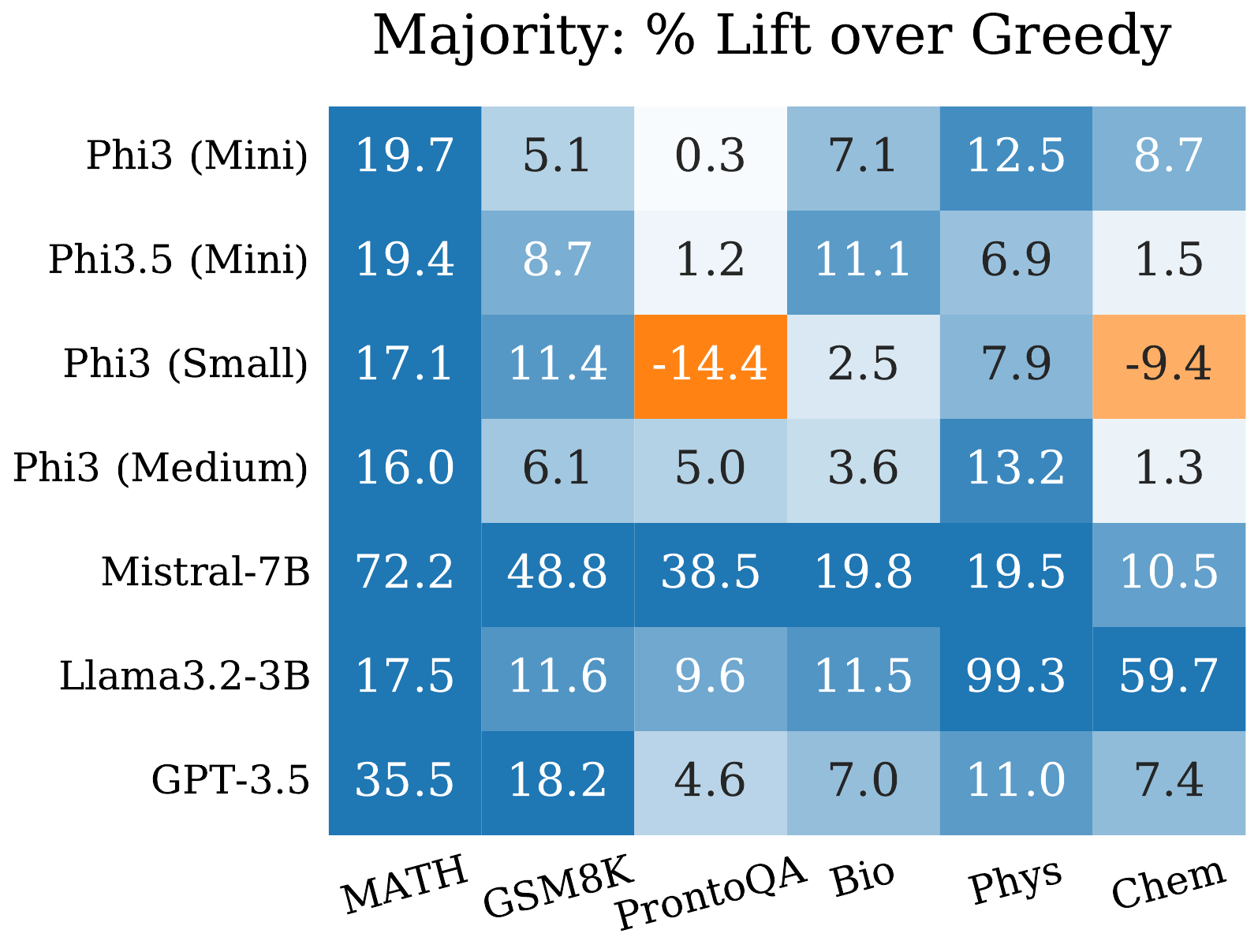}
        \label{sfig:bon-majority-improvement}
    } \\
  
    \subfigure[]{
        \includegraphics[width=0.45\textwidth]{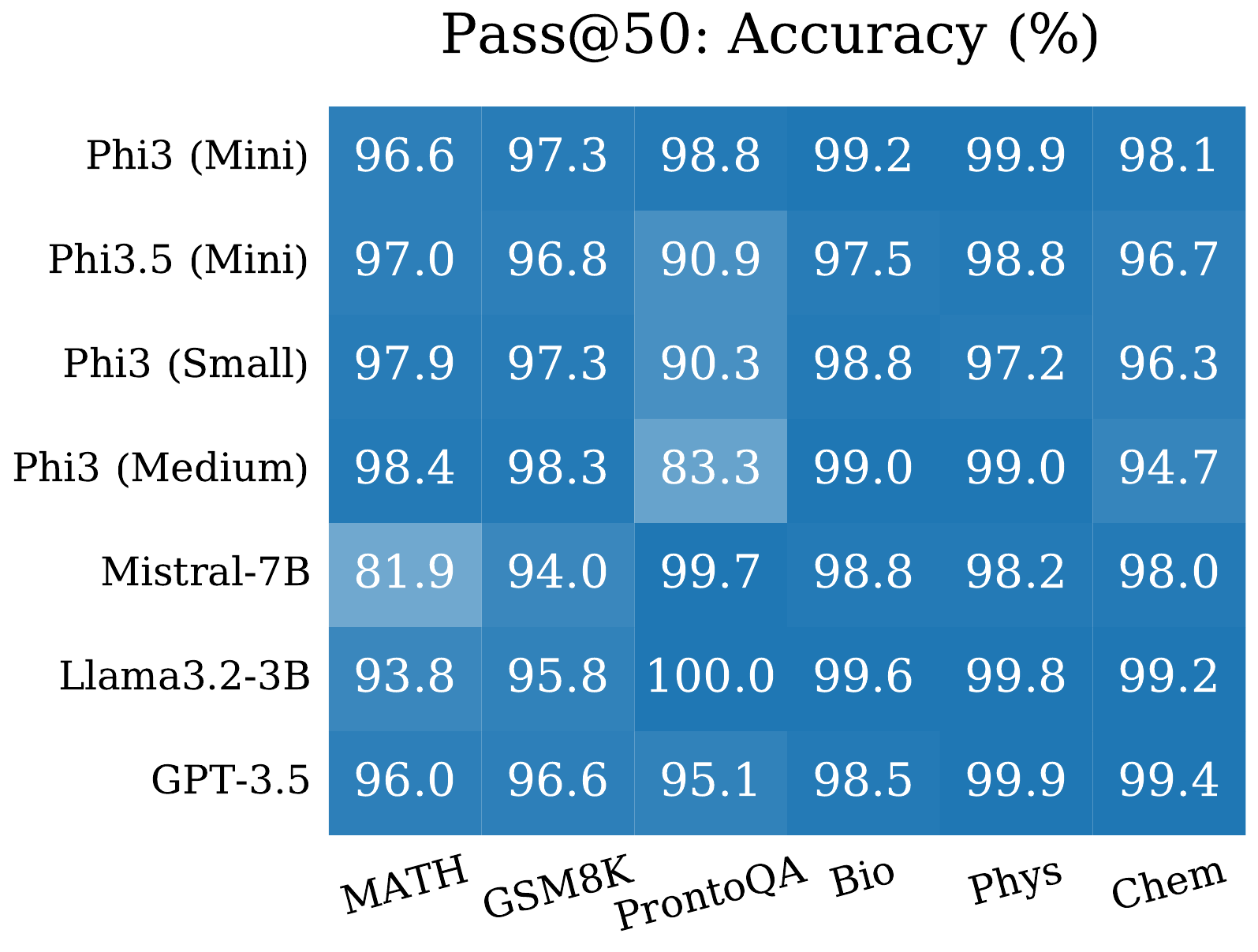}
        \label{sfig:bon-coverage}
    }
    \hfill
    \subfigure[]{
        \includegraphics[width=0.45\textwidth]{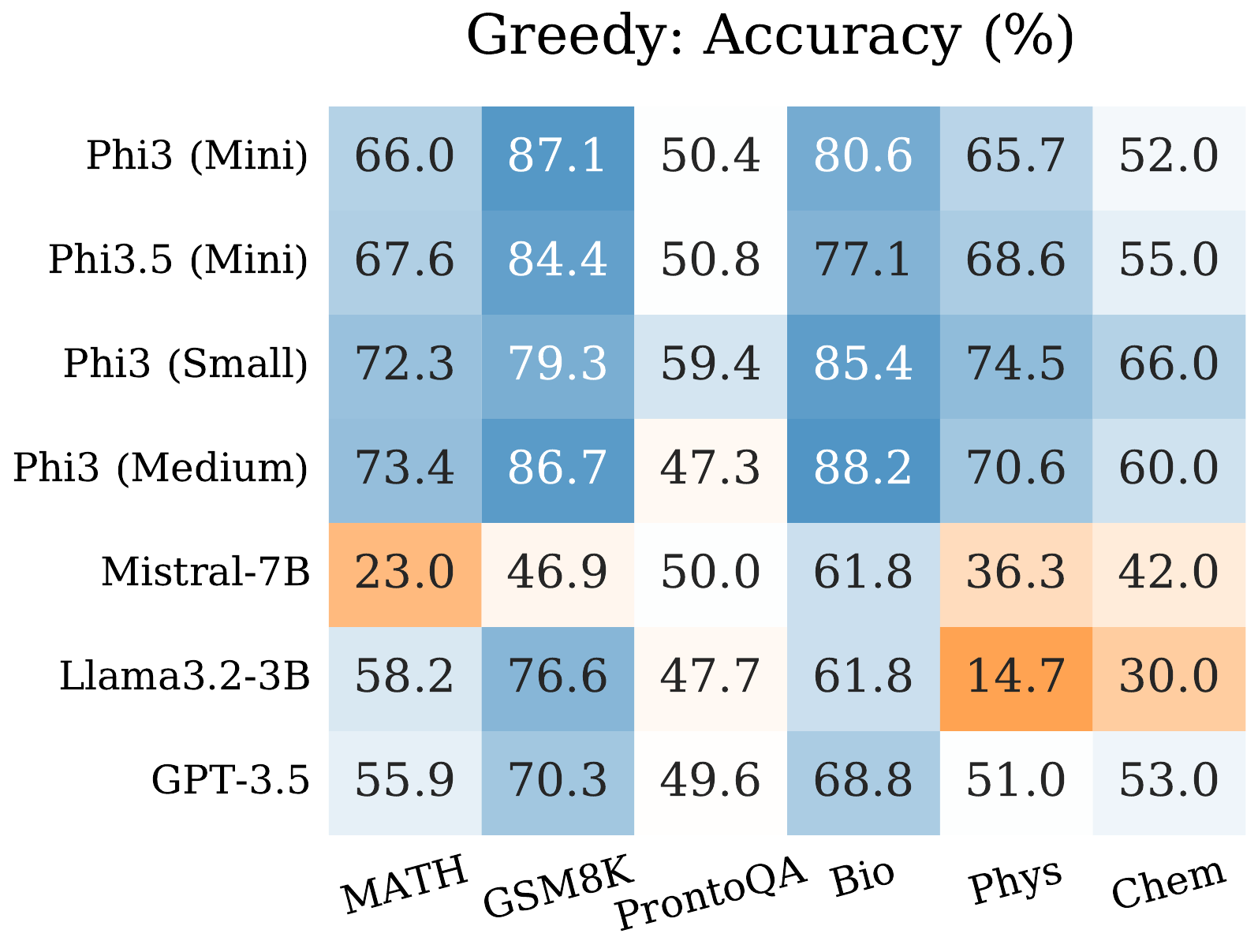}
        \label{sfig:bon-greedy}
    }

    \caption{Performance of alternative self-reward functions for inference-time BoN sharpening:  Percent accuracy improvement over greedy decoding for (a) BoN sharpening with length-normalized log probability and (b) majority voting, with both demonstrating efficacy on a range of model-task pairs.  (c) Coverage of correct answer, a skyline demonstrating that most model-task pairs produce the correct answer in at least one completion out of 50 for most prompts.  (d) Accuracy of greedy decoding baseline on each model-task pair.
    }
    \label{fig:other_rewards}
  \end{figure}

\subsection{Training-Time Sharpening with \bestofnalg}
\begin{table}[t]
    \centering
    \begin{tabular}{cccc}
        \toprule
        \textbf{Model} & \textbf{Dataset} & \textbf{\% Lift over Greedy (Accuracy)} & \textbf{Lift over Greedy (Likelihood)} \\
        \midrule
        \phithreefivemini & \mathdataset & $19.24 \pm 2.41$ & $48.33 \pm 0.17$ \\
        \phithreefivemini & \gsm & $1.82 \pm 0.64$ & $1.49 \pm 0.55$ \\
        \phithreefivemini & \prontoqa & $12.46 \pm 1.08$ & $ 5.64 \pm 0.01$ \\
        \texttt{Mistral-7B} & \mathdataset & $8.88 \pm 5.55$ & $5.71 \pm 3.00$ \\
        \bottomrule
    \end{tabular}
    \caption{Experimental results for \bestofnalg}
    \label{tab:performance}
\end{table}
In addition to inference-time experiments, we also evaluate training-time sharpening, and demonstrate empirically that \bestofnalg effectively amortizes inference-time BoN.  Due to limited computational resources, we restrict our attention to a subset of the model-task pairs considered in \Cref{ssec:inference} that have particularly promising inference-time BoN performance. For each pair, we evaluate the performance of \bestofnalg as a means to amortize the inference-time cost of multiple generations.

For each of the chosen model-dataset pairs (cf. \Cref{tab:performance}), we sample $N=50$ responses with temperature 1 for each prompt in the dataset and select the most likely (according to the relevant reference model).  We then combine these likely responses with the prompts in order to form a training corpus and train with the \bestofnalg objective. We apply Low Rank Adaptation \citep{hu2021lora} to the model, sweeping over LoRA rank, learning rate scheduler, and weight decay in order to return the best optimized model.\footnote{In all experiments involving \phithreefivemini we use a batch size of 4; unfortunately, due to a known numerical issue with LoRA on \mistral involving batch size $>1$, we use a batch of 1 in this case.  Because of this choice, instead of the 30 epochs we use to train our other models, for \mistral, we run only 10 epochs.}  We report the specific hyperparameters chosen in \Cref{tab:hyperparameters}.  On all models, we used a learning rate of $3 \times 10^{-4}$ with linear decay to zero and gradient clamping at 0.1. \loose

\paragraph{Results}  \Cref{tab:performance} reports our results for \bestofnalg. We report best model checkpoint during training for each model-dataset pair, averaged across 3 random seeds, with responses are sampled with temperature 1 from the fine-tuned model.  We report (i) the percent lift in accuracy on the dataset with respect to the greedy generation of the reference model; and (ii) the increase in average sequence level log-likelihood with respect to the same. For all model-dataset pairs, we observe improvement on both metrics, demonstrating that some amortization is possible with \bestofnalg.

Next, in \Cref{fig:training_curve_phi,fig:training_curve_mistral} (\cref{sec:additional_experiments}), we display the evolution of the metrics in \cref{tab:performance} throughout training for each model-dataset pair.  In \cref{fig:training_curve_phi}, we find that while \phithreefivemini is quite well-behaved on \mathdataset and \prontoqa, the training curve for \gsm is quite noisy. The log-probability appears to be a significantly less useful proxy for accuracy on this dataset than for the others; similar phenomena were observed in \citet{block2023butterfly} in a variety of tasks.  In \cref{fig:training_curve_mistral}, we find that for \mistral on \mathdataset, we achieve improvement after training for sufficiently long, but the optimization suffers an substantial initial drop and spends $\sim90\% $ of the gradient steps recovering before improvement is observed; we speculate that this is a function of insufficient hyper-parameter tuning for the optimization itself, rather than a fundamental barrier.

Finally, in \Cref{fig:sft-N} (\cref{sec:additional_experiments}), we investigate the effect of the parameter $N$ on the performance of \bestofnalg for \phithreefivemini on \mathdataset.  In particular, in forming our training set, we choose $N \in \{10, 25, 50 \}$ and repeat the procedure described above, averaging our results over three seeds.  We find that increasing $N$ leads to a modest increase in the sequence-level log-likelihood, in accordance with our theory, and a consequent increase in the accuracy of the fine-tuned model.

}

\section{Conclusion}
\label{sec:conclusion}
We view our
theoretical framework for sharpening as a starting point toward a foundational
understanding of self-improvement that can guide the design and evaluation of algorithms. To this end, \abedit{we raise several} directions
for future research.
\begin{itemize}
\item \emph{Representation learning.} A conceptually appealing feature
  of our framework is that it is agnostic to the
  structure of the model under consideration, but an important
  direction for future work is to study the dynamics of
  self-improvement for specific models/architectures and
  understand the representations that these models learn under
  self-training.
\item \emph{Richer forms of self-reward.} Our theoretical results study the
  dynamics of self-training in a stylized framework where the model
  uses its own log-probabilities as a self-reward. Empirical research on
  self-improvement leverages more sophisticated \abedit{approaches} (e.g., specific prompting
  techniques)
  \citep{huang2022large,wang2022self,bai2022constitutional,pang2023language,yuan2024self}
  and it is important to understand when and how these forms of
  self-improvement are beneficial.
\end{itemize}

\section*{Acknowledgments}
We thank Sivaraman Balakrishnan, Miro Dud\'{i}k, Susan Dumais, John Langford, Qinghua Liu, and Yuda Song for helpful discussions. 

\clearpage

\iclr{\bibliographystyle{iclr2025_conference}}
\bibliography{refs,sharpening}

\clearpage

\appendix

\renewcommand{\contentsname}{Contents of Appendix}
\addtocontents{toc}{\protect\setcounter{tocdepth}{2}}
{
  \hypersetup{hidelinks}
  \tableofcontents
}

\clearpage

\part{Additional Discussion and Results}

\iftoggle{iclr}{
\section{Additional Experiments and Details}
}{\section{Additional Experimental Results}}\label{app:experiments}
\label{sec:additional_experiments}

\arxiv{
  In this section we display omitted figures discussed in \Cref{sec:experiments}.
  }

\iclr{
\begin{figure}[t]
  \centering
  \subfigure[]{
      \includegraphics[width=0.45\textwidth]{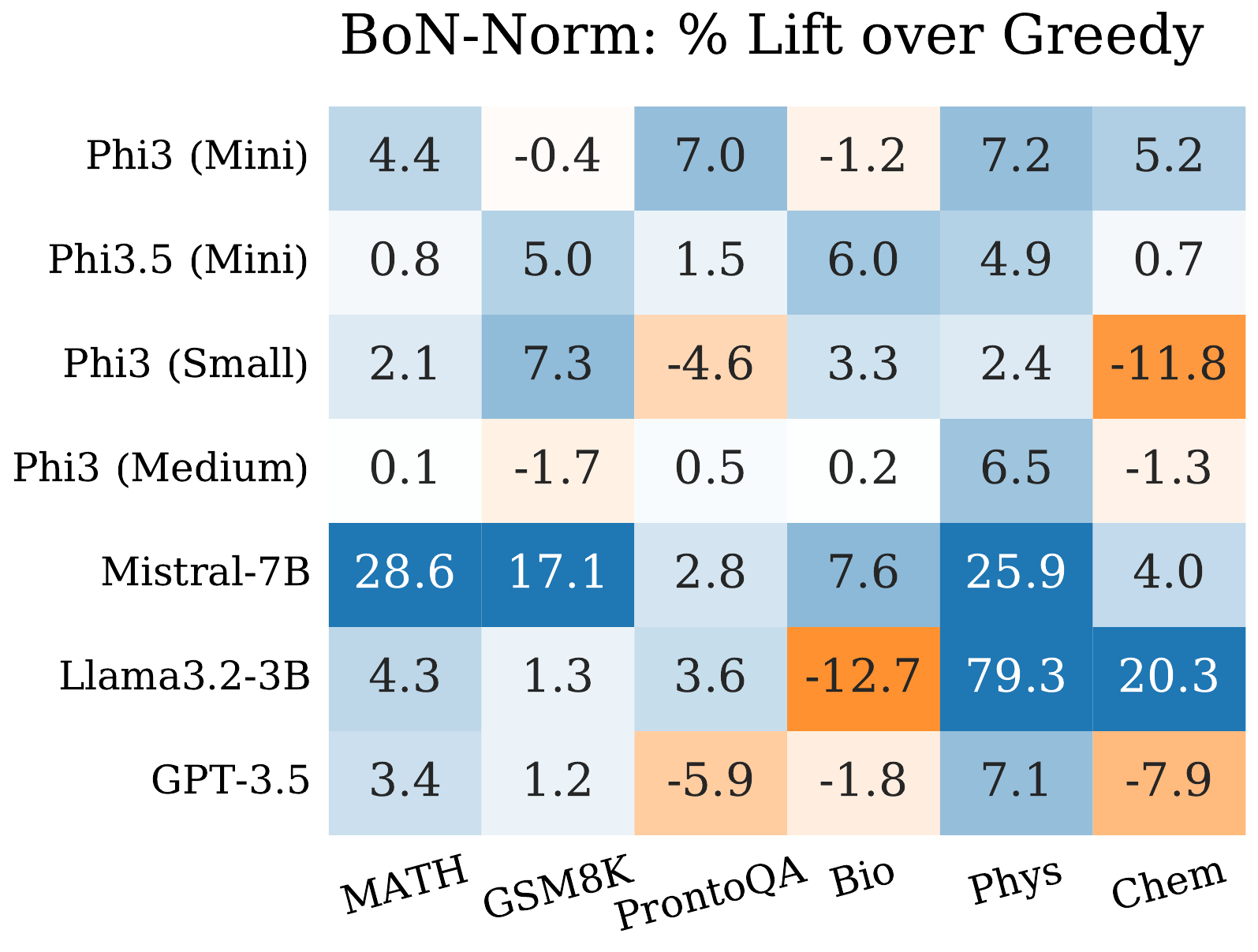}
      \label{sfig:bon-normalized-improvement} 
  }
  \hfill \subfigure[]{
      \includegraphics[width=0.45\textwidth]{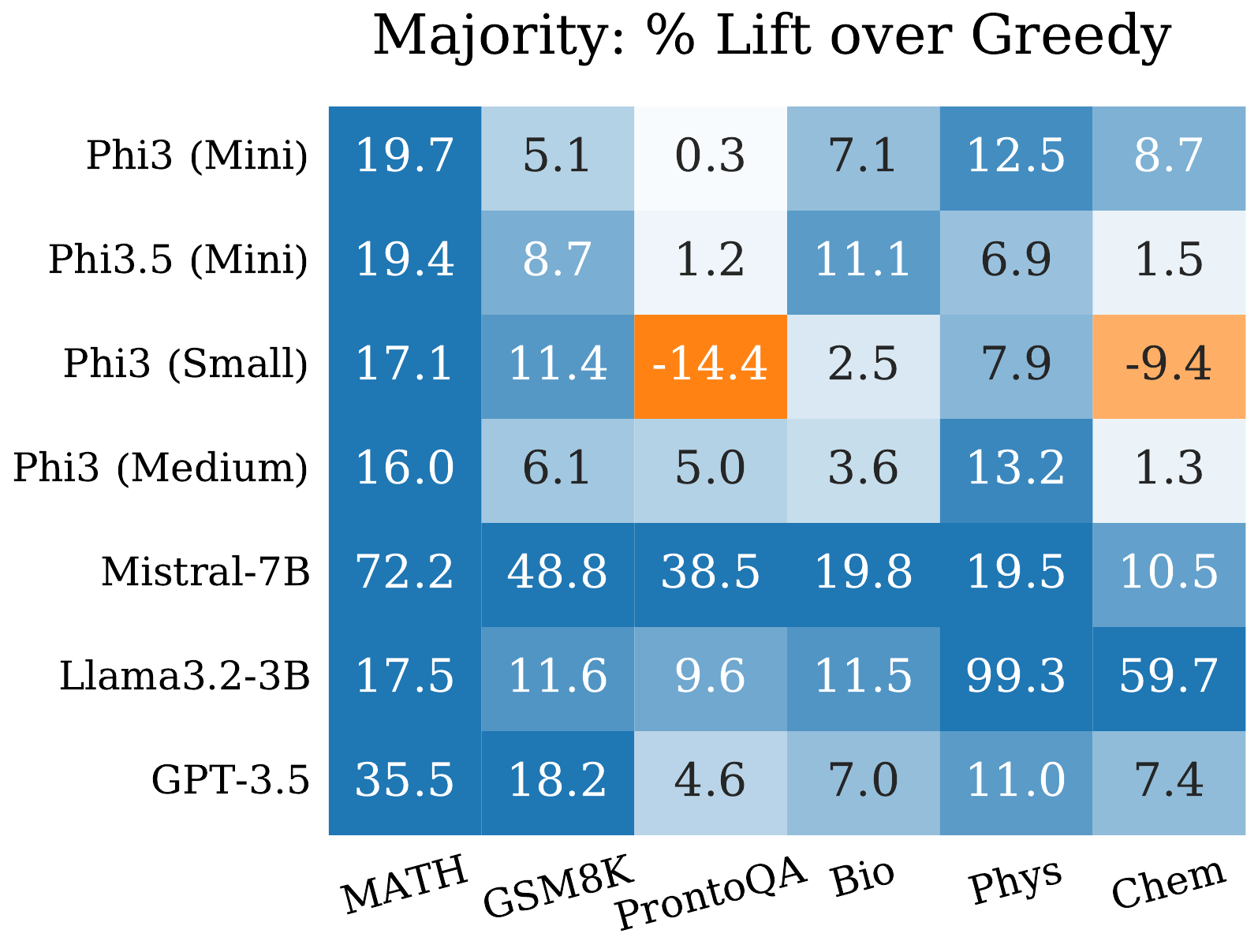}
      \label{sfig:bon-majority-improvement}
  } \\

  \subfigure[]{
      \includegraphics[width=0.45\textwidth]{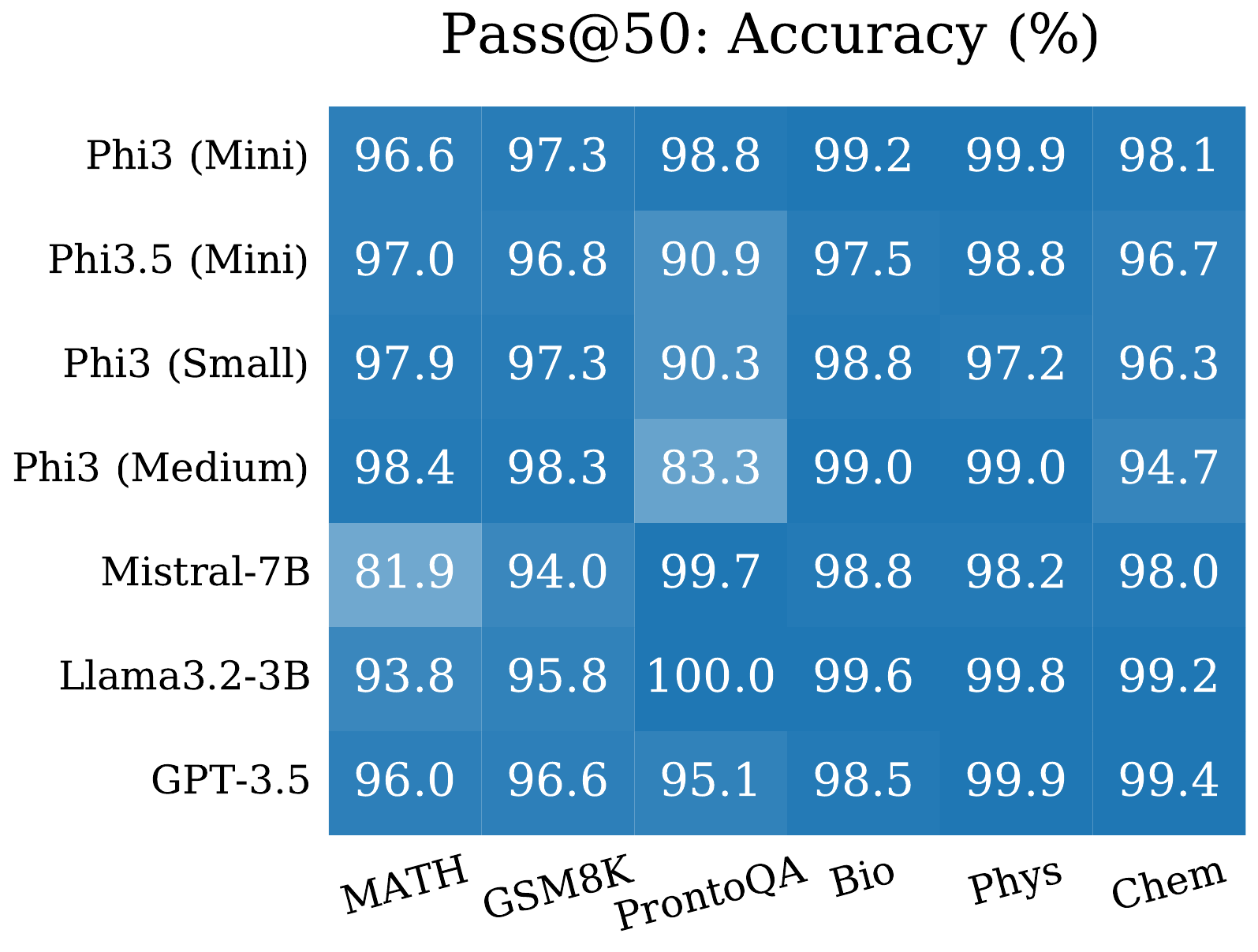}
      \label{sfig:bon-coverage}
  }
  \hfill
  \subfigure[]{
      \includegraphics[width=0.45\textwidth]{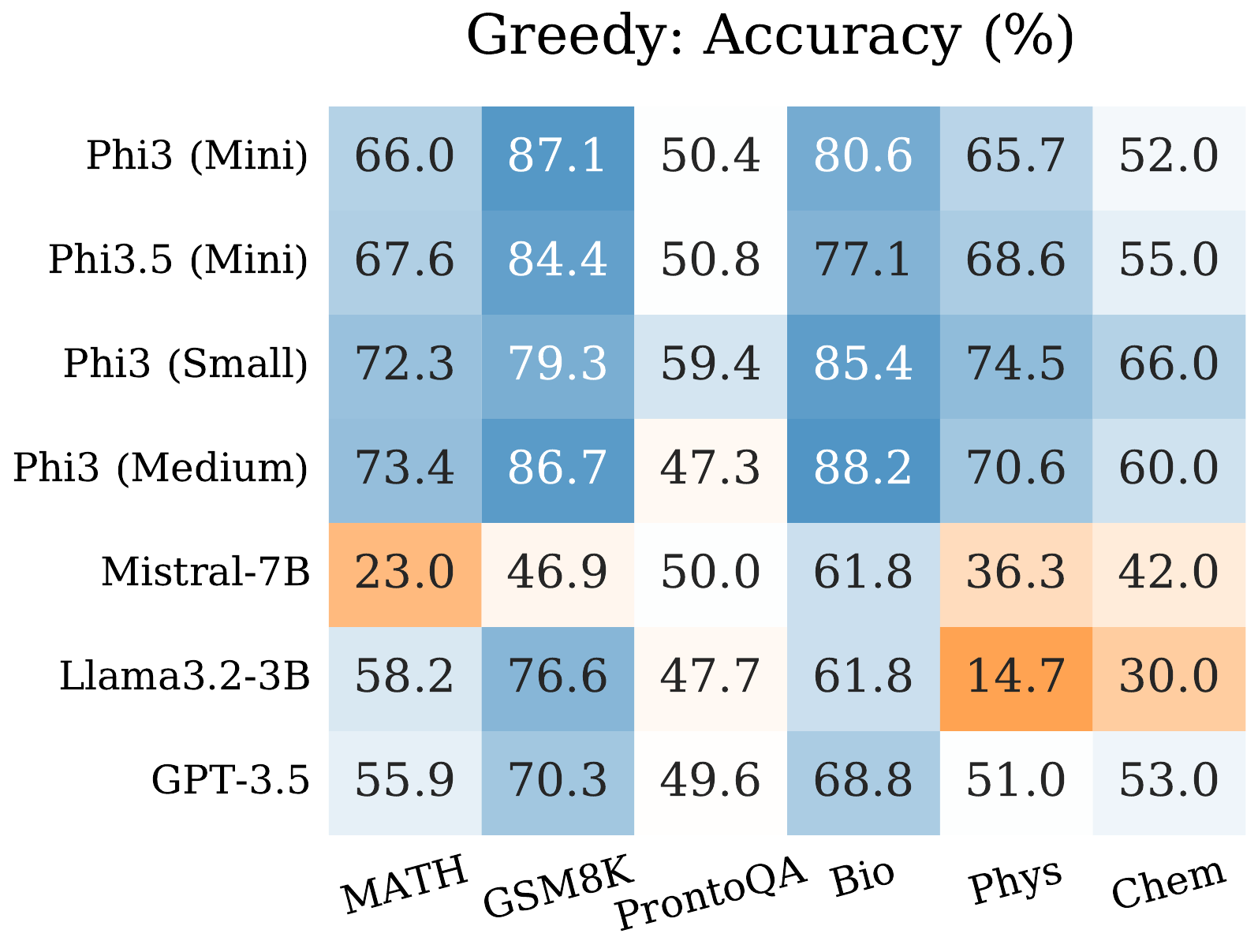}
      \label{sfig:bon-greedy}
  }

  \caption{Performance of alternative decoding schemes beyond BoN.  Percent accuracy improvement over greedy decoding for self-improvement with length-normalized log probability (a) and majority voting (b), with both demonstrating efficacy on a range of model-task pairs.  (c) Measure of coverage of correct answer, demonstrating that most model-task pairs produce the correct answer most of the time with at least one completion out of 50.  (d) Accuracy of greedy decoding baseline on each model-task pair.
  }
  \label{fig:other_rewards}
\end{figure}

}
\arxiv{
\vfill
  \begin{table}[hp]
    \centering
    \begin{tabular}{ccccc}
        \toprule
        \textbf{Model} & \textbf{Dataset} & \textbf{Weight Decay} & \textbf{LoRA Rank} \\
        \midrule
        \phithreefivemini & \mathdataset & 0.1 & 16\\
        \phithreefivemini & \gsm & 0.5 & 16 &  \\
        \phithreefivemini & \prontoqa & 0.0 & 16\\
        \mistral & \mathdataset & 1.0 & 8 \\
        \bottomrule
    \end{tabular}
    \caption{Hyperparameters for training-time sharpening experiments with \bestofnalg.}
    \label{tab:hyperparameters}
  \end{table}
  \vfill
\begin{figure}[hp]
    \centering
    \includegraphics[width=0.45\textwidth]{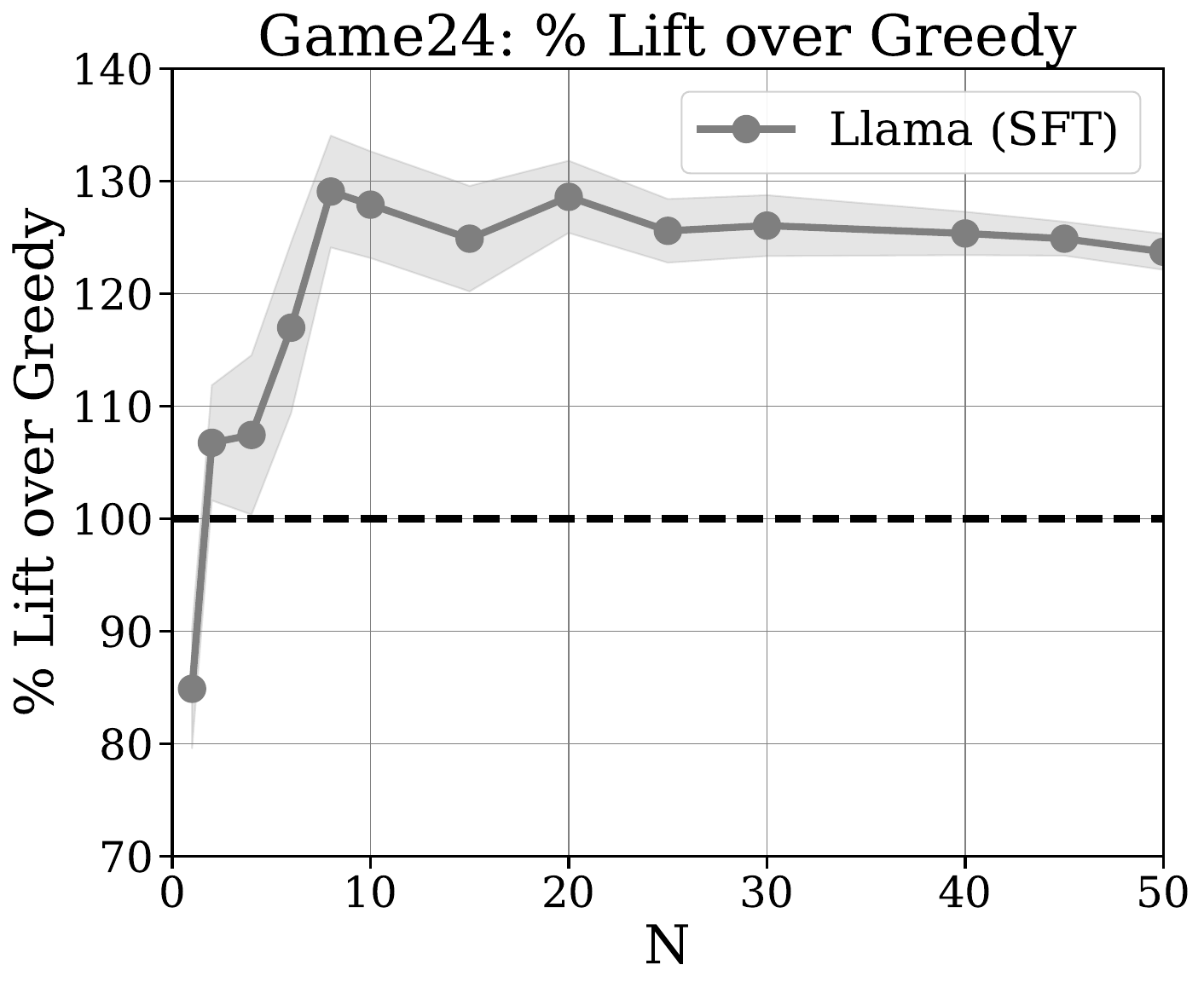}
    \hfill
    \includegraphics[width=0.45\textwidth]{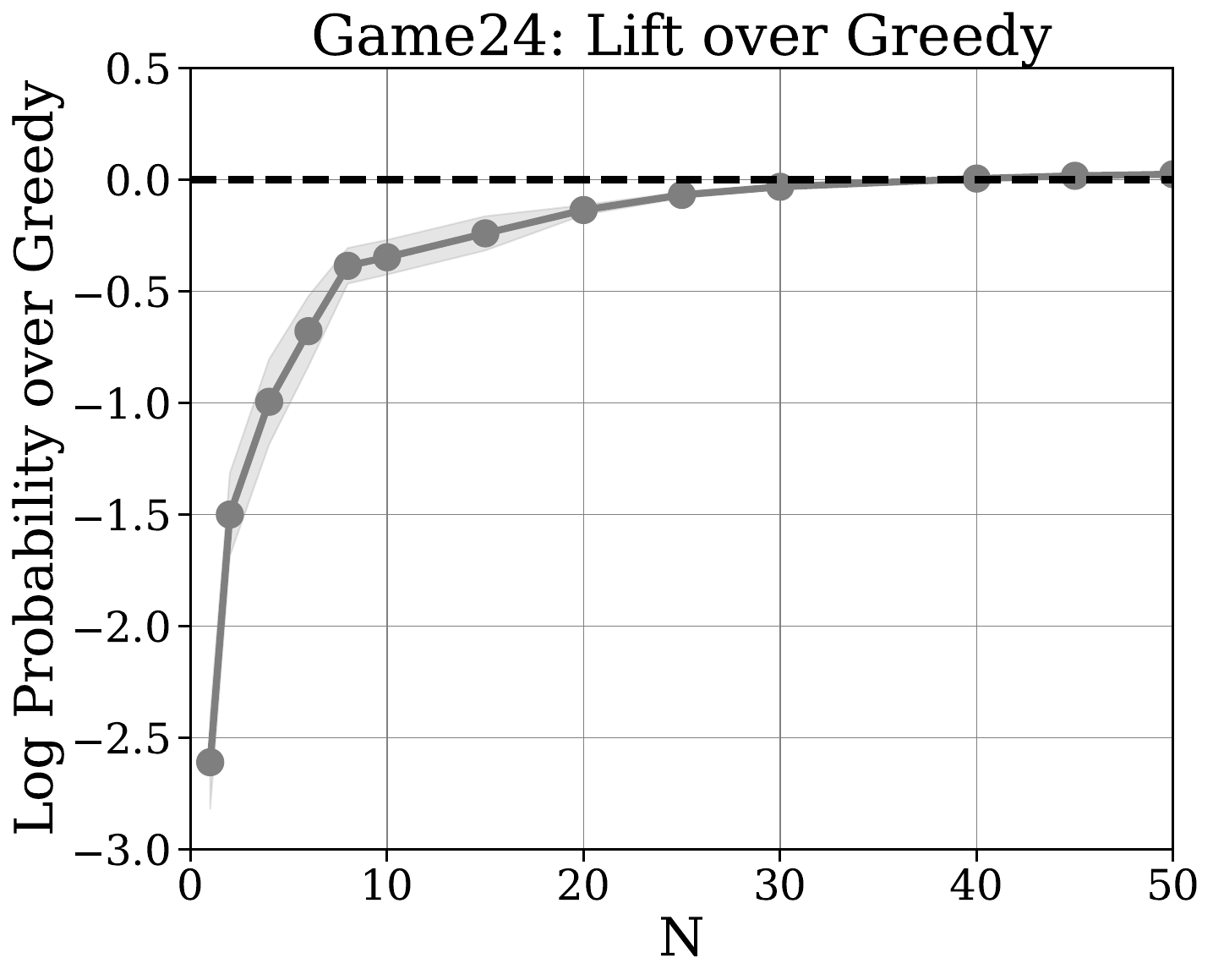}
    \caption{Effect of inference-time BoN-sharpening on \gametwentyfour with finetuned \llamagame model from \citet{wan2024alphazero}.}
    \label{fig:game24_inf}
  \end{figure}
\vfill
}

\begin{figure}[h]
    \centering
    \subfigure[]{
    \includegraphics[width=0.45\textwidth]{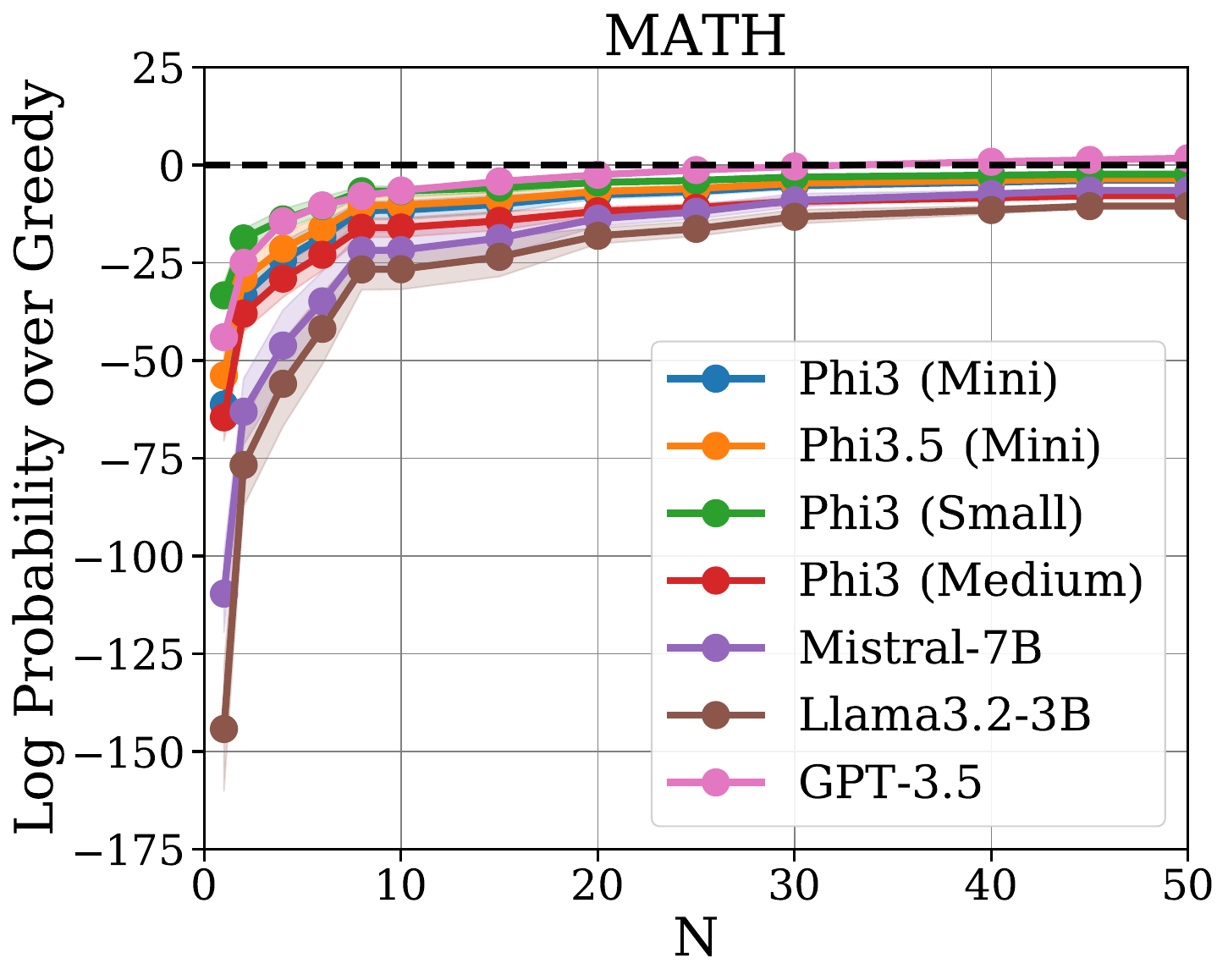}
    }
    \hfill
    \subfigure[]{
    \includegraphics[width=0.45\textwidth]{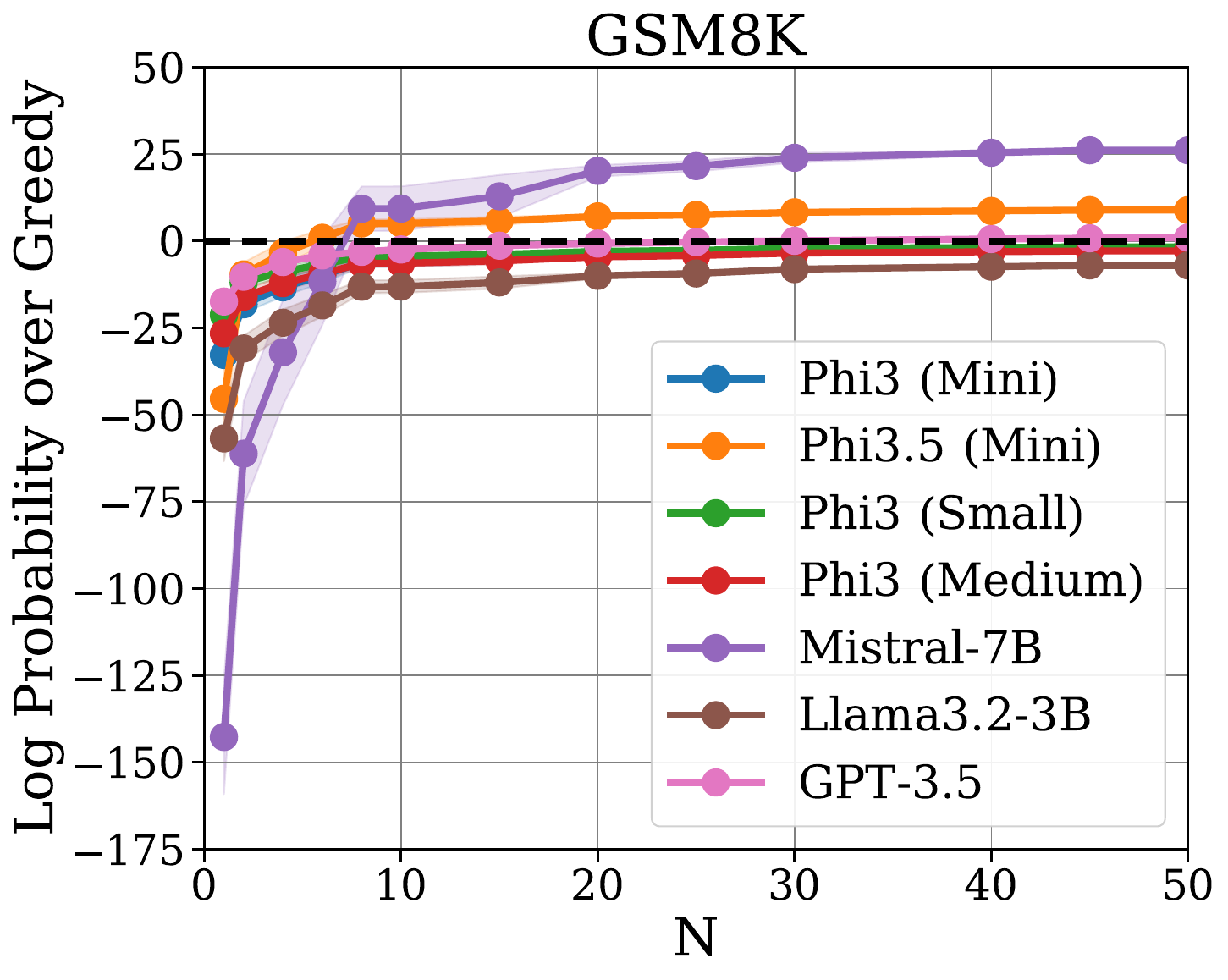}
    }
    \\
    \subfigure[]{
    \includegraphics[width=0.45\textwidth]{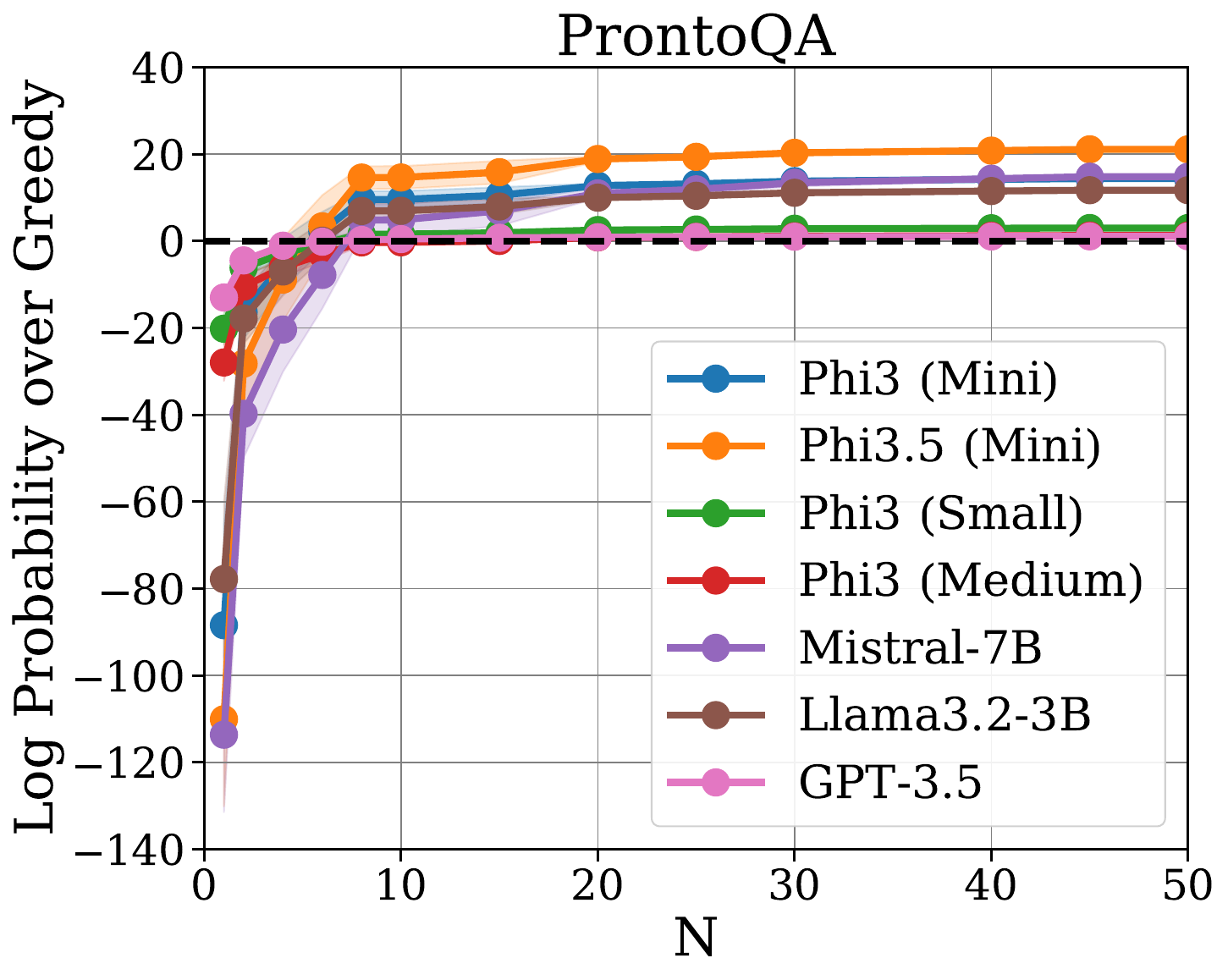}
    }
    \hfill
    \subfigure[]{
    \includegraphics[width=0.45\textwidth]{figs/BoN-mmlu-bio-logprobs.pdf}
    }
    \\
    \subfigure[]{
    \includegraphics[width=0.45\textwidth]{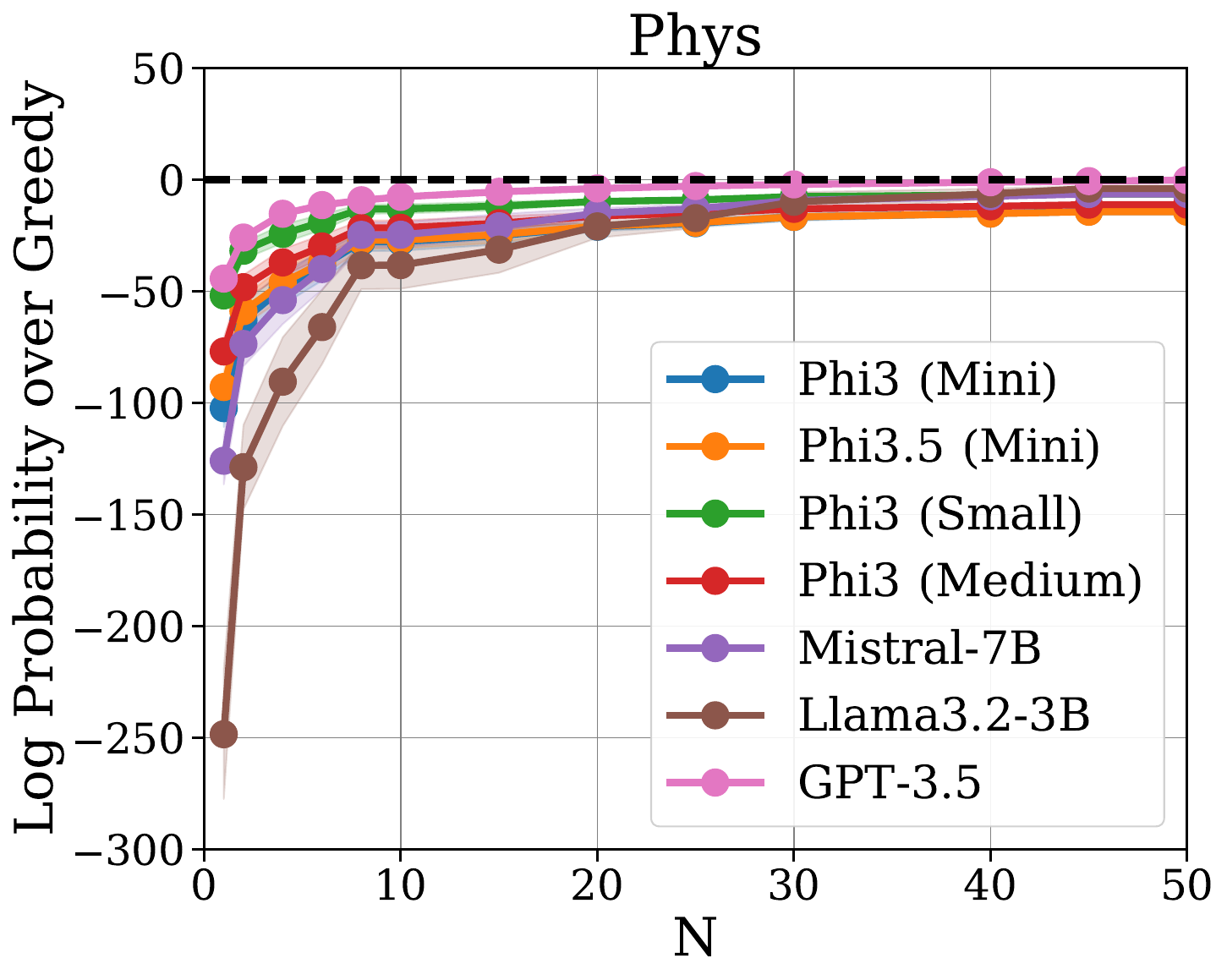}
    }
    \hfill
    \subfigure[]{
    \includegraphics[width=0.45\textwidth]{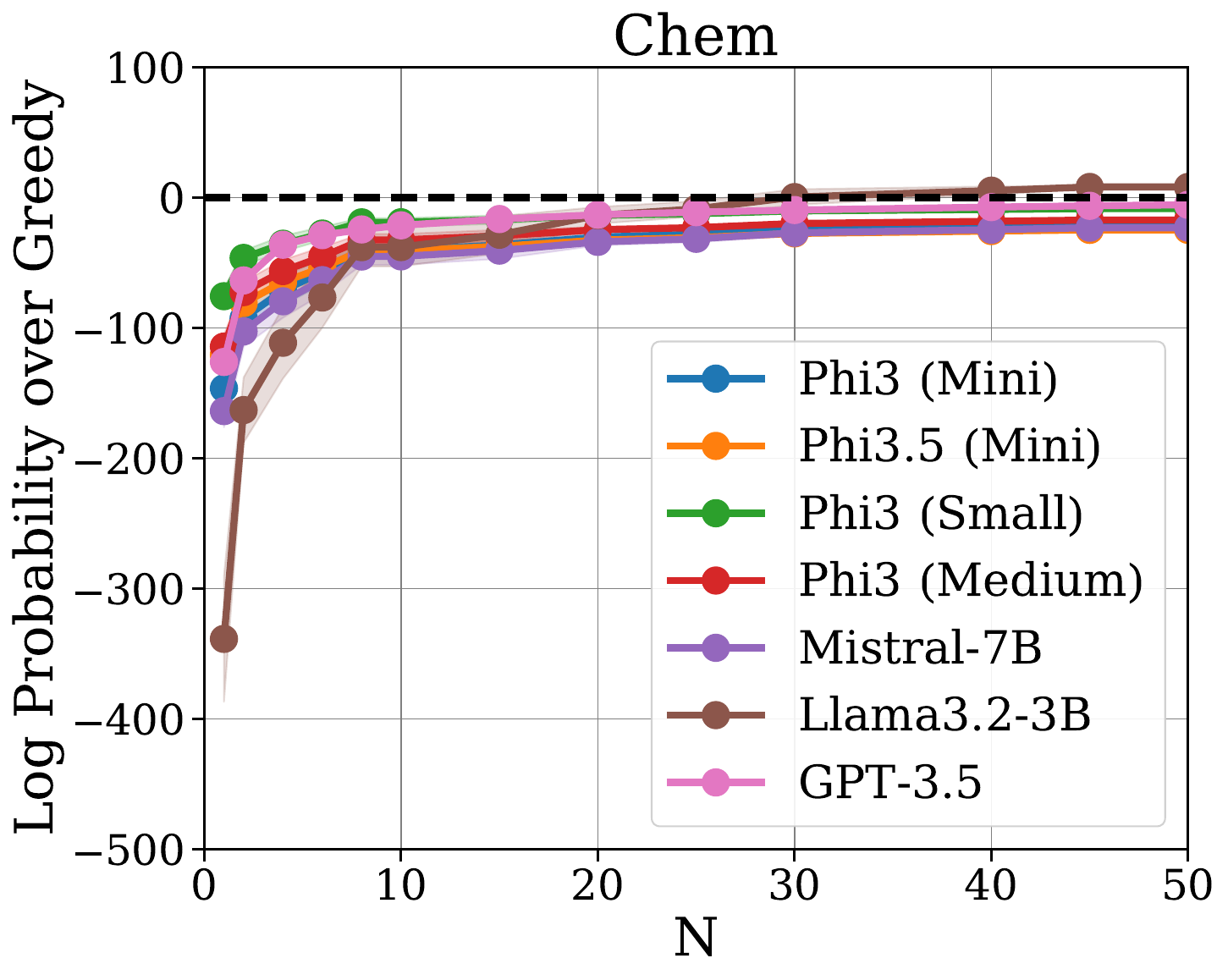}
    }
    
    \caption{Effect of $N$ on average sequence level log-probabilities for inference-time BoN-sharpening on various model-task pairs, compared to greedy decoding baseline. As predicted by theory, the likelihood of sequences sampled with BoN-sharpening increases with $N$.}
    \label{fig:BoN-inference-logprobs}
\end{figure}

\iftoggle{iclr}{
In this section we detail the precise setup required to replicate our empirical results.  All of our experiments were run either on 40G NVIDIA A100 GPUs, 192G AMD MI300X GPUs, or through the OpenAI API.  We considered the following models. All models, except for \gptthree, are available on \url{https://huggingface.co} and we provide HuggingFace model identifiers below.
\begin{enumerate}
\item Phi models: We experiment with several models from the Phi family of models~\citep{abdin2024phi}, specifically \phithreemini (``microsoft/Phi-3-mini-4k-instruct''), \phithreesmall (``microsoft/Phi-3-small-8k-instruct''), \phithreemedium (``microsoft/Phi-3-medium-4k-instruct''), and \phithreefivemini (``microsoft/Phi-3.5-mini-instruct''). 
\item \llamathree (``meta-llama/Llama-3.2-3B-Instruct'')~\citep{dubey2024llama}
\item \mistral (``mistralai/Mistral-7B-Instruct-v0.3'')~\citep{jiang2023mistral}
\item \gptthree~\citep{brown2020language}: We access this model via the OpenAI API.
\item \llamagame (``OhCherryFire/llama2-7b-game24-policy-hf''): We use the model of \iftoggle{workshop}{\cite{wan2024alphazero}}{\citet{wan2024alphazero}}, which is a Llama-2 model finetuned on the \gametwentyfour task \citep{yao2024tree}. We use this model only the \gametwentyfour task. 
\end{enumerate}

We consider the following tasks:

\begin{enumerate}
    \item \mathdataset: We use the above models to generate responses to prompts from the \mathdataset \citep{hendrycks2021measuring}, which consists of more difficult math questions.  We consider ``all'' subsets and take the first 256 examples of the test set where the solution matches the regular expression \verb|(\d*)|.\footnote{\url{https://huggingface.co/datasets/lighteval/MATH}.}  %
    \item \gsm: We use the above models to generate responses to prompts from the GSM-8k dataset \citep{cobbe2021training} where the goal is to generate a correct answer to an elementary school math question. We take the first 256 examples from the test set in the main subset.\footnote{\url{https://huggingface.co/datasets/openai/gsm8k}.} %
    \item \prontoqa: We use the above models to generate responses to prompts from the \prontoqa dataset \citep{saparov2023language}, which consists of chain-of-thought-style reasoning questions with boolean answers.  We take the first 256 examples from the training set.\footnote{\url{https://huggingface.co/datasets/longface/prontoqa-train}.} %
    \item \mmlu: We use the above models to generate responses to prompts from three subsets of the \mmlu dataset \citep{hendrycks2020measuring}, specifically \texttt{college\_biology} (Bio),\texttt{college\_physics} (Phys), and \texttt{college\_chemistry} (Chem) all of which consist of multiple choice questions\footnote{\url{https://huggingface.co/datasets/cais/mmlu}.}. We take the first 256 examples of the test set. 
    \item \gametwentyfour: We use only the model of \iftoggle{workshop}{\cite{wan2024alphazero}}{\citet{wan2024alphazero}} (i.e., \llamagame), on the \gametwentyfour task \citep{yao2024tree}.  The prompts are four numbers and the goal is to combine the numbers with standard arithmetic operations to reach the number `24.'  Here we use both the train and test splits of the dataset.\footnote{\url{https://github.com/princeton-nlp/tree-of-thought-llm/tree/master/src/tot/data/24}}  %
\end{enumerate}

\subsection{Inference-time validation experiments}\label{ssec:inference}
To form the plots in \Cref{fig:validation} and in \Cref{fig:BoN-inference-granular,fig:BoN-inference-logprobs}, for each (model, task) pair, we sampled $N$ generations per prompt with temperature 1 and returned the best of the $N$ generations according to the \mlsharp self-reward function $\rself(y\mid{}x)=\log\piref(y\mid{}x)$; we compare against greedy decoding as a baseline, whose accuracy is displayed in \Cref{sfig:bon-greedy}.  

\paragraph{Implementation details}
For all models and datasets except for \gametwentyfour, we used 1-shot prompting to ensure that models conform to the desired output format and to elicit chain of thought reasoning (for \gametwentyfour we do not provide a demonstration in the prompt). We set the maximum length of decoding to be $512$ tokens. We used 10 seeds for all (model, task) pairs with a maximum value of $N=50$ in \BestofN sampling. We simulated $N$ responses for $N<50$ by subsamplng the 50 generated samples. For \BestofN sampling, we always use temperature $1.0$. Since greedy decoding is a deterministic strategy, we only use 1 seed for each (model, task) pair. In all experiments, we collect both the responses and their log-likelihoods under the \emph{reference model} (i.e., the original model from which samples were generated). 

\paragraph{Results}
Results for most datasets are presented in \Cref{fig:BoN-inference-granular,fig:BoN-inference-logprobs}.  Because we only consider a single model for \gametwentyfour, we separate this task into \Cref{fig:game24_inf} For all datasets, we visualize both performance---measured as normalized improvement in accuracy over greedy decoding---and log-likelihoods---under $\piref$---of the selected responses. 

In all cases, \BestofN sampling (using $\rself(y\mid{}x)=\log\piref(y\mid{}x)$) improves over the na{\"i}ve sampling strategy, wherein we simply sample a single generation with temperature 1.0.  In all datasets, we also see improvements over the standard \emph{greedy decoding} strategy, at least for some models. 
Analogously, for every model, there is at least one dataset for which \BestofN sampling improves over greedy decoding. 

We further explore the relationship between sequence level log probabilities and generation quality in \Cref{fig:distributions}, where we plot the empirical distributions of responses sampled with temperature 1 from the base model for a variety of model-dataset pairs, conditioned on whether or not the response is correct.  It is clear from the figures that the distribution of log probabilities conditioned on correctness stochastically dominates that conditioned on incorrectness in each case, which provides yet more evidence that log likelihoods represent a reasonable sel-improvement target.

We mention several other observations from the experiments. First, in most cases, performance and log-likelihood saturate at relatively small values of $N$, typically around 10 or 20. This suggests that significant improvements can be obtained with relatively low computational overhead. Second, in some cases, performance can degrade as $N$ increases. We found that this happens for two reasons: (1) the performance of the reference model is quite low and so $\rself$ provides a poor signal (e.g., with \llamathree) and (2) the \BestofN criteria selects for short responses, which have higher log-likelihood but cannot leverage the computational/representational benefits of chain-of-thought, and thus yield worse performance (e.g., with \gptthree on \gsm). }
{

}

\iclr{
\begin{figure}[hp]
  \centering
  \includegraphics[width=0.45\textwidth]{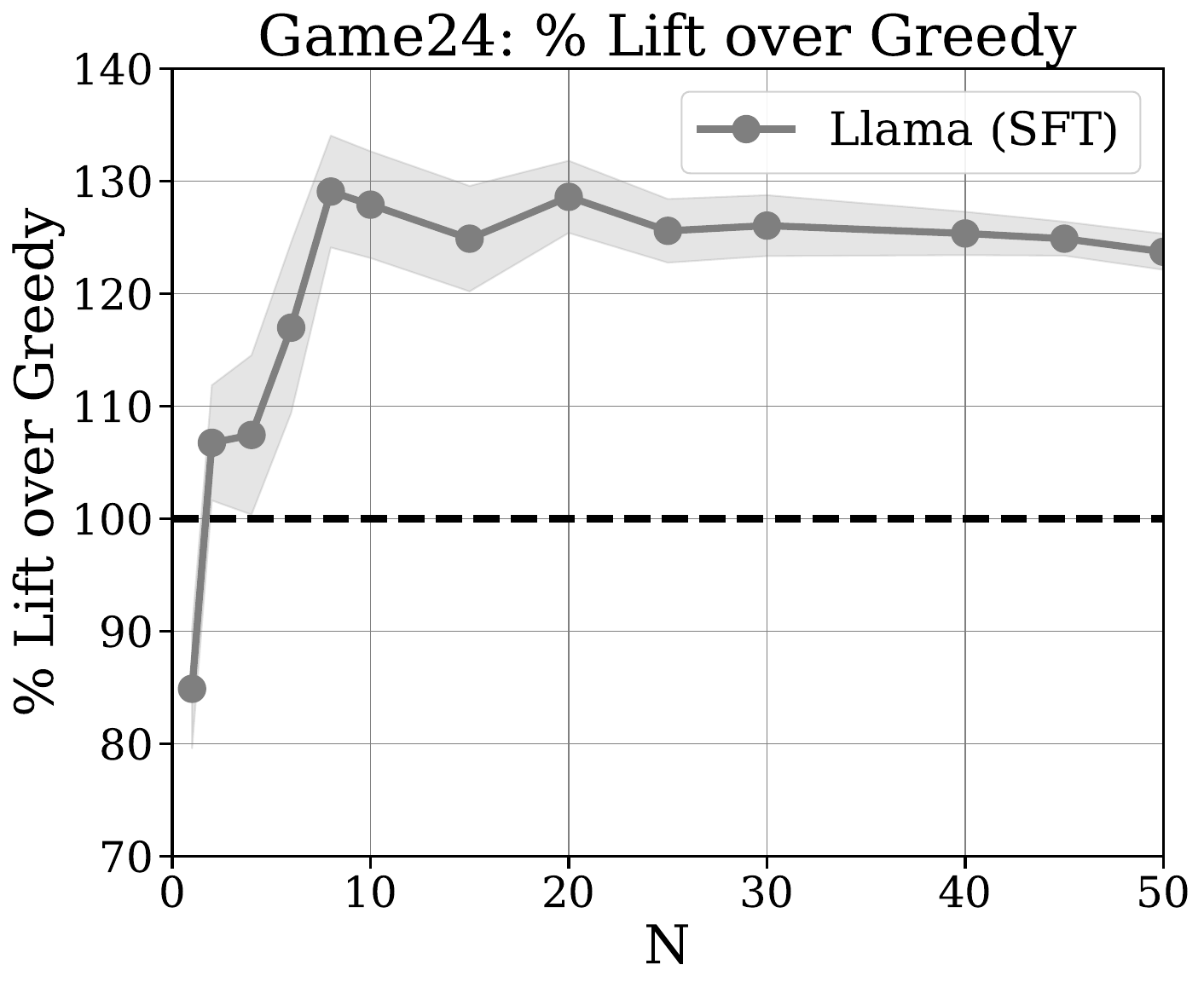}
  \hfill
  \includegraphics[width=0.45\textwidth]{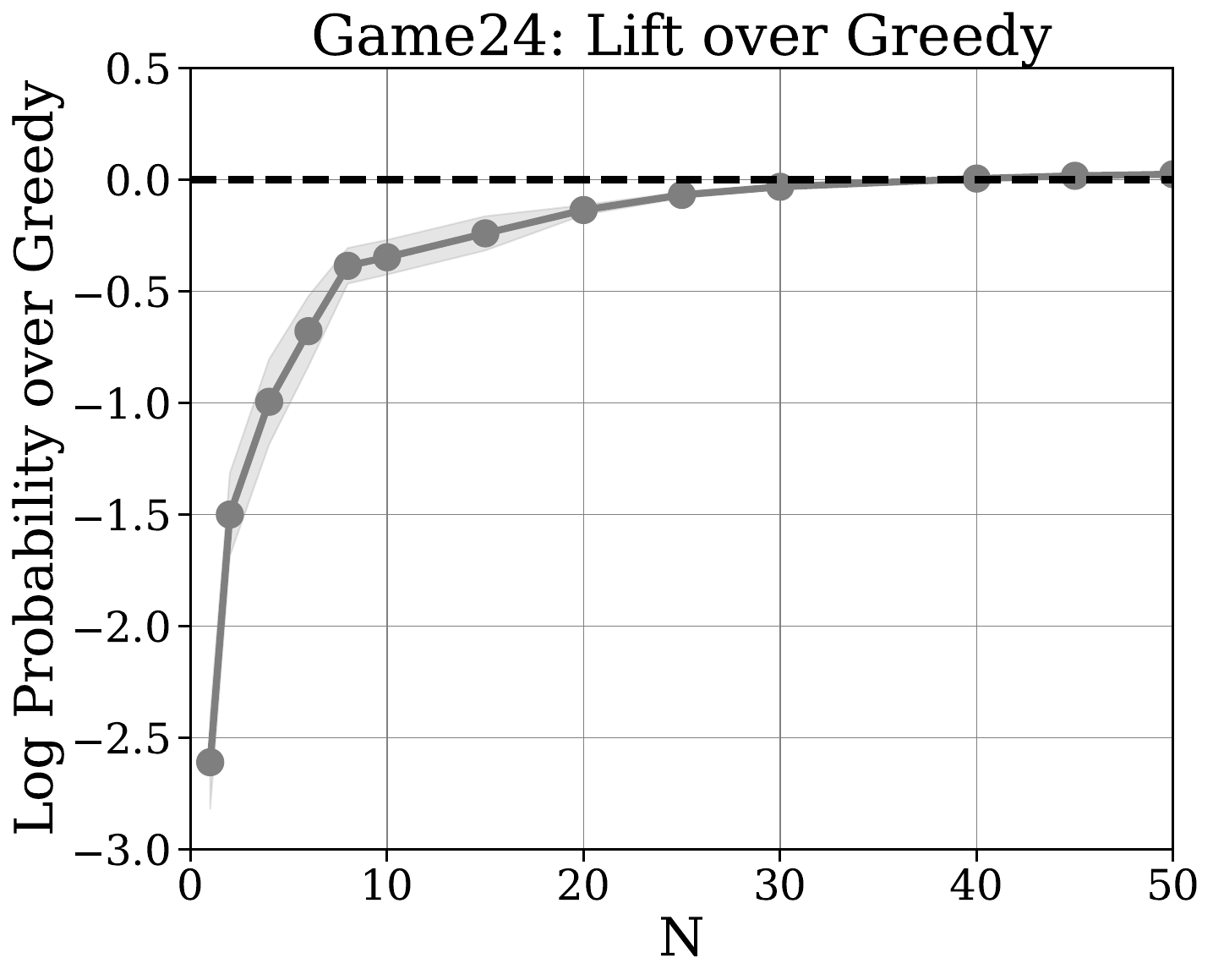}
  \caption{Effect of inference-time BoN-sharpening on \gametwentyfour with the finetuned \llamagame from \citet{wan2024alphazero}.}
  \label{fig:game24_inf}
\end{figure}
}
\iclr{
\subsection{Experiments with other self reward functions}
Although we focus on $\rself(y \mid{}x) = \log\piref(y \mid{}x)$ throughout the paper, the sharpening framework is significantly more general. As such, we also ran experiments with other choices for $\rself$, specifically:
\begin{enumerate}
  \item Length-normalized log-likelihood: $\rself(y \mid{}x) = \log\piref(y \mid{}x) / |y|$ where $|y|$ is the length, in tokens, of the response.
  \item Majority (self-consistency): All datasets except \gametwentyfour have multiple-choice, boolean, or numerical answers. Although we allow responses to contain chain-of-thought tokens, we can extract the answer from each response and use the most-frequently-occuring answer. This can be seen as a sample-based approximation to the following self-reward function: $\rself(y\mid{}x) = \sum_{y': y'_{\texttt{ans}}=y_{\texttt{ans}}} \piref(y'\mid{}x)$, where $y_{\texttt{ans}}$ are the ``answer'' tokens in the full response $y$. 
\end{enumerate}
Finally, as a skyline we consider the \emph{coverage} criterion~\citep{brown2024large}, where we simply check if any of the sampled responses corresponds to the correct answer. This criterion is a skyline and does not fit into the self-improvement framework due to the fact that it uses knowledge of the ground truth (external) task reward function. 

Results are displayed in \Cref{fig:other_rewards}. For length-normalized log-likelihood and majority, we see qualitatively similar behavior to (unnormalized) log-likelihood in the sense that inference-time sharpening via these self-reward functions offers improvements over both vanilla (temperature 1.0) sampling and greedy decoding. In both cases, the improvements are generally much larger than those obtained with log-likelihood. Finally, examining the coverage criteria, we see that with $N=50$ samples, these models almost always produce a correct answer on these tasks, raising the possibility of other self-reward functions that further improve performance. 
}

\begin{figure}
    \centering
    \subfigure[]{
    \includegraphics[width=0.45\textwidth]{figs/Logprob-Distribution-phi35mini-math.pdf}
    }
    \hfill
    \subfigure[]{
    \includegraphics[width=0.45\textwidth]{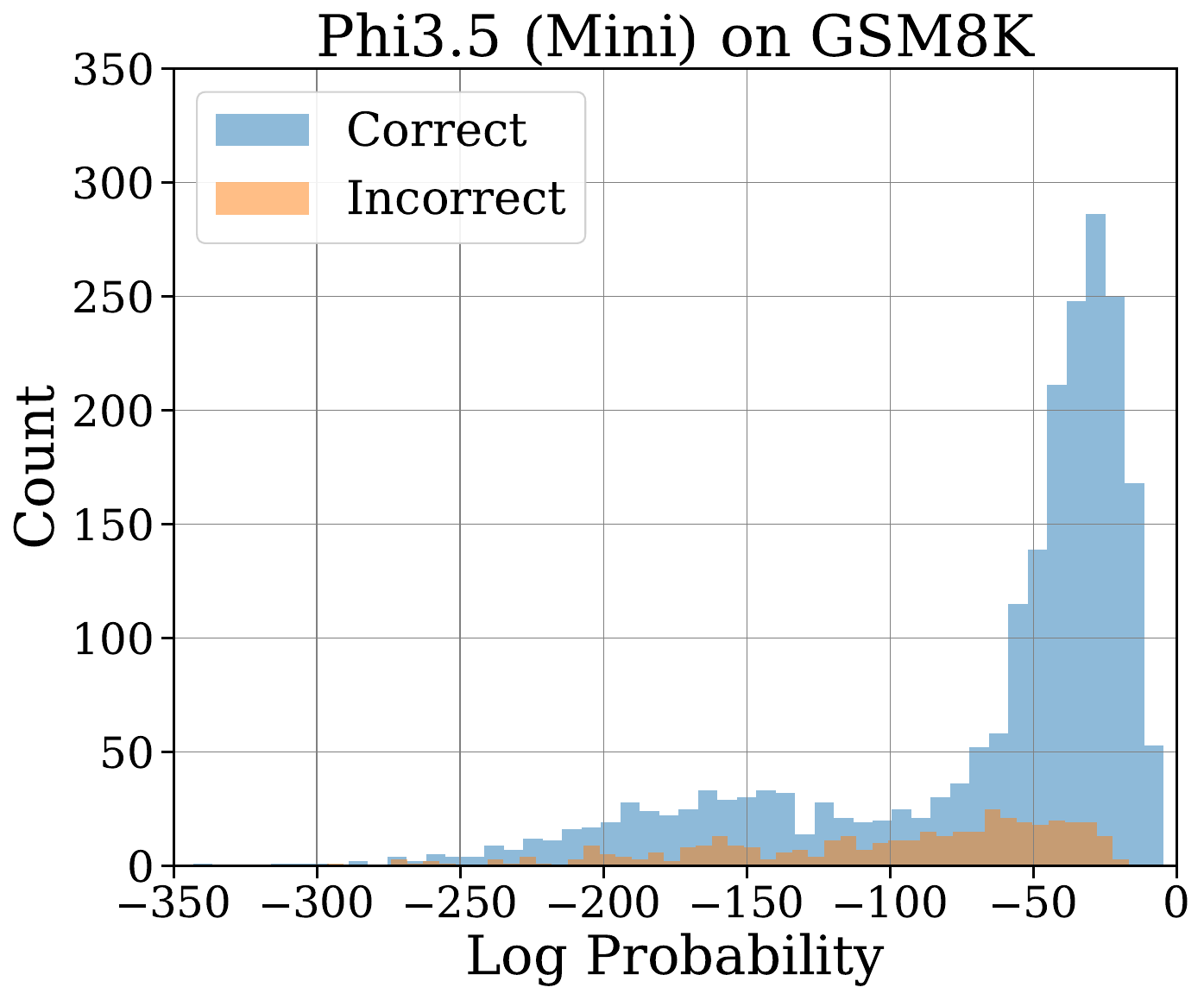}
    }
    \\
    \subfigure[]{
    \includegraphics[width=0.45\textwidth]{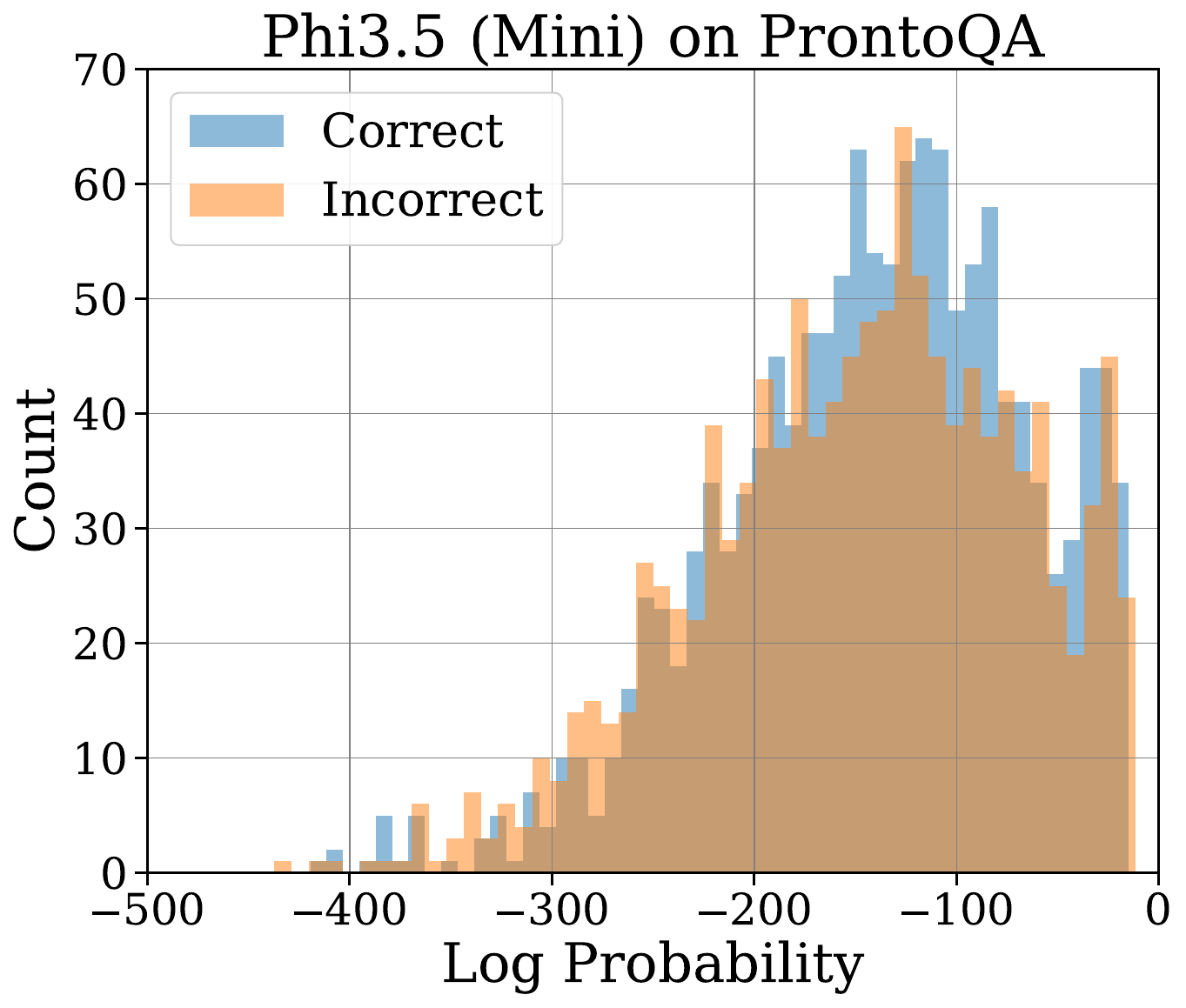}
    }
    \hfill
    \subfigure[]{
    \includegraphics[width=0.45\textwidth]{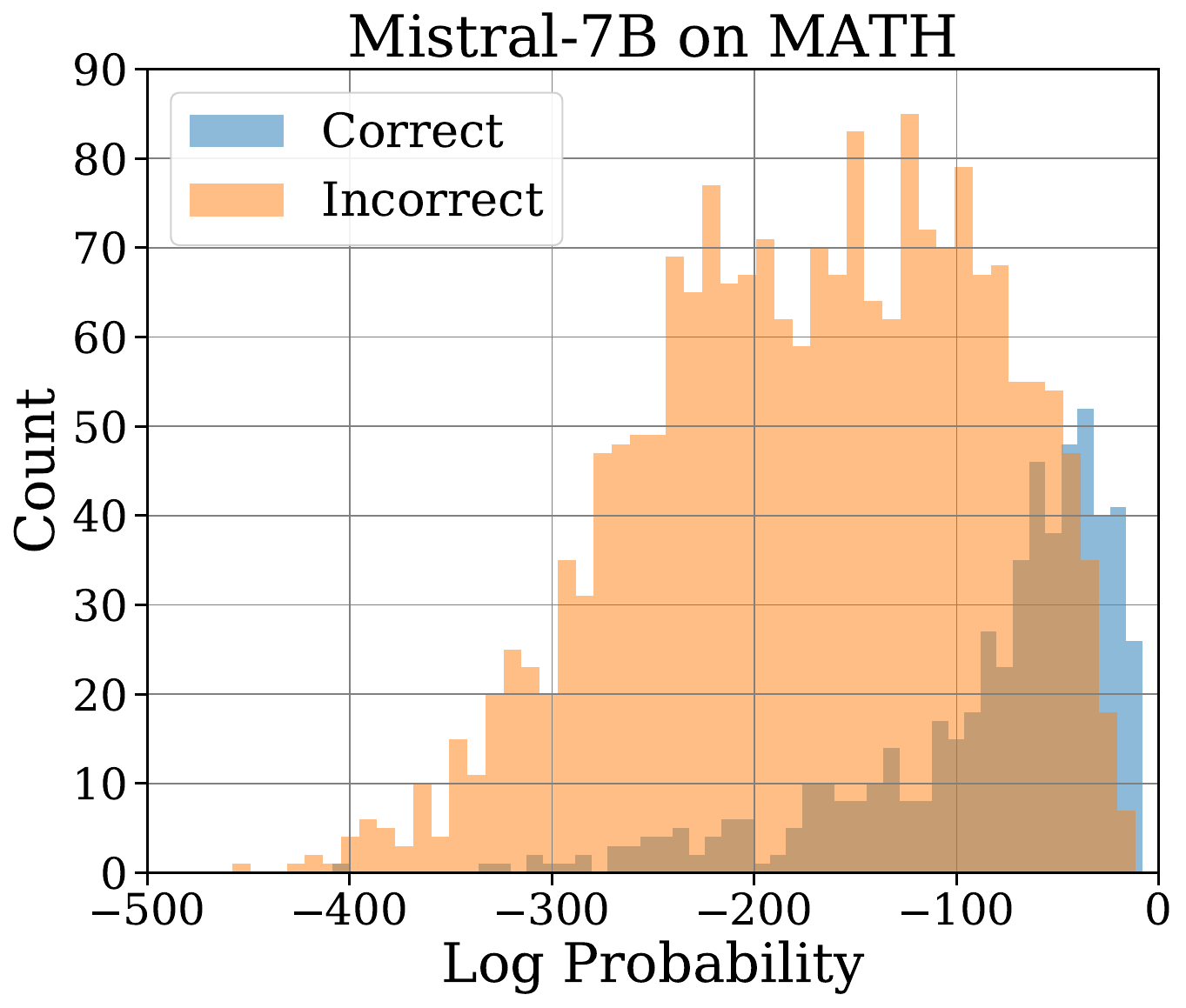}
    }
    \caption{Distribution of sequence-level log-probabilities for responses sampled with temperature 1, conditioned on whether or not the response is correct. We consider four model-dataset pairs: (a) (\phithreefivemini, \mathdataset); (b) (\phithreefivemini, \gsm); (c) (\phithreefivemini, \prontoqa); (d) (\mistral, \mathdataset).  In all cases except perhaps (c), conditioning on correctness of the response leads to a noticeable increase in log-probabilities, further justifying the use of sequence-level log-probabilities as a self-reward for self-improvement.}
    \label{fig:distributions}
\end{figure}

\iclr{
\subsection{Effect of \bestofnalg}
\begin{table}[t]
    \centering
    \begin{tabular}{cccc}
        \toprule
        \textbf{Model} & \textbf{Dataset} & \textbf{\% Lift over Greedy (Accuracy)} & \textbf{Lift over Greedy (Likelihood)} \\
        \midrule
        \phithreefivemini & \mathdataset & $19.24 \pm 2.41$ & $48.33 \pm 0.17$ \\
        \phithreefivemini & \gsm & $1.82 \pm 0.64$ & $1.49 \pm 0.55$ \\
        \phithreefivemini & \prontoqa & $12.46 \pm 1.08$ & $ 5.64 \pm 0.01$ \\
        \texttt{Mistral-7B} & \mathdataset & $8.88 \pm 5.55$ & $5.71 \pm 3.00$ \\
        \bottomrule
    \end{tabular}
    \caption{Experimental results for \bestofnalg}
    \label{tab:performance}
\end{table}
In addition to inference-time experiments demonstrating the validity of the amortization objective considered in our theory, we also demonstrate empirically that amortization can be effected with \bestofnalg.  Due to the realities of limited computational resources, we choose a strict subset of the model-task pairs considered in \Cref{ssec:inference} that have particularly promising inference-time BoN performance and apply \bestofnalg to amortize the inference time cost of multiple generations.

For each of the chosen model-dataset pairs (cf. \Cref{tab:performance}), we sample $N=50$ responses with temperature 1 for each prompt in the dataset and select the most likely (according to the relevant reference model).  We then combine these likely responses with the prompts in order to form a training corpus and train a Low Rank Adaptation \citep{hu2021lora} to the model, sweeping over LoRA rank, learning rate scheduler, and weight decay in order to return the best optimized model.\footnote{In all experiments involving \phithreefivemini we use a batch size of 4; unfortunately, due to a known numerical issue with LoRA on \mistral involving batch size $>1$, we use a batch of 1 in this case.  Because of this choice, instead of the 30 epochs we use to train our other models, for \mistral, we run only 10 epochs.}  We report the specific hyperparamters chosen in \Cref{tab:hyperparameters}.  On all models, we used a learning rate of $3 \times 10^{-4}$ with linear decay to zero and gradient clamping at 0.1.

\paragraph{Results}  In \Cref{tab:performance} we report our results for the best model during training of each model-dataset pair, averaged across 3 random seeds, where responses are sampled with temperature 1 from the fine-tuned model.  We report both the percent lift in accuracy on the dataset with respect to the greedy generation of the reference model and the increase in average sequence level log likelihood with respect to the same.  In all cases, we see improvement on both metrics, demonstrating that some amortization is possible with \bestofnalg.  In \Cref{fig:training_curve_phi,fig:training_curve_mistral}, we display the evolution throughout training of these same metrics for each of the model-dataset pairs.  While \phithreefivemini is quite well-behaved on \mathdataset and \prontoqa, there appears to be a fair amount of noise in the training on \gsm, with the log probability being a significantly less useful proxy for accuracy on this dataset than the others, as has been previously found in \citet{block2023butterfly}.  In the case of \mistral on \mathdataset, while we do see some improvement after sufficient training, the optimization suffers an initial substantial drop and then spends $\sim90\% $ of the gradient steps recovering; we speculate that this is a function of insufficient hyper-parameter tuning of the optimization itself, rather than a fundamental barrier.

Finally, in \Cref{fig:sft-N}, we investigate the effect that the choice of $N$ has on \bestofnalg for \phithreefivemini on \mathdataset.  In particular, in forming our training set, we choose $N \in \{10, 25, 50 \}$ and repeat the procedure described above, averaging our results over three seeds.  We find that increasing $N$ leads to a modest increase in the sequence-level log-likelihood and a consequent increment in the accuracy of the fine-tuned model, in accordance with our theory.
}

\iclr{
\begin{table}[t]
    \centering
    \begin{tabular}{ccccc}
        \toprule
        \textbf{Model} & \textbf{Dataset} & \textbf{Weight Decay} & \textbf{LoRA Rank} \\
        \midrule
        \phithreefivemini & \mathdataset & 0.1 & 16\\
        \phithreefivemini & \gsm & 0.5 & 16 &  \\
        \phithreefivemini & \prontoqa & 0.0 & 16\\
        \mistral & \mathdataset & 1.0 & 8 \\
        \bottomrule
    \end{tabular}
    \caption{Hyperparameters for \bestofnalg}
    \label{tab:hyperparameters}
\end{table}
}

\begin{figure}[ht]
    \centering
    \subfigure[]{
    \includegraphics[width=0.45\textwidth]{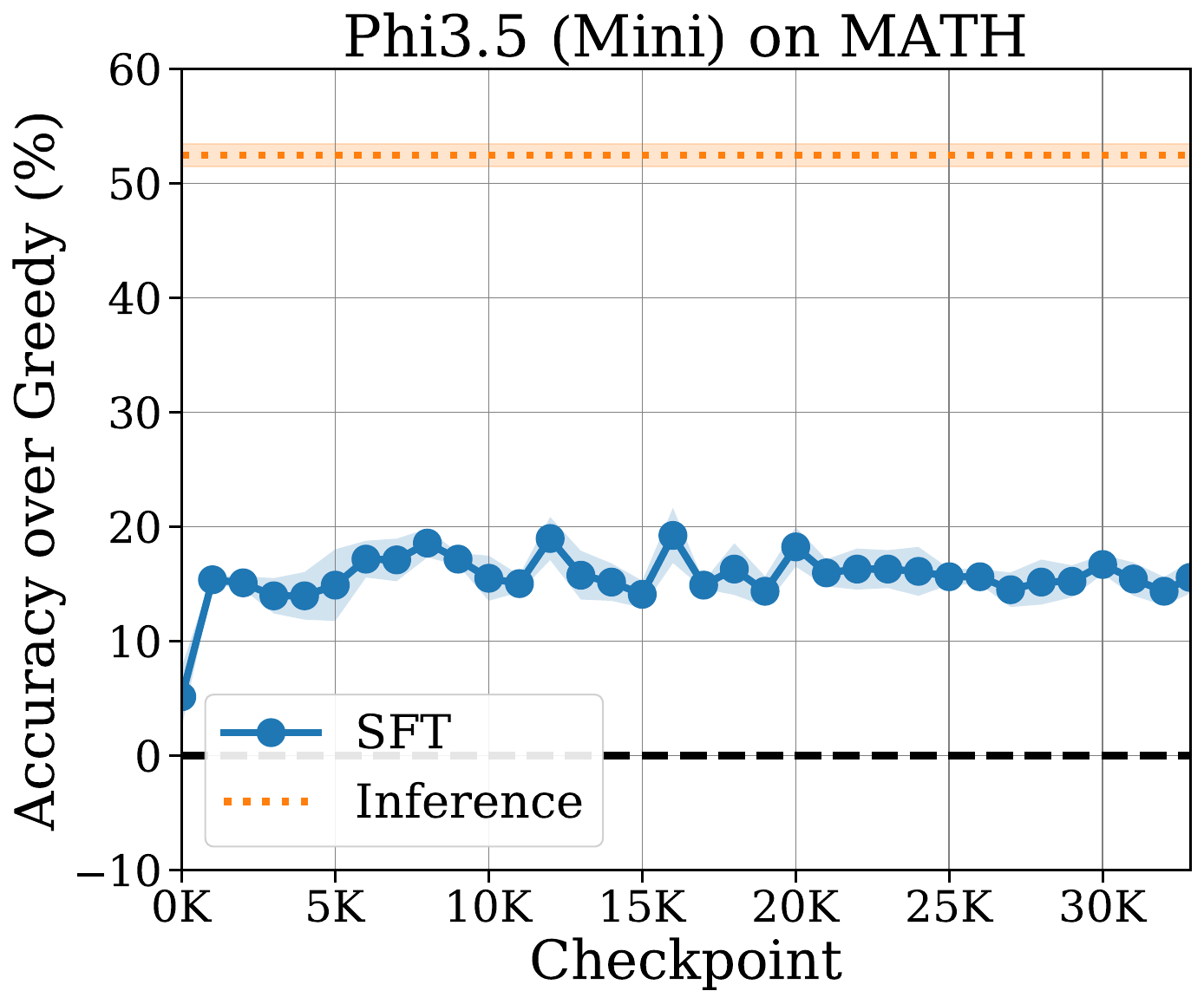}
    } 
    \hfill
    \subfigure[]{
    \includegraphics[width=0.45\textwidth]{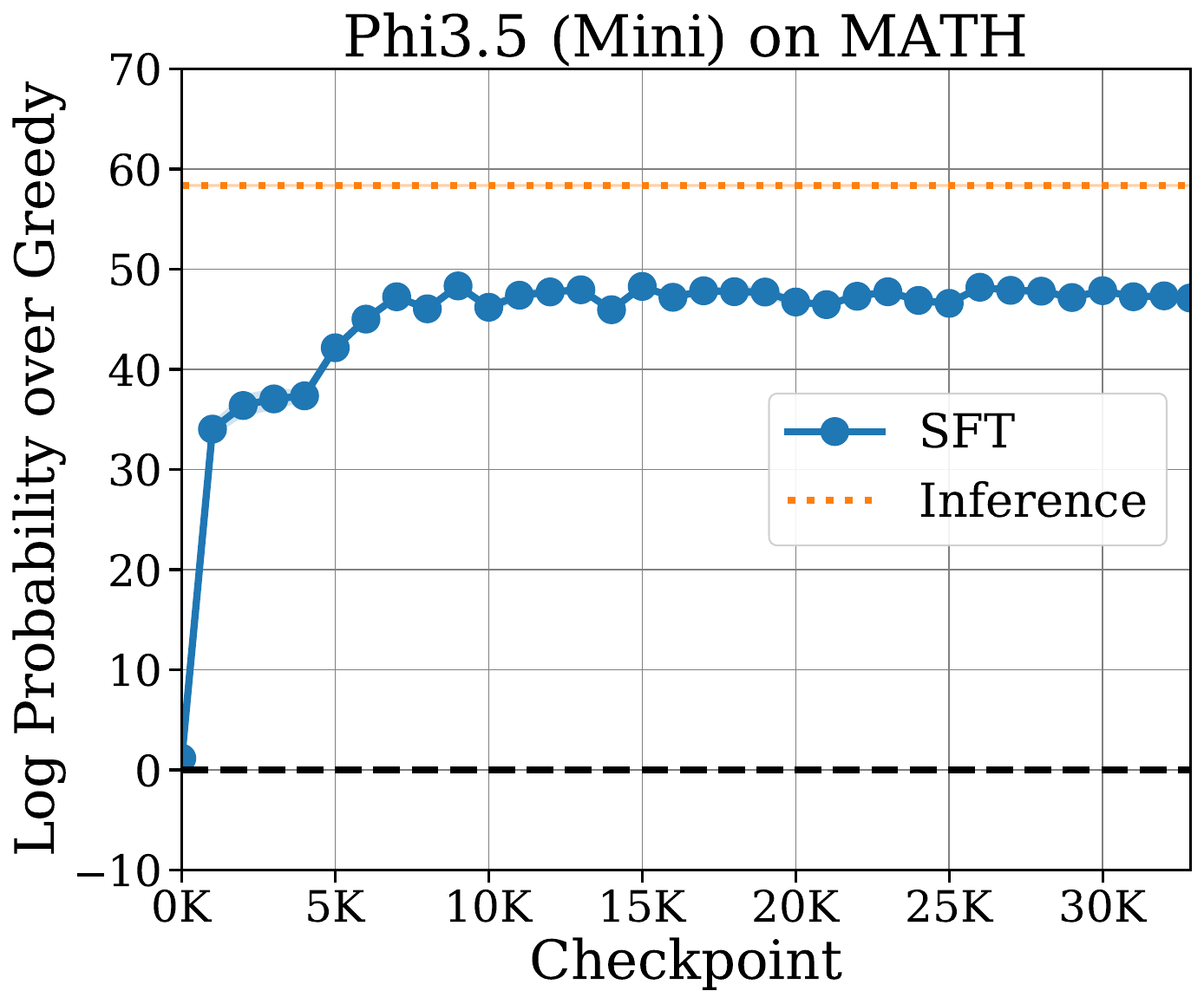}
    } 
    \vspace{-0.15cm}
    \\
    \subfigure[]{
    \includegraphics[width=0.45\textwidth]{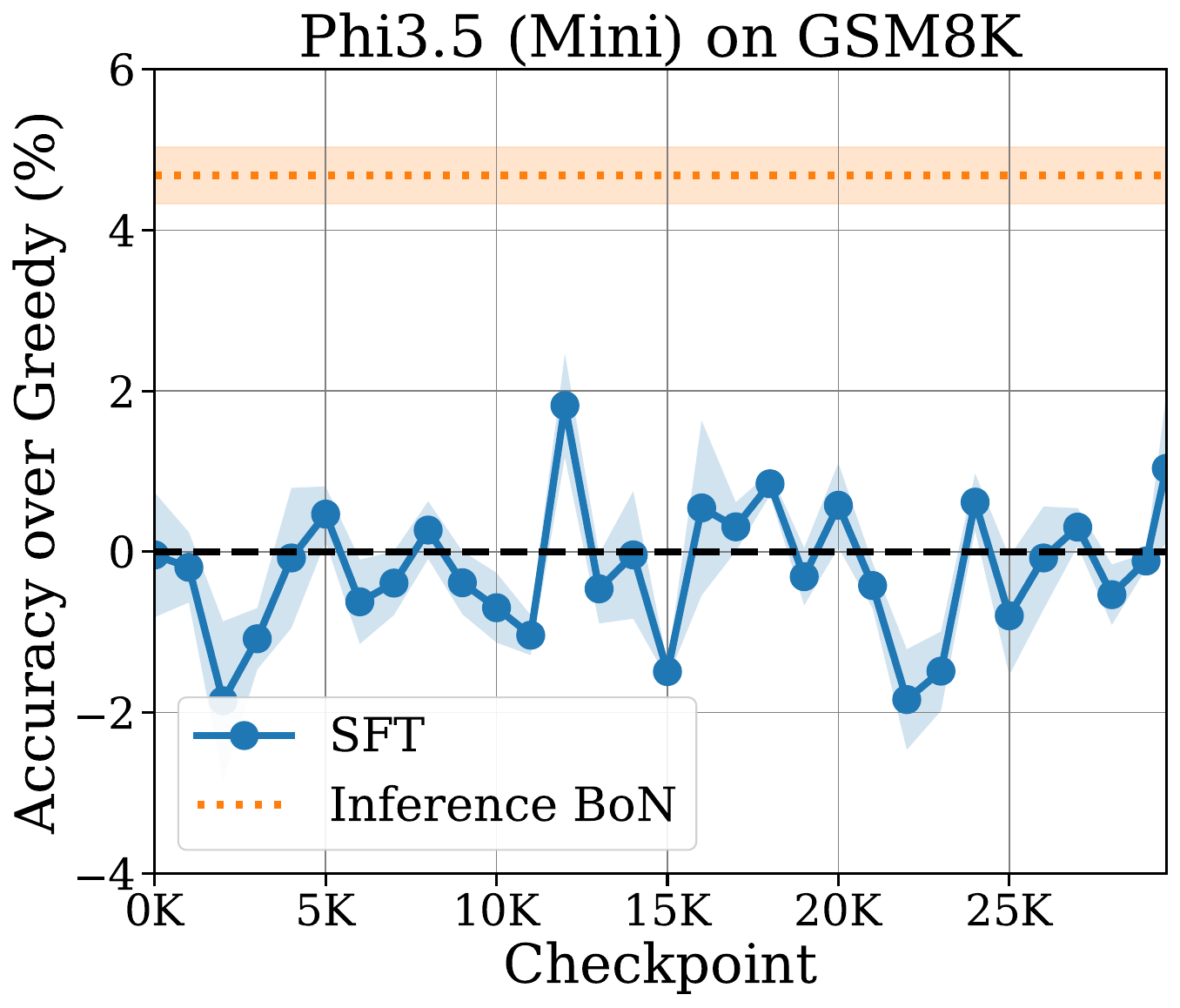}
    } 
    \hfill
    \subfigure[]{
    \includegraphics[width=0.45\textwidth]{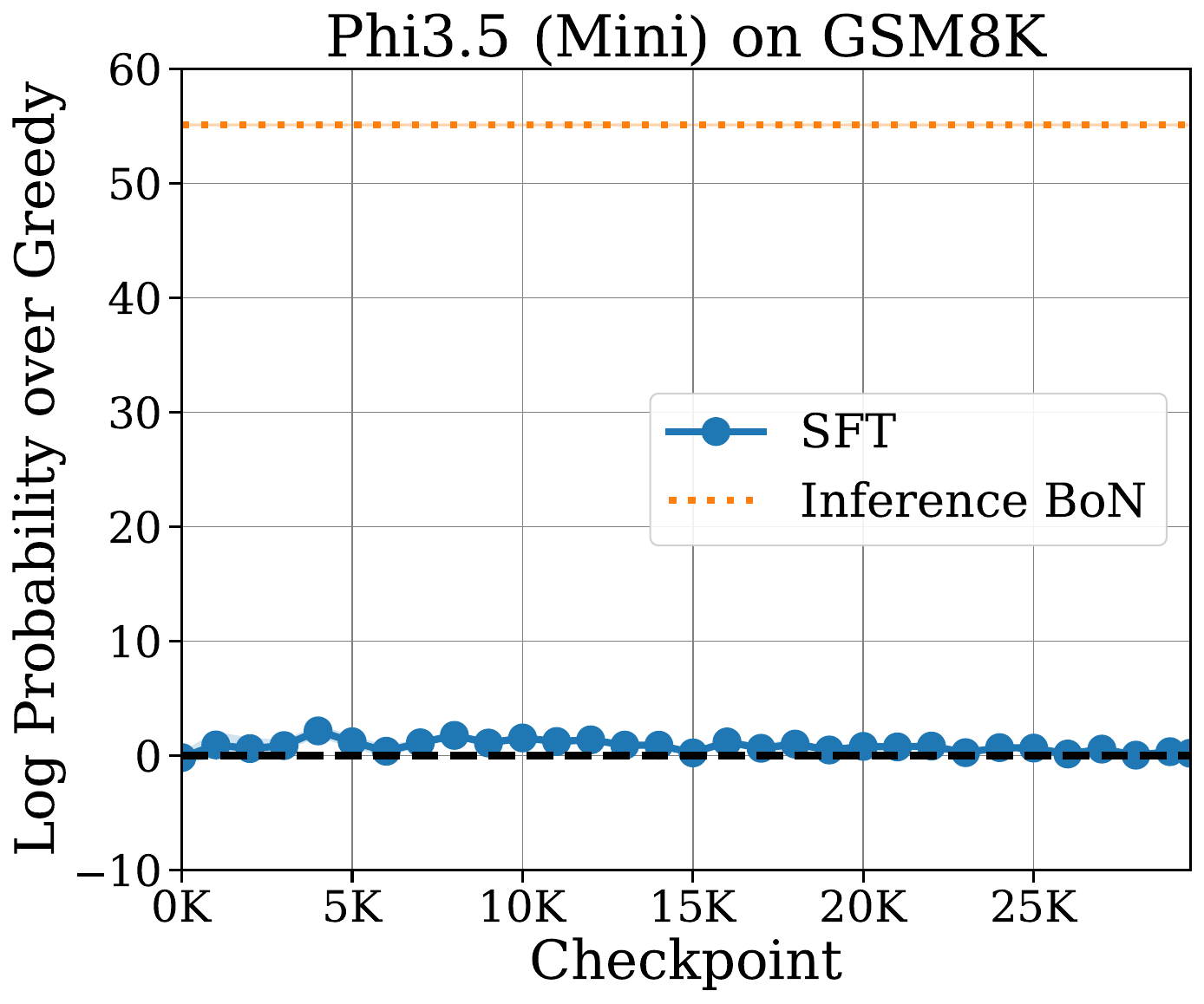}
    } 
    \vspace{-0.15cm}
    \\
    \subfigure[]{
    \includegraphics[width=0.45\textwidth]{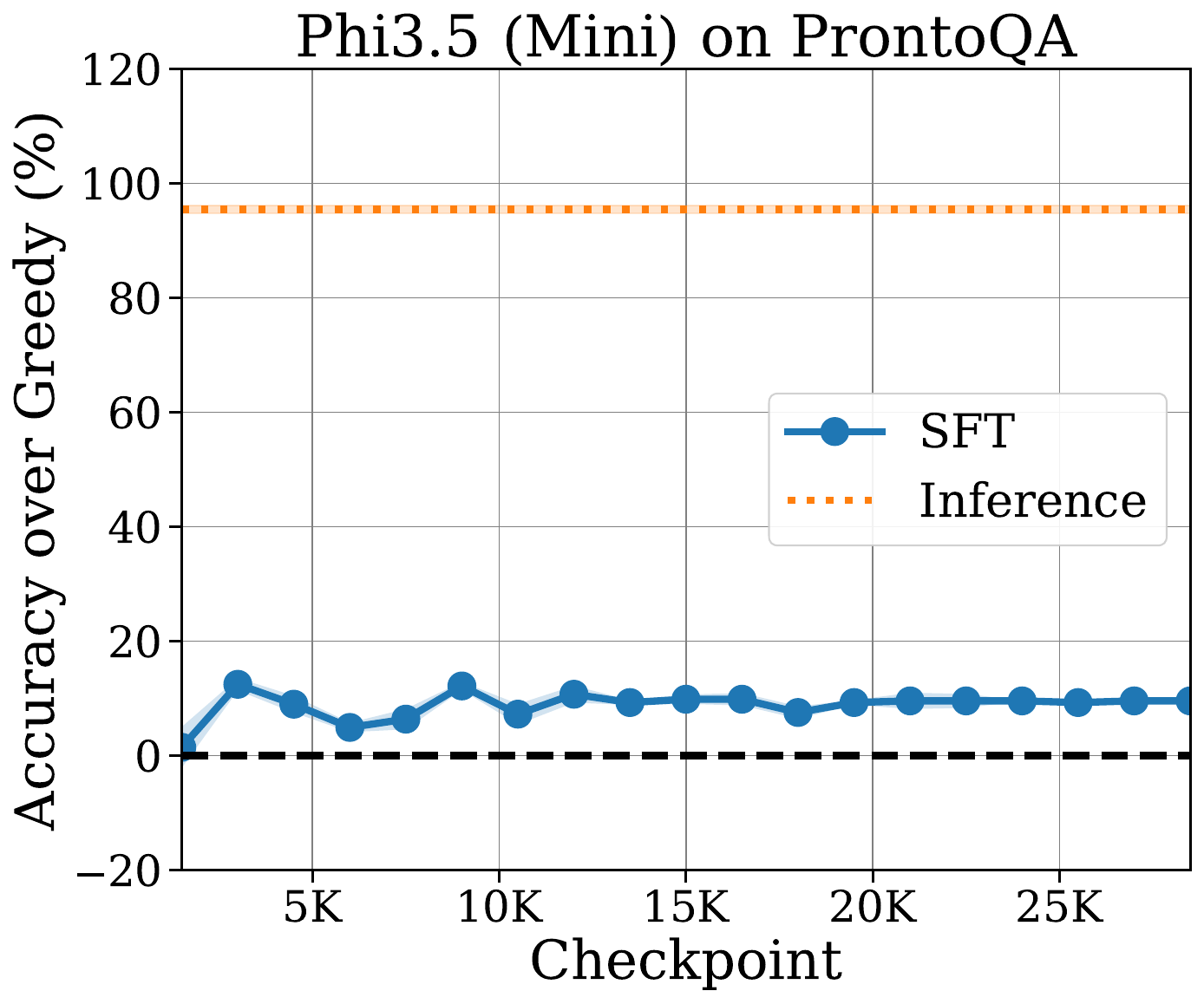}
    } 
    \hfill
    \subfigure[]{
    \includegraphics[width=0.45\textwidth]{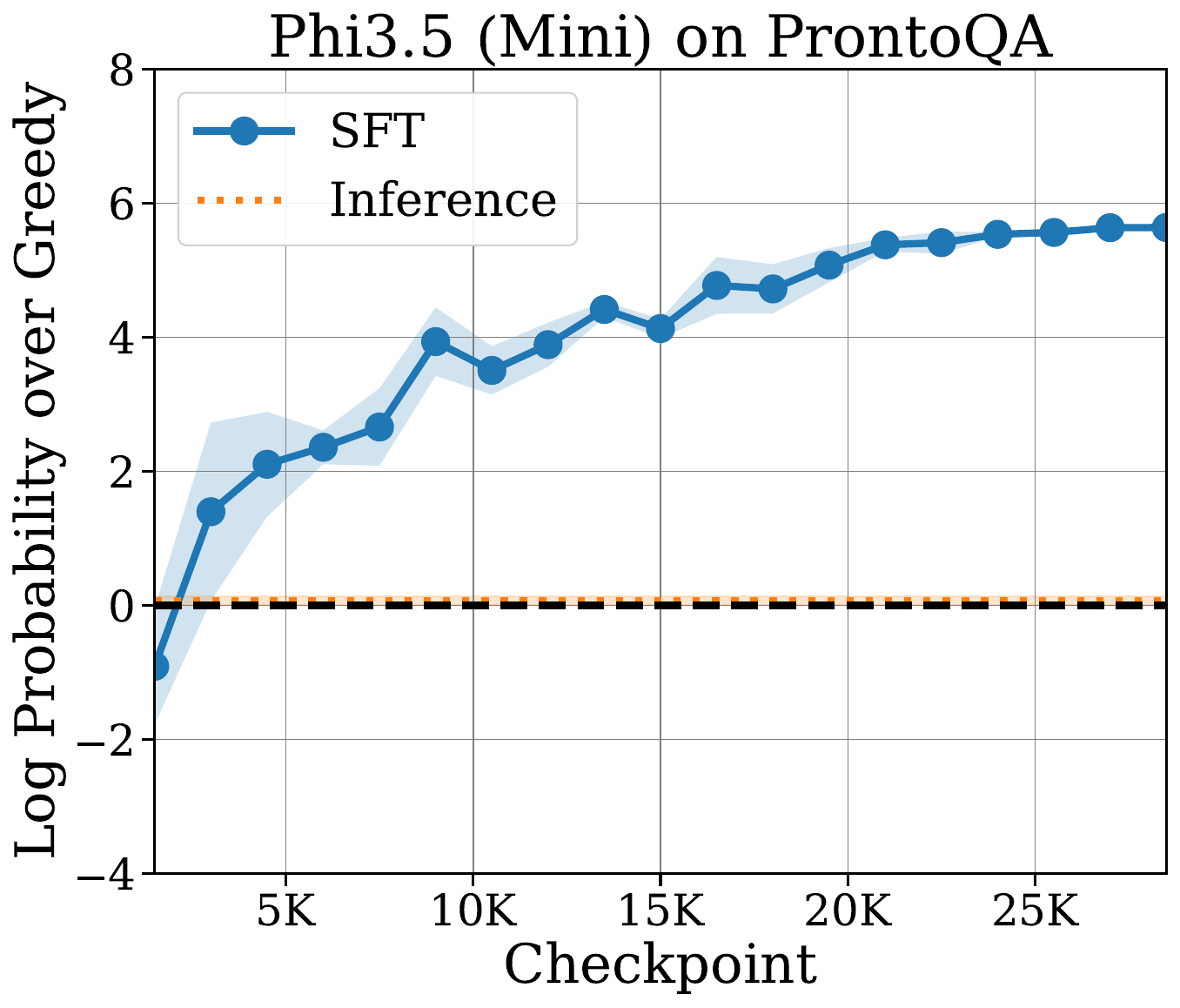}
    } 
    \vspace{-0.15cm}
    \caption{Evolution of \phithreefivemini under \bestofnalg ($N=50$) on different datasets, as measured by (i) \% lift over Greedy in accuracy; and (ii) difference in average sequence-level log-probability of generated responses under the reference model.  The fine-tuned model learns to produce generations with high probability under the reference model, and consequently enjoys an increase in accuracy compared to the base model. However, the model does not fully reach the performance of inference-time BoN sharpening.}
    \label{fig:training_curve_phi}
\end{figure}

\begin{figure}
    \centering
    \subfigure[]{
    \includegraphics[width=0.45\textwidth]{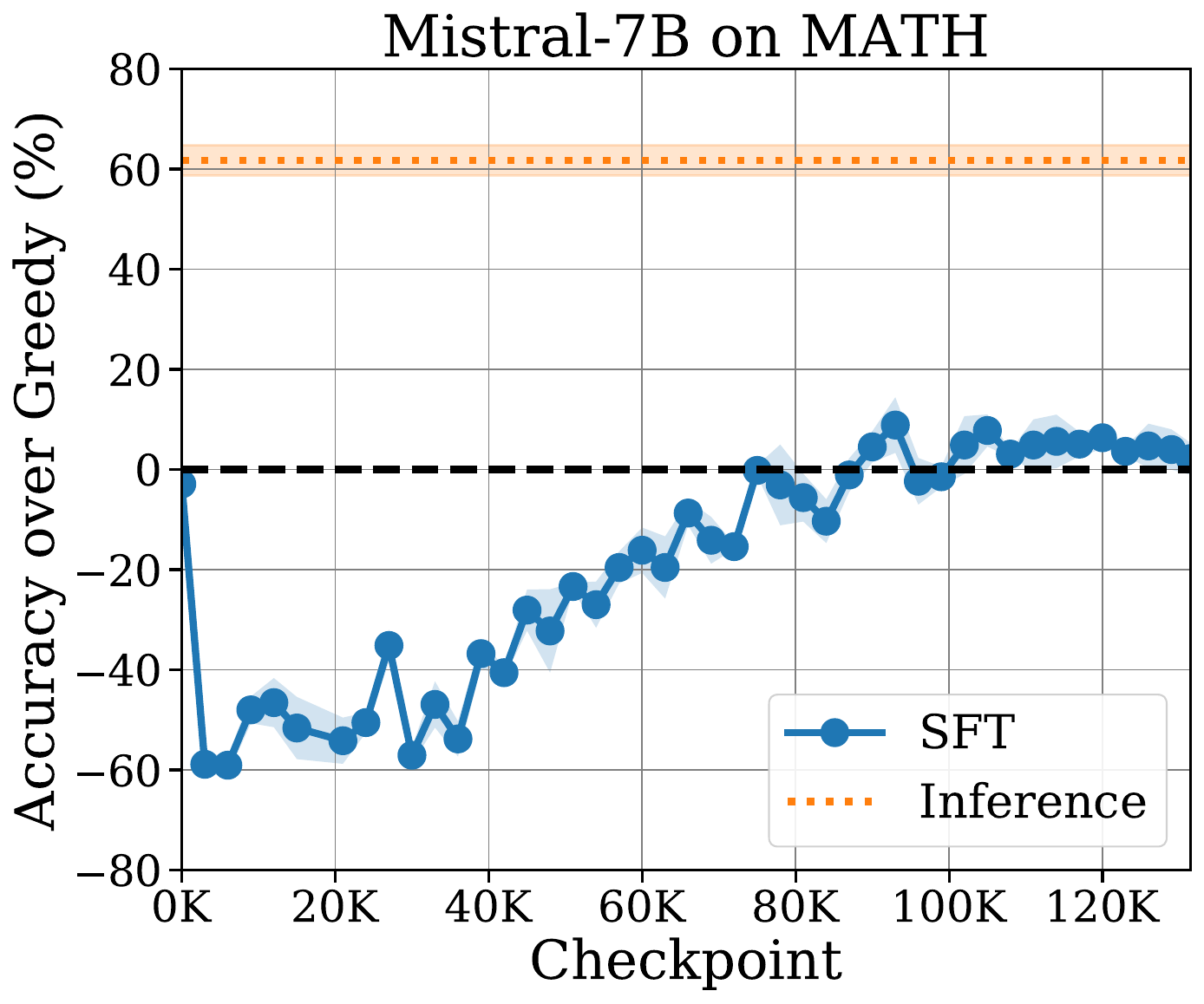}
    } 
    \subfigure[]{
    \includegraphics[width=0.45\textwidth]{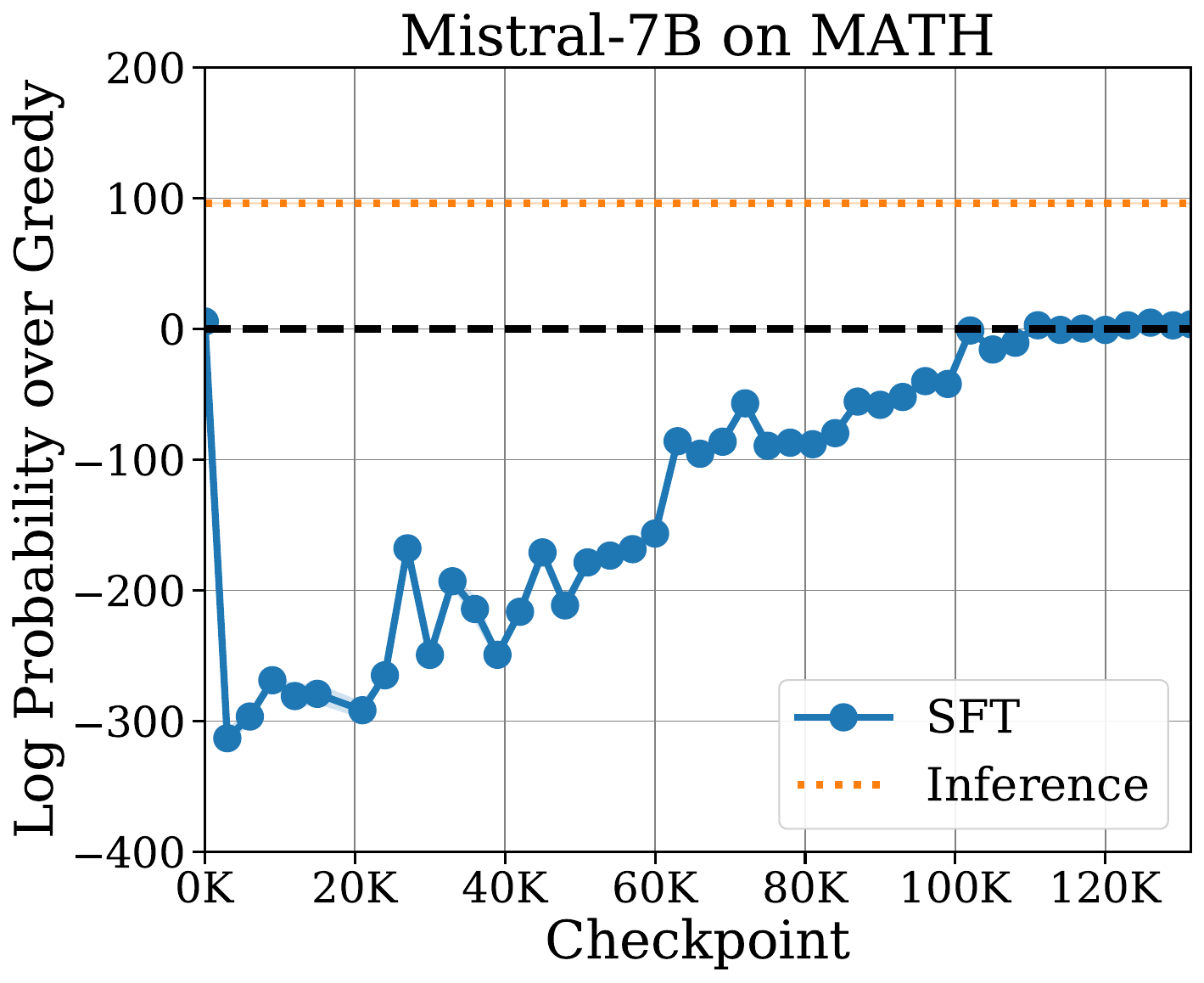}
    } 
    \caption{Evolution of \mistral under \bestofnalg ($N=50$) on \mathdataset, as measured by (i) \% lift over Greedy in accuracy; and (ii) difference in average sequence-level log-probability of generated responses under the reference model.}
    \label{fig:training_curve_mistral}
\end{figure}

\begin{figure}
    \centering
    \subfigure[]{
    \includegraphics[width=0.45\textwidth]{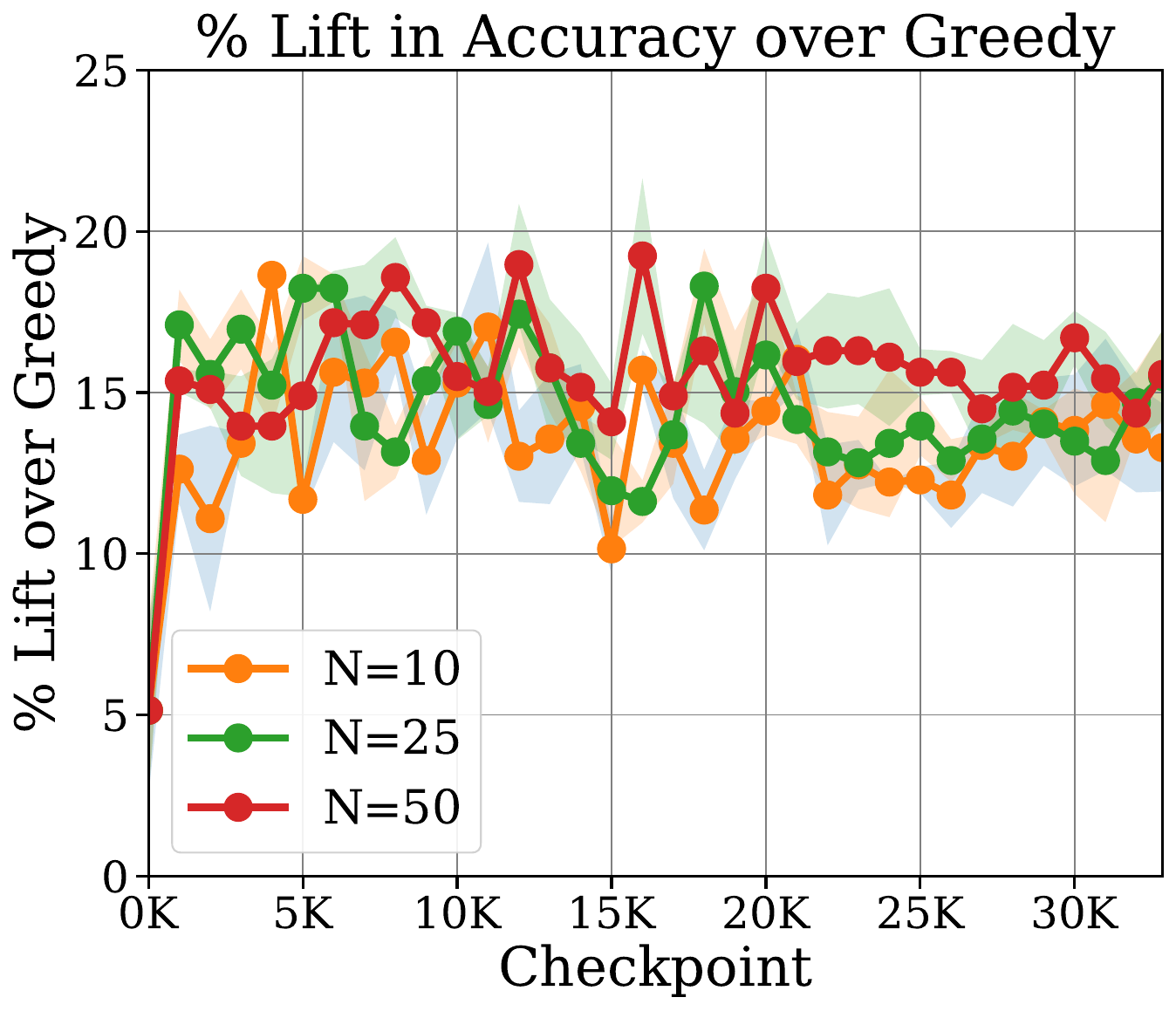}
    }
    \hfill
    \subfigure[]{
    \includegraphics[width=0.45\textwidth]{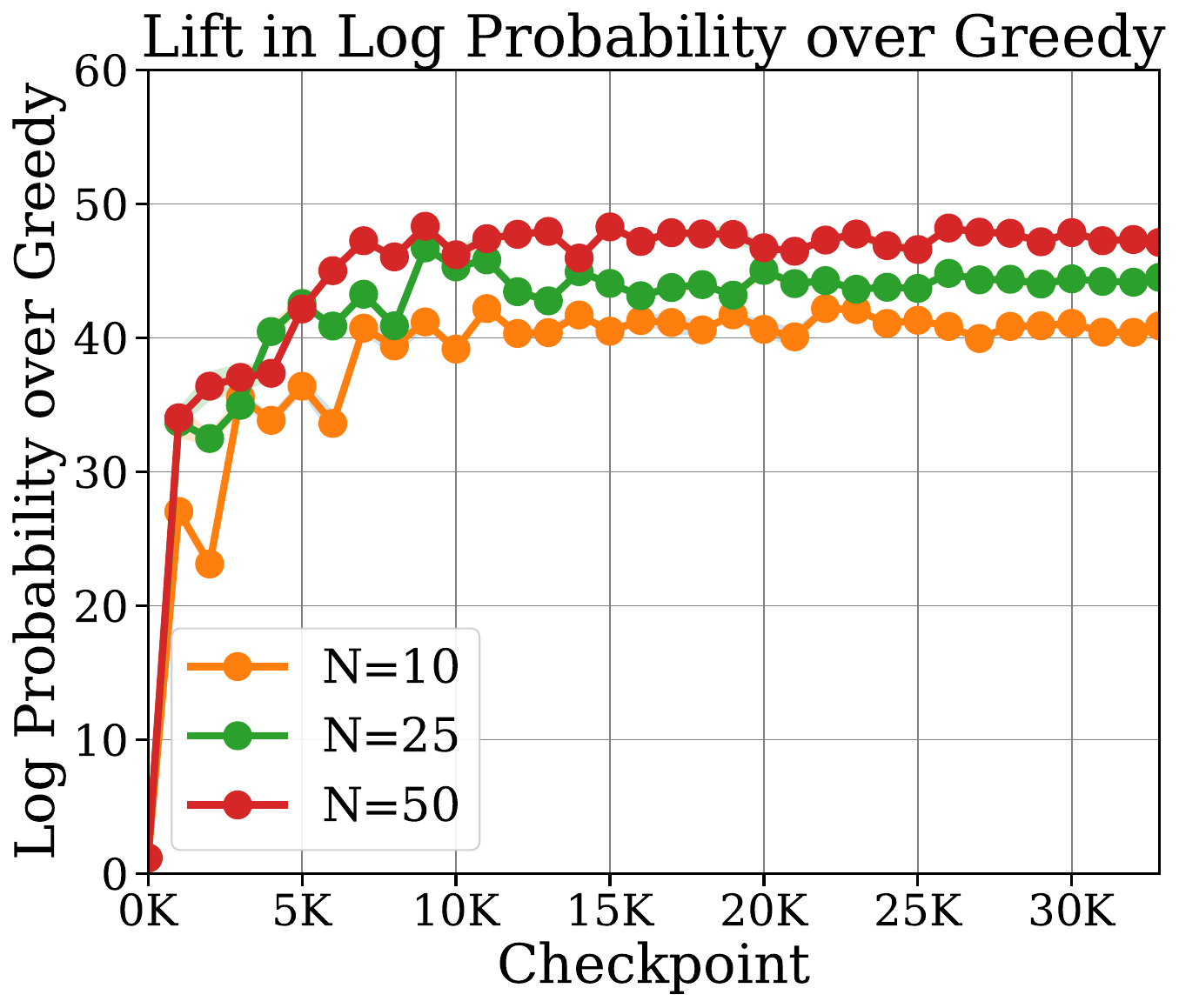}
    }
    \caption{Effect of $N$ on \bestofnalg for \phithreefivemini on \mathdataset.  We report (a) \% lift in accuracy over greedy; and (b) lift in sequence-level log-likelihood (averaged over the dataset).  In both cases, we see that increasing $N$ leads to greater lift, in accordance with theory.}
    \label{fig:sft-N}
\end{figure}

\arxiv{\clearpage}

\section{Detailed Discussion of Related Work}
\label{sec:additional_related}

In this section, we discuss related work in greater detail, including
relevant works not already covered.

\paragraph{Self-improvement and self-training}

Our work is most directly related to a growing body of empirical research that
studies self-improvement/self-training for language models in a
supervision-free setting in which there is no external
feedback~\citep{huang2022large,wang2022self,bai2022constitutional,pang2023language},
and takes a first step toward providing a theoretical understanding
for these methods. There is also a closely related body of research on 
``LLM-as-a-Judge''  techniques, which investigates approaches to
designing self-reward functions $\rself$, often based on specific
prompting techniques \citep{zheng2024judging,yuan2024self,wu2024meta,wang2024self}.

A somewhat complementary line of research develops
algorithms based on self-training and self-play
\citep{zelikman2022star,chen2024self,wu2024self,qu2024recursive}, but
leverages various forms of external feedback (e.g., positive examples
for SFT or explicit reward signal). 
These methods typically outperform feedback-free self-improvement methods~\citep{zelikman2022star}.
However, in many scenarios, obtaining external feedback can be costly or laborious; it may require collecting high-quality labeled/annotated data, rewriting examples in a formal language, etc. 
Thus, these two approaches are not directly comparable.

We also mention that the self-improvement problem we
study is related to a classical line of research on
\emph{self-distillation}
\citep{bucilua2006model,hinton2015distilling,devlin2018bert,pham2021meta,rizve2021defense},
but this specific form of self-training has received limited investigation
in the context of language modeling.

\paragraph{Entropy minimization} Sharpening is also closely related to a line of work on \emph{entropy minimization} or \emph{minimum entropy regularization}, where we seek models that have high predictive accuracy and low entropy/uncertainty. This line of work originated in the semi-supervised learning literature~\citep{grandvalet2004semi} and was popularized as a test-time adaptation method in computer vision~\citep[c.f.,][]{wang2020tent,press2024entropy}. Maximum-likelihood sharpening, especially via RL, is closely related in that \Cref{eq:rlhf} with $\beta \to 0$ and $\rself=\log\piref$ maximizes $\En_\pi[\log\piref(y\mid{}x)]$ rather than $-H(\pi) = \En_\pi[\log \pi(y\mid{}x)]$. (It is important that the latter is optimized continuously with $\piref$ as an initialization, but when this is done it can be seen to sharpen $\piref$, at least heuristically.) Prior work in this direction is largely empirical, focused on computer vision domains with small output spaces $\cY$, and hence studies statistical benefits of entropy minimization. In contrast, we initiate a theoretical study of sharpening, are primarily motivated by applications to language modeling with exponentially large output spaces, and view sharpening primarily as a computational phenomena. However, it would be interesting to understand whether statistical benefits observed in computer vision translate to the language modeling setting.

\paragraph{Alignment and RLHF}
The
specific algorithms for self-improvement/sharpening we study can be
viewed as special cases of standard alignment algorithms, including
classical RLHF methods
\citep{christiano2017deep,bai2022training,ouyang2022training}, direct
alignment \citep{rafailov2024direct}, and (inference-time or
training-time) \bestofn methods
\citep{amini2024variational,sessa2024bond,gui2024bonbon,pace2024west}. \akedit{However,}
the maximum
likelihood sharpening objective \eqref{eq:ml_sharpening} used for our
theoretical results has been relatively unexplored within the
alignment literature.

\paragraph{Inference-time decoding}

Many inference-time decoding
strategies such as greedy/low-temperature decoding, beam-search
\citep{meister2020if}, and chain-of-thought decoding
\citep{wang2024chain} can be viewed as instances of inference-time
sharpening for specific choices of the self-reward function
$\rself$. More sophisticated inference-time search strategies such
tree search and MCTS
\citep{yao2024tree,wan2024alphazero,mudgal2023controlled,zhao2024probabilistic}
are also related, though this line of work frequently makes use of
external reward signals or verification, which is somewhat
complementary to our work.

\paragraph{Theoretical guarantees for self-training}
      On the theoretical side, current understanding of self-training is limited. One line of work, focusing on the
      \emph{self-distillation} objective \citep{hinton2015distilling}
      for binary classification and regression, aims to provide convergence
      guarantees for self-training in stylized setups such as linear
      models
      \citep{mobahi2020self,das2023understanding,das2024retraining,pareek2024understanding},
      with \iftoggle{workshop}{\cite{allen2020towards}}{\citet{allen2020towards}} giving guarantees for feedforward
      neural networks. Perhaps most closely related to our work is
      \iftoggle{workshop}{\cite{frei2022self}}{\citet{frei2022self}}, who show that self-training on a model's
      pseudo-labels can amplify the margin for linear logistic
      regression. However, to the best of our knowledge, our work is the first to study
      self-training in a general framework that subsumes language modeling.

Our results for \rlhfalg are related
to a body of work that provides sample complexity guarantees
for alignment methods
\citep{zhu2023principled,xiong2023gibbs,ye2024theoretical,huang2024correcting,liu2024provably,song2024understanding,xie2024exploratory},
but our results leverage the structure of the
\mlsharp self-reward function $\rself(y\mid{}x)=\log\piref(y\mid{}x)$, and
provide guarantees for the sharpening objective in
\cref{def:sharpening} instead of the usual notion of reward
suboptimality used in reinforcement learning theory.\loose

\dfedit{Lastly, we mention that our results---particularly our
\emph{amortization} perspective on self-improvement---are related to work
that studies representational advantages afforded by additional inference
time \citep{malach2023auto,li2024chain}. These work focus on truly sequential tasks, while our work
focuses on the complementary question of amortizing \emph{parallel}
computation. Thus the representational implications are quite
different.}

      \paragraph{Optimization versus sampling}
      The \mlsharp objective we introduce in \cref{sec:theoretical_framework} connects
      the study of \emph{self-improvement} to a large body of research in
      theoretical computer science on computational tradeoffs (e.g.,
      separations and equivalences) between
      optimization and sampling      \citep{barahona1982computational,kirkpatrick1983optimization,lovasz2006fast,singh2014entropy,ma2019sampling,talwar2019computational,eldan2022spectral}. On the one hand, this
line of research highlights that there exist natural classes of
distributions for which sampling is tractable, yet
maximum likelihood optimization is intractable, and
vice-versa. On the other hand, various works in this line of research also demonstrate \emph{computational reductions} between
optimization and sampling, whereby optimization can be reduced to
sampling and vice-versa. 

Our setting indeed includes natural model classes where one should not expect there to be a computational reduction from optimization ($\argmax_{y\in\cY}\piref(y\mid{}x)$) to sampling
($y\sim\piref(\cdot\mid{}x)$), and hence inference-time sharpening is computationally intractable (\cref{prop:softmax-np-hardness}). Of course, coverage assumptions eliminate this intractability. For training-time sharpening (where the goal is to \emph{amortize} across prompts by training a sharpened model, as formulated in \cref{sec:theoretical_framework}) the obstacle in natural, concrete model classes is not just computational but in fact \emph{representational} (\cref{prop:softmax-representational-lb}). %
Regarding the latter point, we note that while amortized
Bayesian inference has received extensive investigation empirically
\citep{beal2003variational,gershman2014amortized,swersky2020amortized,bengio2021flow,hu2023amortizing},
we are unaware of theoretical guarantees outside of this work.

\section{Guarantees for Inference-Time Sharpening}
\label{sec:inference}

In this section, we give theoretical guarantees for the inference-time \bestofn sampling algorithm for sharpening described in \cref{sec:autoreg_sharpening}, under the \mlsharp self-reward function \[\rself(y\mid{}x;\piref) = \log\piref(y\mid{}x).\]

Recall that given a prompt $x\in\cX$, the inference-time \bestofn sampling algorithm draws $N$ responses
$y_1,\ldots,y_n\sim\piref(\cdot\mid{}x)$, then return the response
$\yhat=\argmax_{y_i}\log\piref(y_i\mid{}x)$. We show that this
algorithm returns an approximate maximizer for the \mlsharp objective
whenever the base policy $\piref$ has sufficient coverage. For a parameter $\gamma\in[0,1)$ we define
\begin{align}
  \Ygamma(x)\ldef{}\crl*{y\mid{}
  \piref(y\mid{}x)\geq{}(1-\gamma)\cdot\max_{y\in\cY}\piref(y\mid
  x)}
\end{align}%
as the set of $(1-\gamma)$-approximate maximizers for
$\log\piref(y\mid{}x)$ (see \cref{sec:approx} for background
on  $\Ygamma(x)$).\loose

\begin{proposition}
    \label{prop:bestofn_inference_general}
    Let a prompt $x\in\cX$ be given. For any $\deltafail\in(0,1)$ and
    $\gamma\in[0,1)$,
    as long as
        \begin{align}
      N \geq \frac{\log(\deltafail^{-1})}{\piref(\Ygamma(x)\mid{}x)},
        \end{align}
        inference-time \bestofn sampling produces a response
        $
    \yhat\in\Ygamma(x)$ with probability at least $1-\deltaf$.
  \end{proposition}

\begin{proof}[\pfref{prop:bestofn_inference_general}]
    Fix a prompt $x \in \cX$, failure probability $\deltaf \in (0, 1)$,
    and parameter $\gamma\in(0,1)$.
    By definition of the set $\Ygamma(x)$, $\yhat \in \Ygamma(x)$ if
    and only if there exists $i \in [N]$ such that $y_i\in\Ygamma(x)$. The complement of this event, i.e., that $y_i \notin \Ygamma(x)$ for all $i \in [N]$, has probability 
    \begin{align}
        \Pr\prn*{y_i \notin \Ygamma(x),{}\forall i \in [N]} = \prn*{1 - \piref(\Ygamma(x)\mid{}x)}^N.
    \end{align}
    Rearranging the \rhs, we have 
    \begin{align}
        \prn*{1 - \piref(\Ygamma\mid{}x)}^N 
        =&~ 
        \exp\prn*{-N\log\prn*{\frac{1}{1-\piref(\Ygamma\mid{}x)}}}
        \le
        \exp\prn*{-N\cdot\piref(\Ygamma\mid{}x)},
    \end{align}
    since $\log(x) \ge 1-\frac{1}{x}$ for $x > 0$, which implies that $\log\prn*{\frac{1}{1-\piref(\Ygamma\mid{}x)}} \ge \piref(\Ygamma\mid{}x)$. 
    Thus, as long as $N \ge
    \frac{\log(\deltaf^{-1})}{\piref(\Ygamma\mid{}x)}$, we have
    \begin{align}
        \Pr\prn*{y_i \notin \Ygamma(x),{}\forall i \in [N]} 
        \le 
        \exp\prn*{-N\cdot\piref(\Ygamma\mid{}x)}
        \le 
        \exp(-\log(\deltaf^{-1})) = \deltaf. 
    \end{align}
    We conclude that with probability at least $1-\deltaf$, there exists $i \in [N]$ such that $y_i \in \Ygamma(x)$, and $\yhat \in \Ygamma(x)$ as a result.  

\end{proof}

\section{Guarantees for \bestofnalg with Adaptive Sampling}
\label{sec:adaptive}

  \bestofnalg is a simple and natural self-training scheme, and
  converges to a sharpened policy as $n, N\to \infty$. However, using
  a fixed response sample size $N$ may be wasteful for prompts where
  the model is confident. To this end, in this section we introduce
  and analyze, a variant of
  \bestofnalg based on \emph{adaptive sampling}, which adjusts the
  number of sampled responses adaptively.

  \paragraph{Algorithm}
  We present the adaptive \bestofnalg algorithm only for the special case of the
  maximum likelihood sharpening self-reward. Let a \emph{stopping parameter} $\mu>0$ be given. For $x_i\in\cX$,
  and $y_{i,1},y_{i,2}\ldots\sim\piref(\cdot\mid{}x_i)$, define a
  stopping time (e.g., \iftoggle{workshop}{\cite{benjamini1995controlling}}{ \citet{benjamini1995controlling}}) via:
  \begin{align}
    \label{eq:bon_stopping}
    \taumu(x_i) := \inf \left\{k : \frac{1}{\max_{1 \le j \le k} \piref(y_{i,j}\mid{}x_i)} \le \frac{k}{\mu} \right\}.
  \end{align}
  The adaptive \bestofnalg algorithm computes adaptively sampled
  responses $\yada_i$ via
  \[
    \yada_i\sim\argmax\crl*{\log\piref(y_{i,j}\mid{}x_i)\mid{}y_{i,1},\ldots,y_{i,\taumu(x_i)}},
  \]
  then trains the sharpened model through SFT:
  \[
    \pihatada =
    \argmax_{\pi\in\Pi}\sum_{i=1}^{n}\log\pi(\yada_i\mid{}x_i).
  \]
  Critically,
  by using scheme in \cref{eq:bon_stopping}, this algorithm can stop
  sampling responses for the prompt $x_i$ if it becomes clear that the
  confidence is large.

\paragraph{Theoretical guarantee}  

We now show that adaptive \bestofnalg enjoys provable benefits over
its non-adaptive counterpart through the dependence on the accuracy
parameter $\eps>0$.

  Given $x\in\cX$, and $y_1,y_2\ldots\sim\piref(x)$, let
  $\taumu(x) := \inf \crl[big]{k : \frac{1}{\max_{1 \le i \le k}
      \piref(y_i\mid{}x)} \le k/\mu }$, and define a random variable 
  $\yada(x)\sim\argmax\crl*{\log\piref(y_i\mid{}x)\mid{}y_1,\ldots,y_{\taumu}\sim\piref(x)}$.
Let $\piada(x)$ denote the distribution over $\yada(x)$. We make the
following realizability assumption.
\begin{assumption}\label{ass:ada-realizability}
  The model class $\Pi$ satisfies $\piada\in\Pi$.
\end{assumption}
Compared to \bestofnalg, we require a somewhat stronger coverage
coefficient given by 
  \begin{align}
    \Cstarb = \En_{x\sim\mu}\brk*{\frac{1}{\max_{y\in\cY}\piref(y\mid{}x)}}.
  \end{align}
  This definition coincides with \cref{eq:cstar} when the arg-max
  response is unique, but is larger in general.\loose

Our main theoretical guarantee for adaptive \bestofnalg is as follows.
  \begin{theorem}
    \label{thm:bestofn_adaptive}
    Let $\delta,\rho\in (0,1)$ be given. Set $\mu =
    \ln(2\delta^{-1})$, and assume \cref{ass:ada-realizability}
    holds. Then with probability at least $1-\deltaf$, the adaptive \bestofnalg
    algorithm has
\begin{align}
  \Pr_{x \sim \cdist}[\pihat(\Ystar(x) \mid x) \leq 1 - \delta] \lesssim  \frac{\log(|\Pi|\deltaf^{-1})}{\delta n} ,
\end{align}
and has sample complexity $\En[m] = n\cdot\Cstarb
\log(\delta^{-1})$. Taking $n \gtrsim
\frac{\log(|\Pi|\deltaf^{-1})}{\delta \epsilon} $ ensures that with
probability at least $1-\deltaf$, $\Pr_{x \sim \cdist}[\pihat(\Ystar(x) \mid x) \leq 1 -
  \delta]\leq\eps$, and gives total sample complexity
\begin{align}
  \En\brk*{m} = O\left(\frac{\Cstarb\log(\abs{\Pi}\delfail^{-1})\log(\delta^{-1})}{\delta\eps}\right).
\end{align}
\end{theorem}
Compared to the result for \bestofnalg in \cref{thm:bestofn}, this
shows that adaptive \bestofnalg achieves sample complexity scaling
with $\frac{1}{\eps}$ instead of $\frac{1}{\eps^2}$. We believe the dependence on $\Cstarb$ for this algorithm
is tight, as the adaptive stopping rule used in the algorithm can be
overly conservative when $\abs{\Ystar(x)}$ is large.
\loose

\paragraph{A matching lower bound}
We now prove a complementary lower bound, which shows that the
$\eps$-dependence in \cref{thm:bestofn_adaptive} is tight. To do so,
we consider the following adaptive variant of the \framework framework.
\begin{definition}[Adaptive sample-and-evaluate framework]
  In the \textbf{Adaptive Sample-and-Evaluate} framework, the
  learner is allowed to sample $n$ prompts $x \sim \cdist$, and sample an
  arbitrary, adaptively chosen number of samples
  $y_1,y_2,\dots \sim \piref(\cdot \mid x)$ before sampling a new
  prompt $x' \sim \cdist$. In this framework we define sample
  complexity $m$ as the total number of pairs $(x,y)$ sampled by the
  algorithm, which is a random variable.
\end{definition}
Our main lower bound is as follows.
\begin{theorem}[Lower bound for sharpening under adaptive sampling]
  \label{thm:lower_adaptive}
  Fix an integer $d \ge 1$ and parameters $\epsilon \in (0,1)$ and $C
  \ge 1$. There exists a class of models $\Pi$ such that (i) $\log |\Pi|
  \eqsim d (1+\log(C \epsilon^{-1}))$, (ii) $\sup_{\pi \in \Pi}
  \Cstar(\pi) \lesssim C$, and (iii) $\Ypi(x)$ is a singleton for all
  $\pi\in\Pi$, for which any sharpening algorithm $\pihat$ in the
  adaptive \framework framework that
  achieves $\En\brk*{\Pr_{x \sim \cdist}[\pihat(\Ypi[\piref](x)\mid{}x) > 1/2]}  \ge 1
  - \epsilon$ for all $\piref \in \Pi$ must collect a total number of samples $m = n\cdot{}N$ at least
\begin{align}
  \En\brk{m}
  \gtrsim 
        \frac{ C\log |\Pi|}{\epsilon\cdot{}(1+\log(C
  \epsilon^{-1}))}.
\end{align}
\end{theorem}
\cref{thm:lower_adaptive} is a special case of a more general theorem, \cref{thm:lower_general}, which is stated and proven in \cref{sec:proofs_lower}.

\section{Computational and Representational Challenges in Sharpening}
\label{sec:hardness}
In this section, we make several basic observations about the inherent computational and representational challenges of \mlsharp. First, in \cref{sec:computational}, we focus on computational challenges, and show that computing a sharpened response for a given prompt $x$ can be computationally intractable in general, even when sampling $y\sim\piref(\cdot\mid{}x)$ can be performed efficiently. Then, in \cref{sec:representational}, we shift our focus to representational challenges, and show that even if $\piref$ is an autoregressive model, the ``sharpened'' version of $\piref$ may not be representable as an autoregressive model with the same architecture. These results motivate the statistical assumptions (coverage and realizability) made in our analysis of \sftalg and \rlhfalg in \cref{sec:theoretical_analysis}.

To make the results in this section precise, we work in perhaps the simplest special case of autoregressive language modelling, where the model class consists of \emph{multi-layer linear softmax models}. Formally, let $\MX$ be the space of prompts, and let $\MY := \MV^H$ be the space of responses, where $\MV$ is the vocabulary space and $H$ is the horizon. For a collection of fixed/known $d$-dimensional feature mappings $\phi_h: \MX \times \MV^h \to \RR^d$ and a norm parameter $B$, we define the model class $\Pi_{\phi,B,H}$ as the set of models
\begin{align}
  \label{eq:linear_softmax}
\pi_\theta(y_{1:H}\mid{} x) = \prod_{h=1}^H \pi_{\theta_h}(y_h\mid{} x,y_{1:h-1})
\end{align}
where 
\[\pi_\theta(y_h\mid{} x,y_{1:h-1}) \propto \exp(\langle \phi(x,y_{1:h}), \theta_h\rangle)\]
and $\theta = (\theta_1,\dots,\theta_H) \in (\RR^d)^H$ is any tuple with $\norm{\theta_h}_2 \leq B$ for all $h \in [H]$.

\subsection{Computational Challenges}
\label{sec:computational}

Given query access to $\phi$, for any given parameter vector $\theta$ and prompt $x$, \emph{sampling} from a linear softmax model $\pi_\theta$ (\cref{eq:linear_softmax}) is computationally tractable, since it only requires time $\poly(H, |\MV|, d)$. Similarly, \emph{evaluating} $\pi_\theta(y_{1:H}\mid{} x)$ for given prompt $x$ and response $y_{1:H}$ is computationally tractable. %
However, the following proposition shows that computing the sharpened response $\argmax_{y_{1:H} \in \MV^H} \pi_\theta(y_{1:H}\mid{} x)$ for a given parameter $\theta$ and response $x$ is $\NP$-hard. Hence, even inference-time sharpening is computationally intractable in the worst case.

\begin{proposition}\label{prop:softmax-np-hardness}
  Set $\MX = \{\perp\}$ and $\MV = \{-1,1\}$. Set $d = d(H) := H+H^2+H^3$. Identifying $[d]$ with $[H] \sqcup [H]^2 \sqcup [H]^3$, we define $\phi_h: \MX \times \MV^h \to \RR^d$ by $\phi_h(\perp,y_{1:h})_i = y_i$ and $\phi_h(\perp,y_{1:h})_{(i,j)} = y_iy_j$ and $\phi_h(\perp,y_{1:h})_{(i,j,k)} = y_iy_jy_k$. There is a function $B(H) \leq \poly(H)$ such that the following problem is $\NP$-hard: given $\theta = (\theta_1,\dots,\theta_H)$ with $\max_{h \in [H]} \norm{\theta_h}_2 \leq B(H)$, compute any element of $\argmax_{y_{1:H} \in \MV^H} \pi_\theta(y_{1:H}\mid{} x)$.
  \dfc{Would be good to add a sentence (perhaps as a footnote) explaining how we define input length and emphasizing that we are considering a sequence of problems indexed by $H$ when we talk about np-hardness}
\end{proposition}
\dfedit{Note that our results in \cref{sec:theoretical_analysis} and \cref{sec:inference} bypass this hardness through the
  assumption that the coverage parameter $\Cstar$ is bounded.}
\begin{proof}[\pfref{prop:softmax-np-hardness}]\dfedit{Fix $H$ and recall that $d(H)=H+H^2+H^3$. We define three collection of basis vectors: $\crl*{e_h}_{h\in\brk{H}}$ cover the first $H$ coordinates, $\crl*{e_{(h,h')}}_{h,h'\in\brk{H}^2}$ cover the next $H^2$ coordinates, and $\crl*{e_{(h,h', h'')}}_{h,h',h''\in\brk{H}^3}$ cover the last $H^3$ coordinates.
    } Suppose we define $\theta_1,\dots,\theta_{H-2} = 0$, so that $\pi_\theta(y_h|x,y_{1:h-1}) = 1/2$ for all $1 \leq h \leq H-2$. Define $\theta_{H-1} = \sum_{1 \leq i,j \leq H-2} J_{ij} e_{(i,j,H-1)}$ for a matrix $J \in \RR^{(H-2)\times(H-2)}$ to be specified later, and define $\theta_H = \frac{B}{2} (e_{(H-1,H)} + e_H)$. Then $2^{H-2}\cdot \pi_\theta(y_{1:H}\mid{} \perp) \leq 1/2$ for any $y_{1:H}$ with $y_{H-1} = -1$ or $y_H = -1$, \dfedit{since this implies that $\pi_{\theta_H}(y_H\mid{}\perp,y_{1:H-1})\leq{}1/2$}. Meanwhile, for any $y_{1:H}$ with $y_{H-1} = y_H = 1$, we have\loose
\[2^{H-2}\cdot\pi_\theta(y_{1:H}\mid{} \perp) = \frac{\exp\prn*{\sum_{i,j\dfedit{\leq{}H-2}} J_{ij}y_iy_j}}{\exp\prn*{\sum_{i,j\dfedit{\leq{}H-2}}J_{ij}y_iy_j} + \exp\prn*{-\sum_{i,j\dfedit{\leq{}H-2}} J_{ij}y_iy_j}} \cdot \frac{\exp(B)}{\exp(B)+\exp(-B)}.\]
Let $G$ be any graph on vertex set $[H-2]$ and let $J = -A(G)$ where $A(G)$ is the adjacency matrix of $G$. Then among $y_{1:H}$ with $y_{H-1} = y_H = 1$, $2^{H-2}\cdot\pi_\theta(y_{1:H}\mid{} \perp)$ is maximized when $y_{1:H-2}$ corresponds to a max-cut in $G$. If $G$ has an odd number of edges, then some max-cut removes strictly more than
half of the edges, and for the corresponding sequence $y_{1:H}$ we have $2^{H-2}\cdot\pi_\theta(y_{1:H}\mid{} \perp) \geq (1/2+\Omega(1)) \cdot (1 - \exp(-\Omega(B)))$, which is greater than $1/2$ when we take $B := H$ and $H$ is sufficiently large. Thus, computing $\argmax_{y_{1:H} \in \MV^H} \pi_\theta(y_{1:H}\mid{} \perp)$ yields a max-cut of $G$. It is well-known that computing a max-cut in a graph is $\NP$-hard, and the assumption that $G$ has an odd number of edges is without loss of generality.
\end{proof}

\subsection{Representational Challenges}
\label{sec:representational}

To give provable guarantees for our sharpening algorithms, we required certain \emph{realizability} assumptions, which in particular posited that the model class actually contains a ``sharpened'' version of $\piref$ (\cref{assumption:bon-realizability,ass:realizability_dpo}). In the simple example of a \emph{single-layer} linear softmax model classes (corresponding to $H=1$ in the above definition), \cref{ass:realizability_dpo} is in fact satisfied, and the sharpened model can be obtained by increasing the temperature of $\piref$. However, multi-layer linear softmax models with $H \gg 1$ are more realistic. The following proposition shows that as soon as $H \geq 2$, multi-layer linear softmax model classes may not be closed under sharpening. This illustrates a potential drawback of training-time sharpening compared to inference-time sharpening, which requires no realizability assumptions. It also provides a simple example where greedy decoding does not yield a sequence-level arg-max response (since increasing temperature in a multi-layer softmax model class exactly converges to the greedy decoding).

\begin{proposition}\label{prop:softmax-representational-lb}
Let $\MX = \{\perp\}$, $\MV = [n]$, and $H = d = 2$. For any $n$ sufficiently large, there is a multi-layer linear softmax policy class $\Pi_{\phi,B,H}$ and a policy $\piref \in \Pi_{\phi,B,H}$ such that $y^\star_{1:H} := \argmax_{y_{1:H} \in \MV^H} \pi_\theta(y_{1:H}\mid{} \perp)$ is unique, but for all $B' > B$ and $\pi \in \Pi_{\phi,B',H}$, it holds that $\pi(y^\star_{1:H}\mid{}\perp) \leq 1/2$.
\end{proposition}

\begin{proof}[\pfref{prop:softmax-representational-lb}]
Throughout, we omit the dependence on the prompt $\perp$ for notational clarity. Since $H=2$, the model class consists of models $\pi_\theta$ of the form
\begin{equation} 
\pi_\theta(a) = \pi_{\theta_1}(y_1) \pi_{\theta_2}(y_2\mid{} y_1) = \frac{\exp(\langle \phi_1(y_1),\theta_1\rangle)}{Z_{\theta_1}} \frac{\exp(\langle \phi_2(y_{1:2}), \theta_2\rangle)}{Z_{\theta_2}(y_1)}\label{eq:softmax-h2}
\end{equation}
\dfedit{for $Z_{\theta_1}\ldef{}\sum_{y_1\in\cV}\exp(\langle \phi_1(y_1),\theta_1\rangle)$ and $Z_{\theta_2}(y_1)\ldef{}\sum_{y_2\in\cV}\exp(\langle \phi_2(y_{1:2}), \theta_2\rangle)$.}

Define $\phi_1$ by:
\[\phi_1(i) = \begin{cases} e_1 & \text{ if } i = 1 \\ e_1 & \text{ if } i = 2 \\ e_2 & \text{ if } i \geq 3 \end{cases}.\]
Define $\phi_2$ by:
\[\phi_2(i,j) = \begin{cases} e_1 & \text { if } i = 2, j = 1 \\ e_2 & \text { if } i = 2, j \neq 1 \\ 0 & \text { if } i \neq 2 \end{cases}.\]

Define $\piref := \pi_{\theta^\star}$ where $\theta^\star_1 := \theta^\star_2 :=B\cdot{}e_1$ for a parameter $B \geq \log(n)$. Then $\piref(1) = \piref(2)$ and $\piref(i) \leq e^{-B} \piref(2)$ for all $i \in \{3,\dots,n\}$. Moreover, $\piref(\cdot\mid{} i) = \Unif([n])$ for all $i \neq 2$, and $\piref(j\mid{} 2) \leq e^{-B} \piref(1\mid{} 2)$ for all $j \neq 1$. Thus,
\[\piref(2,1) = \piref(2) \piref(1\mid{} 2) \geq \frac{1}{2 + (n-2)e^{-B}} \cdot \frac{1}{1 + (n-1)e^{-B}} \geq \Omega(1)\]
whereas $\piref(i,j) = O(1/n)$ for all $(i,j) \neq (2,1)$. Thus, $(2,1)$ is the sequence-level argmax for sufficiently large $n$. However, for any $\pi_\theta$ of the form described in \cref{eq:softmax-h2}, we have
\[\pi_\theta(2,1) \leq \pi_\theta(2) \leq \frac{\pi_\theta(2)}{\pi_\theta(1)+\pi_\theta(2)} = \frac{1}{2}\]
since $\phi(1) = \phi(2)$. This means that there is no $B'$ for which $\Pi_{\phi,B',H}$ contains an $(\eps,\delta)$-sharpened policy for $\piref$ for any $\delta>1/2$.
\end{proof}

\newpage

\clearpage

\part{Proofs}

\section{Preliminaries}
\label{sec:proof_preliminaries}

\subsection{Guarantees for Approximate Maximizers}
\label{sec:approx}

Recall that the theoretical guarantees for sharpening algorithms in \cref{sec:theoretical_analysis} provide convergence to the set $\Ystar(x)\ldef{}\argmax_{y\in\cY}\piref(y\mid{}x)$ of
(potentially non-unique) maximizers for the \mlsharp self-reward
function $\log\piref(y\mid{}x)$. These guarantees require that the
base model $\piref$ places sufficient provability mass on $\Ystar(x)$,
which may not always be realistic. To address this, throughout this appendix we state and prove more general versions of our theoretical results that allow for approximate maximizers, and consequently enjoy weaker coverage assumptions

For a parameter $\gamma\in[0,1)$ we define
\begin{align}
  \Ygamma(x)\ldef{}\crl*{y\mid{}
  \piref(y\mid{}x)\geq{}(1-\gamma)\cdot\max_{y\in\cY}\piref(y\mid
  x)}
\end{align}%
as the set of $(1-\gamma)$-approximate maximizers for $\log\piref(y\mid{}x)$.
We quantify the
quality of a
sharpened model
as follows.
\begin{definition}[Sharpened model]
  \label{def:sharpening_general}
    We say that a model $\pihat$ is $(\eps,\delta,\gamma)$-sharpened relative
    to $\piref$ if
    \begin{align}
      \bbP_{x\sim\cdist}\brk*{\pihat\prn*{\Ygamma(x)\mid{}x}\geq{}1-\delta} \geq{} 1-\eps.
    \end{align}
  \end{definition}
That is, an 
$(\eps,\delta,\gamma)$-sharpened policy places at least
$1-\delta$ mass on $(1-\gamma)$-approximate arg-max responses on all but an $\eps$-fraction
of prompts under $\mu$.

Lastly, we will make use of the following generalized coverage coefficient
\begin{align}
  \Cstarg = \En_{x\sim\cdist}\brk*{\frac{1}{\piref(\Ygamma(x)\mid{}x)}},
\end{align}
which has $\Cstarg\leq{}\Cstar$.

\subsection{Technical Tools}

For a pair of probability measures $\bbP$ and $\bbQ$ with
a common dominating measure $\omega$, Hellinger distance is defined
via
\begin{align}
\Dhels{\bbP}{\bbQ}=\int\prn*{\sqrt{\frac{\mathrm{d}\bbP}{\mathrm{d}\omega}}-\sqrt{\frac{\mathrm{d}\bbQ}{\mathrm{d}\omega}}}^2\mathrm{d}\omega.
\end{align}

\begin{lemma}[MLE for conditional density estimation (e.g., \iftoggle{workshop}{\cite{wong1995probability,Sara00,zhang2006from}}{\citet{wong1995probability,Sara00,zhang2006from}})]
  \label{lem:mle}
  Consider a conditional density $\pistar : \cX \rightarrow
  \Delta(\cY)$. Let $\cD = \{(x_i,y_i)\}_{i=1}^n$ be a dataset in which $(x_i,y_i)$ are drawn i.i.d. as $x_i \sim \cdist \in \Delta(\cX)$ and $y_i \sim \pistar(\cdot\mid{}x)$. Suppose we have a finite function class $\Pi\subset(\cX\to\Delta(\cY))$ such that $\pistar \in \Pi$. Define the maximum likelihood estimator
  \begin{align}
    \pihat \ldef \argmax_{\pi \in \Pi} \sum_{(x,y) \in \cD} \log \pi(y\mid{}x).
  \end{align}
  Then with probability at least $1-\deltaf$, 
  \begin{align}
    \En_{x \sim \cdist}\brk*{\Dhels{\pihat(\cdot\mid{}x)}{\pistar(\cdot\mid{}x)}} \le \frac{2\log(|\Pi|\deltaf^{-1})}{n}.
  \end{align}    
\end{lemma}

\begin{lemma}[Elliptic potential lemma]\label{lemma:elliptic}
Let $\lambda,K>0$, and let $A_1,\dots,A_T \in \RR^{d\times d}$ be positive semi-definite matrices with $\Tr(A_t) \leq K$ for all $t \in [T]$. Fix $\Gamma_0 = \lambda I_d$ and $\Gamma_t = \lambda I_d + \sum_{i=1}^t A_i$ for $t \in [T]$. Then
\[\sum_{t=1}^T \Tr(\Gamma_{t-1}^{-1} A_t) \leq \frac{dK\log \frac{(T+1)K}{\lambda}}{\lambda\log(1+K/\lambda)}.\]
\end{lemma}

\begin{proof}[\pfref{lemma:elliptic}]
Fix $t \in [T]$. Since $\Tr(A_t) \leq 1$, there is some $p_t \in \Delta(\RR^d)$ such that $A_t = \EE_{a \sim p_t}\brk{ aa^\top}$ and $\Pr[\norm{a}_2\leq 1] = 1$. Now observe that
\begin{align}
\log \det(\Gamma_t)
&= \log \det(\Gamma_{t-1} + A_t) \\ 
&= \log \det(\Gamma_{t-1}) + \log \det(I_d + \Gamma_{t-1}^{-1/2} A_t \Gamma_{t-1}^{-1/2}) \\ 
&= \log \det(\Gamma_{t-1}) + \log \det\left(\EE_{a \sim p_t}\left[I_d + \Gamma_{t-1}^{-1/2} aa^\top \Gamma_{t-1}^{-1/2}\right]\right) \\ 
&\geq \log \det(\Gamma_{t-1}) + \EE_{a \sim p_t} \log \det(I_d + \Gamma_{t-1}^{-1/2}aa^\top \Gamma_{t-1}^{-1/2}) \\ 
&= \log\det(\Gamma_{t-1}) + \EE_{a\sim p_t}\log(1 + a^\top \Gamma_{t-1}^{-1} a).
\end{align}
Now $a^\top \Gamma_{t-1}^{-1} a \leq 1/\lambda$ with probability $1$, where $\lambda = \lambda_{\min}(\Gamma_0)$. We know that $\lambda x\log(1+1/\lambda) \leq \log(1+x)$ for all $x \in [0,1/\lambda]$. Thus,
\[\log\det(\Gamma_t) \geq \log\det(\Gamma_{t-1}) + \lambda \log(1+1/\lambda) \EE_{a\sim p_t} a^\top \Gamma_{t-1}^{-1} a.\]
Summing over $t\in [T]$, we get
\[\log \det(\Gamma_T) \geq \log \det(\Gamma_0) + \lambda\log(1+1/\lambda) \sum_{t=1}^T \Tr(\Gamma_{t-1}^{-1} A_t).\]
Finally note that $\lambda_{\max}(\Gamma_T) \leq T+1$ so $\log \det(\Gamma_T) \leq d\log T$, whereas $\log\det(\Gamma_0) \geq d\log \lambda$. Thus,
\[\sum_{t=1}^T \Tr(\Gamma_{t-1}^{-1} A_t) \leq \frac{d\log \frac{T+1}{\lambda}}{\lambda\log(1+1/\lambda)}\]
as claimed.
\end{proof}

\begin{lemma}[Freedman's inequality, e.g. \cite{agarwal2014taming}]\label{lemma:freedman}
Let $(Z_t)_{t=1}^T$ be a martingale difference sequence adapted to filtration $(\MF_t)_{t=0}^{T-1}$. Suppose that $|Z_t| \leq R$ holds almost surely for all $t$. For any $\delta \in (0,1)$ and $\eta \in (0, 1/R)$, it holds with probability at least $1-\delta$ that
\[\sum_{t=1}^T Z_t \leq \eta \sum_{t=1}^T \EE[Z_t^2|\MF_{t-1}] + \frac{\log(1/\delta)}{\eta}.\]
\end{lemma}

\begin{corollary}\label{cor:filtration-tail}
Let $(Z_t)_{t=1}^T$ be a sequence of random variables adapted to filtration $(\MF_t)_{t=0}^{T-1}$. Suppose that $Z_t \in [0,R]$ holds almost surely for all $t$. For any $\delta \in (0,1)$, it holds with probability at least $1-\delta$ that 
\[\sum_{t=1}^T \EE[Z_t|\MF_{t-1}] \leq 2\sum_{t=1}^T Z_t + 4R\log(1/\delta).\]
\end{corollary}

\begin{proof}[\pfref{cor:filtration-tail}]
Observe that for any $t \in [T]$,
\begin{align}
\EE[(Z_t - \EE[Z_t\mid \MF_{t-1}])^2\mid \MF_{t-1}]
&\leq \EE[Z_t^2\mid \MF_{t-1}] \\ 
&\leq R \cdot \EE[Z_t\mid \MF_{t-1}].
\end{align}
Applying \cref{lemma:freedman} to the sequence $(\EE[Z_t\mid \MF_{t-1}] - Z_t)_{t=1}^T$, which is a martingale difference sequence with elements supported almost surely on $[-R,R]$, we get for any $\eta \in (0,1/R)$ that with probability at least $1-\delta$,
\begin{align}
\sum_{t=1}^T (\EE[Z_t\mid \MF_{t-1}] - Z_t)
&\leq \eta \sum_{t=1}^T \EE[(Z_t - \EE[Z_t\mid \MF_{t-1}])^2\mid \MF_{t-1}] + \frac{\log(1/\delta)}{\eta} \\ 
&\leq \eta R \sum_{t=1}^T \EE[Z_t\mid \MF_{t-1}] + \frac{\log(1/\delta)}{\eta}.
\end{align}
Set $\eta = 1/(2R)$. Simplifying gives
\[\sum_{t=1}^T \EE[Z_t\mid \MF_{t-1}] \leq 2\sum_{t=1}^T Z_t + 4R\log(1/\delta).\]
as claimed.
\end{proof}

\section{Proofs from \creftitle{sec:autoreg_sharpening}}
\label{sec:proofs_sample}

\begin{proof}[\pfref{prop:greedy}]
  We prove the result by induction. Fix $x\in\cX$, and let
  $\ystar_1,\ldots,\ystar_H\ldef{}\ystar(x)$. Fix $h\in\brk{H}$, and
  assume by induction that $\yhat_{h'}=\ystar_{h'}$ for all $h'<h$. We
  claim that in this case,
  \begin{align}
    \pi_h(\ystar_h\mid{}\yhat_{1},\ldots,\yhat_{h-1},x)
    = \pi_h(\ystar_h\mid{}\ystar_{1},\ldots,\ystar_{h-1},x)>1/2, 
  \end{align}
  which implies that $\yhat_h=\ystar_h$. To see this, we observe that
  by Bayes' rule,
  \begin{align}
    \label{eq:4}
    \pi(\ystar_1,\ldots,\ystar_H\mid{}x)
    &\leq{}     \pi(\ystar_1,\ldots,\ystar_h\mid{}x)\\
    &=\prod_{h'=1}^{h}\pi_{h'}(\ystar_{h'}\mid{}\ystar_1,\ldots,\ystar_{h'-1},x)
    \leq{}\pi_h(\ystar_h\mid{}\ystar_{1},\ldots,\ystar_{h-1},x).
  \end{align}
  If we were to have
  $\pi_h(\ystar_h\mid{}\yhat_{1},\ldots,\yhat_{h-1},x)=\pi_h(\ystar_h\mid{}\ystar_{1},\ldots,\ystar_{h-1},x)\leq{}1/2$,
  it would contradict the assumption that
  $\pi(\ystar_1,\ldots,\ystar_H\mid{}x)>1/2$. This proves the result.
  \end{proof}

\section{Proofs from \creftitle{sec:lower}}
\label{sec:proofs_lower}

Below, we state and prove a generalization of \cref{thm:lower,thm:lower_adaptive} which allows for approximate maximizers in the sense of \cref{def:sharpening_general}, as well as a more general coverage coefficient.

To state the result, for a model $\pi$, we define
\begin{align}
\Ypigamma(x) = \crl*{y\mid{}
  \pi(y\mid{}x)\geq{}(1-\gamma)\cdot\max_{y\in\cY}\pi(y\mid
                                                         x)}.
\end{align}
Next, for any integer $p\in\bbN$, we define 
\begin{align}\Cstargp(\pi) = \prn*{\En\brk*{
    \frac{1}{(\pi(\Ypigamma(x)\mid{}x))^p
  }}}^{1/p},
  \end{align} 
  with the convention that $\Cstargp = \Cstargp(\piref)$. \msedit{Our
    most general lower bound, \cref{thm:lower_general}, holds in the
    regime where $\gamma=1/2$, and thus the best response$y$ has bounded margin away from suboptimal responses. }
\begin{thmmod}{thm:lower}{$'$}[Lower bound for sharpening]
  \label{thm:lower_general}
  Fix integers $d \ge 1$ and $p \ge 1$ and parameters   $\epsilon \in
  (0,1)$ and $C \ge 1$, and set $\gamma = 1/2$.  There exists a class
  of models $\Pi$ such that i) $\log |\Pi| \asymp d (1+\log(C
  \epsilon^{-1/p}))$, ii) $\sup_{\pi \in \Pi} \Cstargp(\pi) \lesssim
  C$, and iii) $\Ypigamma(x)$ is a singleton for all $\pi\in\Pi$, for
  which any sharpening algorithm $\pihat$ that attains $\En\brk*{\Pr_{x \sim
    \cdist}[\pihat(\Ypigamma[\piref](x)) > 1/2]}  \ge 1 - \epsilon$ for
  all $\piref \in \Pi$ must collect a total number of samples $m = n\cdot{}N$ at least
\begin{align}
  m \gtrsim \begin{cases}
    \frac{ C\log |\Pi|}{\epsilon^{1+1/p} (1+\log(C \epsilon^{-1/p}))}
    & \text{sample-and-evaluate oracle},\\
    \frac{ C\log |\Pi|}{\epsilon^{1/p} (1+\log(C \epsilon^{-1/p}))}  &
    \text{adaptive sample-and-evaluate oracle}.
    \end{cases}
\end{align}
\end{thmmod}
\paragraph{Proof of \Cref{thm:lower_general}}%
  \newcommand{\Enki}{\En^{\cI}}%
\newcommand{\Pnki}{\Pr^{\cI}}%
\newcommand{\EnI}{\En_{\cI \sim \texttt{Unif}}}%
Let parameters $d, p \in \bbN$ and $\epsilon > 0$ be given, and set
$\gamma=1/2$. Let $M\in\bbN$ and $\Delta>0$ be parameters to be chosen
later. Let $\cX = \{x_0,x_1,\dots,x_d\}$ and $\cY =
\{y_0,y_1,\dots,y_M\}$ be arbitrary discrete sets (with
$\abs{\cX}=d+1$ and $\abs{\cY}=M+1$).

\paragraph{Construction of prompt distribution and model class}
We use the same construction for the non-adaptive and adaptive lower
bounds in the theorem statement. We define the prompt distribution $\cdist$ via
\begin{align}
    \cdist:= (1 - \Delta) \dirac_{x_0} +  \frac{\Delta}{d}
    \sum_{i=1}^d \dirac_{x_i},
\end{align}
where $\dirac_x$ denotes the Dirac delta distribution on element $x$.

As the first step toward constructing the model class $\Pi$, we
introduce a family of distributions $(P_0,P_1,\dots,P_M)$  on $\cY$ as follows
\begin{align}
    P_0 = \dirac_{y_0}, \quad \forall i \ge 1,~ P_{i} = \frac{1}{(1-\gamma)M}\dirac_{y_i} + \sum_{j \in [M] \setminus \{i\}} \frac{1}{M}\prn*{1 - \frac{\gamma}{(M-1)(1-\gamma)}} \dirac_{y_j}.
\end{align}
Next, for or any index $\cI = (j_1,j_2,\dots,j_d)\in
[M]^d$, define a model
\begin{align}
    \pi^{\cI}(x_i) =  \begin{cases} P_0 & i = 0\\
    P_{j_i} & i > 0
    \end{cases}.
\end{align}
We define the model class as
\begin{align}
    \Pi := \{\pi^{\cI}: \cI \in [M]^d\},
\end{align}
which we note has
\begin{align}
    \log |\Pi| &= d \log M. 
\end{align}

\paragraph{Preliminary technical results}
Define 
\begin{align}
    \YIgamma(x) := \{y: \pi^{\cI}(y \mid x) \ge (1-\gamma) \max_{y\in\cY}\abedit{\pi^{\cI}}(y\mid{}x)\}.
\end{align}
The following property is immediate.
\begin{lemma}\label{claim:Ygamma_lb} Let  $\cI = (j_1,\dots,j_d) \in
  [d]^M$. Then $\YIgamma(x_i) = \{y_{j_i}\}$ if $i > 0$, and $\YIgamma(x_0)=\crl{y_0}$.
\end{lemma}
In view of this result, we define $y^{\cI}(x) = \argmax_{y}
\pi^{\cI}(y\mid{}x)$ as the unique arg-max response for $x$.

Going forward, let us fix the algorithm under consideration. Let $\Pnki\brk{\cdot}$ denote the law over the dataset used by the
algorithm when the true instance is $\pi^{\cI}$ (including possible
randomness and adaptivity from the algorithm itself), and let
$\Enki\brk{\cdot}$ denote the corresponding expectation. The following lemma is a basic
technical result.
\begin{lemma}[Reduction to
  classification]\label{lem:redux_to_classification} Let $\pihat$ be
  the model produced by an algorithm with access to a
  (adaptive) sample-and-evaluate oracle for $\pi^{\cI}$. Suppose that for some $\eps\geq{}0$,
\begin{align}
    \EnI\Enki\Pr_{x \sim \cdist}[\pihat(\YIgamma(x)\mid{}x) > 1/2] \ge 1 - \epsilon. \label{eq:EnIlb}
\end{align}
Define $\wh \cI = (\wh j_1,\dots, \wh j_d)$ via $\wh{j}_i = \argmax_j\pihat(y_j \mid x_i)$, and write $\cI = (j_1^\star,\dots,j_d^\star)$.  Then, 
\begin{align}
    \frac{1}{d}\sum_{i=1}^d\EnI\Enki\left[ \indic\{\wh{j}_i \ne j_i^\star\}\right] \le \epsilon/\Delta.
\end{align}
\end{lemma}

\begin{proof}[\pfref{lem:redux_to_classification}] As established in
  \pref{claim:Ygamma_lb}, under instance $\cI$, $\YIgamma(x_i) =
  \{y_{j_i^\star}\}$ for any $i \in [d]$. Thus, whenever
  $\pihat(\YIgamma(x_i)) > 1/2$, $j_i^\star = \argmax_j \pihat(y_j
  \mid x_i) =: \wh{j}_i$. The result follows by noting that the event
  $\{\exists i \in [d]: x = x_i\}$ occurs with probability at least
  $\Delta$ under $x\sim\cdist$.
\end{proof}

\newcommand{\Priunif}{\Pr_{i,\mathrm{unif}}}

\paragraph{Lower bound under sample-and-evaluate oracle} 
Recall that in the non-adaptive framework, the sample complexity $m$
is fixed. In light of \cref{lem:redux_to_classification}, it suffices to establishes the following claim.
\begin{lemma}
  \label{lem:nonadaptive_lower}
  There exists a universal constant $c>0$ such that for all $M \ge 8$, if $m \le cdM/\Delta$, then 
    $ \EnI\Enki\left[ \indic\{\wh{j}_i \ne j_i^\star\}\right] \ge 1/8$
    for all $i$.
\end{lemma}
With this, the result follows by selecting $\Delta = 16\eps$, with
which \Cref{lem:redux_to_classification} implies that any algorithm
with $\EnI\Enki\Pr_{x \sim
  \cdist}[\pihat(\YIgamma(x)\mid{}x) > 1/2] \geq{} 1 - \epsilon$ must
have $m \approxgeq dM/\Delta$. To
conclude, we choose $M \asymp 1+ C \epsilon^{-1/p}$, which gives
$m\asymp dM/\Delta \asymp dC\epsilon^{-(1+1/p)} \asymp \epsilon^{-(1+1/p)}
\log \Pi/\log(1+ C \epsilon^{1/p})$. Finally, we check that with this
choice, all $\pi \in \Pi$ satisfy
\begin{align}
  \Cstargp(\pi) 
    &= \left(\Pr_{x \sim \cdist}\brk*{x=x_0} + \prn*{M(1-\gamma)}^p\Pr_{x \sim \cdist}\brk*{x \neq x_0}\right)^{1/p}\\
    &= \left((1-\Delta) + \prn*{M(1-\gamma)}^p \Delta\right)^{1/p} \\
    &\approxleq \left((1-\Delta) + \prn*{8C(1-\gamma)}^p \right)^{1/p} \lesssim C.
\end{align}

\begin{proof}[\pfref{lem:nonadaptive_lower}]
  Let $i\in\brk{d}$ be fixed. Of the $m=n\cdot{}N$ tuples $(x,y,\log\piref(y\mid{}x))$ that are
  observed by the algorithm,
  let $m_i$ denote the (random) number of such examples
  for which $x =x_i$.  From Markov's inequality, we have
  \begin{align}
    \label{eq:lower_markov}
    \Pr[ m_i \le 2 \Delta m/d] \ge \frac{1}{2}
  \end{align}
  Going forward, let $\cD = \crl*{(x,y,\log\piref(y\mid{}x))}$ denote
  the dataset collected by the algorithm, which has $\abs{\cD}=m$. Let $\cE_i$ denote the
  event that, for prompt $x = x_i$, (i) there are at least two
  distinct responses $y_j$ for which $(x_i,y_j) \notin \cD$; and (ii)
  there are no pairs $(x_i,y)\in\cD$ for which
  $\piref(y\mid{}x_i)>\frac{1}{M}$. Since $\cE_i$ is a measurable
  function of $\cD$, we can write
  \begin{align}
    \EnI\Enki\left[ \indic\{\wh{j}_i \ne j_i^\star\}\right]
    &\ge   \EnI\Enki\left[ \indic\{\wh{j}_i \ne j_i^\star\} \cdot \indic\left\{ \cE_i\right\}\right]\\
    &=   \EnI\Enki \left[\indic\{\cE_i\}\En_{\cI \sim \Pr[\cI=\cdot
      \mid \cD]}\left[ \indic\{\wh{j}_i \ne j^{\star}_i\}
      \right]\right],
      \label{eq:lower_event1_passive}
  \end{align}
  where $\cI \sim \Pr[\cI=\cdot \mid \cD]$ is sampled from the
  posterior distribution over $\cI$ conditioned on the dataset
  $\cD$. Observe that conditioned on $\cE_i$, the posterior
  distribution over $j_i^{\star}$ under $\cI \sim \Pr[\cI=\cdot \mid
  \cD]$ is uniform over the set of indices $j\in\brk{M}$ for which
  $(x_i, y_j)\notin\cD$, and this set has size at least $2$. Hence, $\indic\{\cE_i\}\En_{\cI \sim \Pr[\cI=\cdot
      \mid \cD]}\left[ \indic\{\wh{j}_i \ne
      j^{\star}_i\}\right]\geq\frac{1}{2}$, and
  resuming from \cref{eq:lower_event1_passive}, we have
  \begin{align}
    \EnI\Enki\left[ \indic\{\wh{j}_i \ne j_i^\star\}\right] \ge \frac{1}{2}\EnI\Enki \left[\indic\{\cE_i\}\right] &\ge  \frac{1}{2}\EnI\Pnki \left[\cE_i \cap \crl{m_i \le 2\Delta m/d}\right] \\
                                                                                                                  &\ge \frac{1}{4} \EnI\Pnki \left[\cE_i \mid m_i \le 2\Delta m/d\right],
  \end{align}
  where the last inequality is from \cref{eq:lower_markov}.
  Finally, we can check that under the law $\Pnki$, the probability of
  the event $\cE_i$---conditioned on the value $m_i$---is at least the probability that
  $(x_i,y_{j^{\star}_i}),(x_i,y_{j'}) \notin \cD$ for an arbitrary
  fixed index $j' \ne j^{\star}_i$,
  which on the event $\{m_i \le 2\Delta{}m/d\}$ is at least
  \begin{align}
    \left(1 - \frac{3}{M}\right)^{m_i} \ge \left(1 - \frac{3}{M}\right)^{2\Delta m/d},
  \end{align}
  where we have used that $\gamma=1/2$.
  The value above is at least $\frac{1}{4}$ whenever $m \le c\cdot{}
  dM/\Delta$ for a sufficiently small
  absolute constant $c> 0$. For this value of $m$, we conclude that
  $\EnI\Enki\left[ \indic\{\wh{j}_i \ne j_i^\star\}\right] \ge
  \frac{1}{4} \EnI\Pnki \left[\cE_i \mid \{m_i \le 2\Delta
    m/d\}\right] \ge \frac{1}{8}$.
\end{proof}

\paragraph{Lower bound under adaptive sample-and-evaluate oracle}
\newcommand{\Sgood}{S_{\texttt{good}}}

In the adaptive framework, we let $m_i$ denote the (potentially
random) number of tuples $(x,y,\log\piref(y\mid{}x))$ observed by the algorithm in which $x
= x_{i}$.  Note that unlike the non-adaptive framework, the
distribution over $m_i$ depends on the
underlying instance $\cI$ with which the algorithm interacts.

To begin, from \Cref{lem:redux_to_classification} and Markov's
inequality, if $\pihat$ satisfies the guarantee $
\EnI\Enki\Pr_{x \sim \cdist}[\pihat(\YIgamma(x)) > 1/2] \ge 1 -
\epsilon$, then there exists a set of indices $\Sgood\subset\brk{d}$
such that\footnote{We emphasize that the set $\Sgood$ is not a random
  variable, and depends only on
  the algorithm itself.}\loose
\begin{align}
  \label{eq:sgood}
    |\Sgood| \ge \floor{d/2}, \quad \forall i \in \Sgood, ~ \EnI\Enki\left[ \indic\{\wh{j}_i \ne j_i^\star\}\right] \le \frac{2\epsilon}{\Delta}.
\end{align}
We now appeal to the following lemma.
\begin{lemma}
  \label{lem:adaptive_lower}
  As long as $M\geq{}6$, it holds that for all $i\in\brk{d}$,
  \begin{align}
    \EnI\Enki\left[ \indic\{\wh{j}_i \ne j_i^\star\}\right]
    \geq{} \frac{1}{4e}    \EnI\Enki\left[ \indic\{m_i\leq{}M/3\}\right].
  \end{align}
\end{lemma}
Combining \cref{lem:adaptive_lower} with \cref{eq:sgood}, it follows
that there exist absolute constant $c_1,c_2,c_3>0$ such that if
  $\Delta=c_1\cdot{}\eps$, then for all $i\in\Sgood$, 
  \begin{align}
    \EnI\bbP^{\cI}\brk*{m_i\geq{}c_2M} \geq{}c_3.
  \end{align}

Thus, with this choice for $\Delta$, we have that $i \in \Sgood$,
\begin{align}
  \EnI\Enki\left[ m_i\right] \gtrsim M,
\end{align}
and we can lower bound the algorithm's expected sample complexity by
summing over $i\in\Sgood$:
\begin{align}
    \EnI\Enki\left[ m\right] \ge   \EnI\Enki\left[ \sum_{i\in \Sgood} m_i\right] \gtrsim |\Sgood|M \gtrsim dM.
\end{align}
The result now follows by tuning $M \asymp 1+ C \epsilon^{-1/p}$ as in
the proof of the lower bound for non-adaptive sampling, which gives
$\En\brk{m}\approxgeq{}dM \asymp dC\epsilon^{-1/p} \asymp \epsilon^{-1/p}
\log \Pi/\log(1+ C \epsilon^{1/p})$ and $\Cstargp(\pi)\approxleq{}C$
for all $\pi\in\Pi$.\loose

\begin{proof}[\pfref{lem:adaptive_lower}]
    Let $i\in\brk{d}$ be fixed. Let $\cD = \crl*{(x,y,\log\piref(y\mid{}x))}$ denote
  the dataset collected by the algorithm at termination, which has $\abs{\cD}=m$. Let $\cE_i$ denote the
  event that, for prompt $x = x_i$, (i) there are at least two
  distinct responses $y_j$ for which $(x_i,y_j) \notin \cD$; and (ii)
  there are no pairs $(x_i,y)\in\cD$ for which
  $\piref(y\mid{}x_i)>\frac{1}{M}$. Since $\cE_i$ is a measurable
  function of $\cD$, we can write
  \begin{align}
    \EnI\Enki\left[ \indic\{\wh{j}_i \ne j_i^\star\}\right]
    &\ge   \EnI\Enki\left[ \indic\{\wh{j}_i \ne j_i^\star\} \cdot \indic\left\{ \cE_i\right\}\right]\\
    &=   \EnI\Enki \left[\indic\{\cE_i\}\En_{\cI \sim \Pr[\cI=\cdot
      \mid \cD]}\left[ \indic\{\wh{j}_i \ne j^{\star}_i\}
      \right]\right],
      \label{eq:lower_event1}
  \end{align}
  where $\cI \sim \Pr[\cI=\cdot \mid \cD]$ is sampled from the
  posterior distribution over $\cI$ conditioned on the dataset
  $\cD$. Observe that conditioned on $\cE_i$, the posterior
  distribution over $j_i^{\star}$ under $\cI \sim \Pr[\cI=\cdot \mid
  \cD]$ is uniform over the set of indices $j\in\brk{M}$ for which
  $(x_i, y_j)\notin\cD$, and this set has size at least $2$. Hence, $\indic\{\cE_i\}\En_{\cI \sim \Pr[\cI=\cdot
      \mid \cD]}\left[ \indic\{\wh{j}_i \ne
      j^{\star}_i\}\right]\geq\frac{1}{2}$, and
  resuming from \cref{eq:lower_event1}, we have
  \begin{align}
    \EnI\Enki\left[ \indic\{\wh{j}_i \ne j_i^\star\}\right] &\ge
                                                              \frac{1}{2}\EnI\Enki
                                                              \left[\indic\{\cE_i\}\right]
    \\
                                                            &\ge  \frac{1}{2}\EnI\Pnki \left[\cE_i \cap \crl{m_i \le M/3}\right] \\
                                                                                                                  &= \frac{1}{2} \EnI\brk*{\Pnki \left[\cE_i \mid m_i \le M/3\right]\cdot\bbP^{\cI}\brk*{m_i\leq{}M/3}}.
  \end{align}
  \dfc{It might be good to justify the claim below slightly more
    formally in the adaptive case}\mscomment{see below:}
    
   The event $\cE_i$ is a superset of the event $\cE_{i,j'}$ that
  $(x_i,y_{j^{\star}_i}),(x_i,y_{j'}) \notin \cD$ for an arbitrary
  fixed index $j' \ne j^{\star}_i$. 
  Thus,
  \begin{align}
      \Pnki \left[\cE_i \mid m_i \le M/3\right] \ge \Pnki \left[ \cE_{i,j'} \mid m_i \le M/3\right] 
  \end{align}
  Moreover, we can realize the law of $\Pnki$ considering an infinite tape, associated to index $i$, of i.i.d. samples $y \sim \pibase(\cdot \mid x_i)$, and taking the first $m_i$ elements on this tape to be the samples 
  $(x,y,\log \pibase(y \mid x)) \in \cD$ with $x = x_i$  (see, e.g. \citet{simchowitz2017simulator} for an argument of this form).  On the event $\{m_i \le M/3\}$, the $m_i$ samples in $(x,y,\log \pibase(y \mid x)) \in \cD$ with $x = x_i$ are a subset of the first $M/3$ samples from the index-$i$ tape. Viewed in this way, we can lower bound the probability of $\cE_{i,j}$ by the probability of the event $\tilde \cE_{i,j'}$ that the first $M/3$ $y$'s on the index-$i$ tape contain neither $j^\star_i$, nor the designated index $j'$. As these first $M/3$ $y$'s are not chosen adaptively, the probability of $\tilde \cE_{i,j'}$ is at least
  \begin{align}
    \left(1 - \frac{3}{M}\right)^{m_i} \ge \left(1 - \frac{3}{M}\right)^{M/3}\geq\frac{1}{2e},
  \end{align}
  as long as $M\geq{}6$ and $\gamma=1/2$. We conclude
  that
  \begin{align}
    \EnI\Enki\left[ \indic\{\wh{j}_i \ne j_i^\star\}\right]
    \geq{} \frac{1}{4e}    \EnI\Enki\left[ \indic\{m_i\leq{}M/3\}\right].
  \end{align}  
\end{proof}

\section{Proofs from \creftitle{sec:bestofn_theory} and \creftitle{sec:adaptive}}
\label{sec:proofs_bestofn}

The following theorem is a generalization of \cref{thm:bestofn} which allows for approximate maximizers in the sense of \cref{def:sharpening_general}.
  \begin{thmmod}{thm:bestofn}{$'$}
  \label{thm:bestofn_general}
  Let $\delfail,\delta\in(0,1)$ be given, and suppose we set
  $N=\Nstar\log(2\delta^{-1})$ for a parameter $\Nstar\in\bbN$. Then
  for any $n\in\bbN$, \bestofnalg ensures that with probability at least $1-\delfail$, for any $\gamma\in(0,1)$, the output model $\pihat$ satisfies
  \begin{align}
    \bbP_{x\sim\cdist}\brk*{\pihat(\Ygamma(x)\mid{}x)\leq{}1-2\delta}
    \approxleq{} \frac{1}{\delta}\cdot{}\frac{\log(\abs{\Pi}\delfail^{-1})}{n} + \frac{\Cstarg}{\Nstar}.
  \end{align}
In particular, given $(\eps,\delta,\gamma)$, by setting
$n=C_{\ref{thm:bestofn}}\frac{\log\abs{\Pi}}{\delta\eps}$ and
$\Nstar=C_{\ref{thm:bestofn}}\frac{\Cstarg}{\eps}$ for a sufficiently
large absolute constant $C_{\ref{thm:bestofn}}>0$,
we are guaranteed that
\begin{align}
  \bbP_{x\sim\cdist}\brk*{\pihat(\Ygamma(x)\mid{}x)\leq{}1-\delta}
  \leq \eps. \label{eq:top_k_ub}
\end{align}
The total sample complexity is
\begin{align}
  m = O\left(\frac{\Cstarg\log(\abs{\Pi}\delfail^{-1})\log(\delta^{-1})}{\delta\eps^2}\right).
\end{align}
\end{thmmod}

\begin{proof}[\pfref{thm:bestofn_general}]%
  \newcommand{\cXgood}{\cX_{\texttt{good}}}%
  \newcommand{\pik}{\pi_k}%
  \newcommand{\kstar}{k_{\star}}%
Under realizability of $\pin$ (\cref{assumption:bon-realizability}),
\cref{lem:mle} implies that the output of \bestofnalg satisfies, with probability at least $1-\delfail$,
\begin{align}
  \label{eq:sft_stat}
  \En_{x\sim\cdist}\brk*{\Dhels{\pihat(\cdot\mid{}x)}{\pin(\cdot\mid{}x)}}\leq\vepsstat^2\ldef{}\frac{2\log(\abs{\Pi}/\delfail)}{n}.
\end{align}  
Henceforth we condition on the event that \cref{eq:sft_stat} holds. Let
\[
  \cXgood\ldef{}\left\{
    x\in\cX\mid\Nstar\geq\frac{1}{\piref(\Ygamma(x)\mid{}x)}
    \right\}
  \]
  denote the set of prompts for which $\piref$ places sufficiently
  high mass on $\Ygamma(x)$.
  We can bound
  \begin{align}
    &\bbP_{x\sim\cdist}\brk*{\pihat(\Ygamma(x)\mid{}x)\leq{}1-\delta}
    \\ 
    &\qquad\leq{}
    \bbP_{x\sim\cdist}\brk*{\pihat(\Ygamma(x)\mid{}x)\leq{}1-\delta,
    x\in\cXgood}
    +     \bbP_{x\sim\cdist}\brk*{x\notin\cXgood}.\label{eq:pihat-error-decomp}
  \end{align}
  To bound the first term in \cref{eq:pihat-error-decomp}, note that if $x\in\cXgood$, then
  $\pin(\Ygamma(x)\mid{}x)\geq{}1-\delta/2$. Indeed, observe that
  $y\sim\pin(\cdot\mid{}x)\notin\Ygamma(x)$ if and only if
  $y_1,\ldots,y_N\sim\piref(x)$ have $y_i\notin\Ygamma(x)$ for all
  $i$, which happens with probability
  $(1-\piref(\Ygamma(x)\mid{}x))^N \leq (1-1/\Nstar)^N \leq \delta/2$ since $x \in \cXgood$. It follows that for any such $x$, we can lower bound (using the data processing inequality)
  \begin{align}
    \Dhels{\pihat(\cdot\mid{}x)}{\pin(\cdot\mid{}x)}
    &\geq \left(\sqrt{1-\pihat(\Ygamma(x)\mid{}x)} - \sqrt{1-\pin(\Ygamma(x)\mid{}x)}\right)^2 \\
    &\approxgeq{} \delta\cdot\indic\crl*{\pihat(\Ygamma(x)\mid{}x)\leq{}1-\delta}. \label{eq:topk_key_ineq}
  \end{align}
  By \cref{eq:sft_stat,eq:topk_key_ineq}, it follows that
  \[
    \bbP_{x\sim\cdist}\brk*{\pihat(\Ygamma(x)\mid{}x)\leq{}1-\delta,
      x\in\cXgood}\approxleq \frac{\vepsstat^2}{\delta}.
  \]

  For the second term in \cref{eq:pihat-error-decomp}, we bound
  \begin{align}
    \bbP_{x\sim\cdist}\brk*{x\notin\cXgood}
    &=
    \bbP_{x\sim\cdist}\brk*{\Nstar < \frac{1}{\piref(\Ygamma(x)\mid{}x)}} \\
    &=
    \bbP_{x\sim\cdist}\brk*{\frac{1}{\Nstar\piref(\Ygamma(x)\mid{}x)}> 1} \\
    &\leq{}
    \frac{1}{\Nstar}\En_{x\sim\cdist}\brk*{\frac{1}{\piref(\Ygamma(x)\mid{}x)}} \\
    &\leq{}  \frac{\Cstarg}{\Nstar}
  \end{align}
  via Markov's inequality and the definition of $\Cstarg$. Substituting both bounds into \cref{eq:pihat-error-decomp} completes the proof.
\end{proof}

\begin{proof}[\pfref{thm:bestofn_adaptive}]
  The proof begins similarly to \cref{thm:bestofn}. By realizability of $\pi_{\taumu}$,
\cref{lem:mle} implies that the output of \bestofnalg satisfies, with probability at least $1-\delfail$,
\begin{align}
\En_{x\sim\cdist}\brk*{\Dhels{\pihat(\cdot\mid{}x)}{\pi_{\taumu}(\cdot\mid{}x)}}\leq\vepsstat^2\ldef{}\frac{2\log(\abs{\Pi}/\delfail)}{n}.
\end{align}  
Condition on the event that this guarantee holds. We invoke
the following lemma, proven in the sequel.
\begin{lemma}\label{lem:fdr}  Let $P$ be a distribution on a discrete
  space $\cY$. Let
  $\Ystar = \argmax_{y\in\cY}P(y)$
  and let $\Pstar := \max_{y\in\cY}P(y)$. Let
  $y_1,y_2,\ldots\sim{}P$, and for any stopping time $\tau$, define
  \begin{align}
  \hatytau \in \argmax \left\{ P(y) : y \in
    \left\{y_1,\dots,y_\tau\right\}\right\}.
  \end{align}
  Next, for a parameter $\mu>0$, define the stopping time
\begin{align}
  \taumu := \inf \left\{k : \frac{1}{\max_{1 \le i \le k} P(y_i)} \le k/\mu \right\}.
\end{align}
Then
\begin{align}
  \En[\taumu] \le \frac{\mu+(1/|\Ystar|)}{\Pstar}.
\end{align}
In addition, for any stopping time $\tau \ge \taumu$ (including $\tau =
\taumu$ itself), we have $\Pr[\hatytau \notin \Ystar] \le e^{-\abs{\Ystar}\mu}$.
\end{lemma}
This lemma, with our choice of $\mu$, ensures that \emph{for all
  $x\in\cX$}, 
\begin{align}
  \pi_{\taumu}(\Ystar(x)\mid x) \ge 1 - e^{-\mu} = 1 - \delta/2.
\end{align}
Following the reasoning in \cref{eq:topk_key_ineq}, this implies that
\begin{align}
    \Dhels{\pihat(\cdot\mid{}x)}{\pi_{\taumu}(\cdot\mid{}x)}
    \approxgeq{} \delta\cdot\indic\crl*{\pihat(\Ystar(x)\mid{}x)\leq{}1-\delta}, 
\end{align}
so that
\begin{align}
      \bbP_{x\sim\cdist}\brk*{\pihat(\Ystar(x)\mid{}x)\leq{}1-\delta}\approxleq \frac{\vepsstat^2}{\delta}
\end{align}
as desired.

To bound the expected sample complexity, we observe that
\begin{align}
  \En[m] = n \cdot \En[\taumu(x)] \overset{(i)}{\le} \En\left[\frac{1 + \mu}{\piref(\Ystar(x)\mid x)}\right] = (1+\mu)\Cstarb,
\end{align}
where inequality $(i)$ invokes \cref{lem:fdr} once more.
\end{proof}
\dfc{Remark: If we know $\abs{\Ystar}$, we can set
  $\mu=\log(1/\delta)/\abs{\Ystar}$ to get
  \[
    \En\brk{\taumu} \leq \frac{1}{\abs{\Ystar}\Pstar} = \frac{1}{P(\Ystar)},
  \]
  which improves the result. Something similar probably works for
  $\Ygamma$. Without this form of prior knowledge, the current result
  is most likely tight.
  }

\begin{proof}[\pfref{lem:fdr}] Define $\Nstar := \mu/\Pstar $.
To bound the tails of $\taumu$, define
\begin{align}
\tau = \inf \{k\mid{}k \ge \Nstar \text{ and } \Ystar \cap
  \{y_1,\dots,y_k\} \ne \emptyset\}.
  \end{align}
  It follows from the definition that $\taumu \le \tau$, since
  for any $k\geq{}\Nstar$, if there exists $i\leq{}k$ such that
  $y_i\in\Ystar$, then
  \begin{align}
    \frac{1}{P(y_i)}
    = \frac{1}{\Pstar} = \frac{\Nstar}{\mu} \leq \frac{k}{\mu}.
  \end{align}

  Thus, for $k \ge \Nstar$, we can bound
\begin{align}
    \Pr[\taumu >k] \le \Pr[\tau > k] = \Pr[\cY^{\star} \cap \{y_1,\dots,y_k\} = \emptyset]  \le (1 - |\Ystar|\Pstar)^k,
\end{align}
and consequently 
\begin{align}
    \En[\taumu] \le \En[\tau] &\le  \En[\tau \indic\{\tau \le
                               \Nstar\}] +  \En[\tau \indic\{\tau >
                               \Nstar\}]\\
                               &\le \Nstar + \sum_{k > \Nstar}  (1 -
                                 |\Ystar|\Pstar)^k \\
                             &\le \Nstar +
                                 \frac{1}{|\Ystar|P(\ystar)}
                                 = \frac{\mu+1/|\Ystar|}{P(\ystar)}.
\end{align}
To prove correctness, observe that  $\taumu \ge \Nstar$,  
because for all $y \in \cY$, $\frac{1}{P(y)} \ge \Nstar/\mu$. Hence,
any stopping time $\tau \ge \taumu$ also satisfies $\tau \ge \Nstar$,
and moreover has $\hatytau \in \Ystar$ whenever $\Ystar \cap
\{y_1,y_2,\dots,y_{\tau}\} \ne \emptyset$. This fails to occur with
probability no more than
\begin{align}
    \left(1 - \frac{|\Ystar|}{\Pstar}\right)^{\Nstar} = \left(1 - \frac{|\Ystar|}{\Pstar}\right)^{\mu/\Pstar} \le e^{-|\Ystar|\mu}.
\end{align}

\end{proof}

\section{Proofs from \creftitle{sec:rlhf_theory}}
\label{sec:proofs_rlhf}

\subsection{Proof of \creftitle{thm:dpo}}

We state and prove a generalized version of \cref{thm:dpo}. In the
assumptions below, we fix a parameter $\gamma \in [0,1)$; the setting $\gamma = 0$ corresponds to \cref{thm:dpo}.

\begin{assumption}[Coverage]
  \label{ass:conc-closs-apx}
  All $\pi\in\Pi$ satisfy
  $\cC_{\pi} \leq \Cconc$
  for a parameter $\Ccon \geq (1-\gamma)^{-1}\Cstarg$, and
  $\Cpp{\piref}{\pi} \leq \Closs$
  for a parameter $\Closs \geq \abs{\cY}$.
\end{assumption}
By \cref{lem:conc_bound_general}, \cref{ass:conc-closs-apx} is consistent
with the assumption that $\pistarb\in\Pi$.

\begin{assumption}[Margin]
  \label{ass:hard_margin-apx}
  For all $x\in\supp(\cdist)$, the initial model $\piref$ satisfies
  \[
    \piref(\Ygamma(x)\mid{}x) \geq{}(1+\gammargin)\cdot\piref(y\mid{}x)\quad\forall{}y\not\in\Ygamma(x)
  \]
  for a parameter $\gammargin>0$.
\end{assumption}

      \begin{thmmod}{thm:dpo}{$'$}
        \label{thm:dpo_general}
                Assume that $\pistarb\in\Pi$ (\cref{ass:realizability_dpo}), and that
        \cref{ass:conc-closs} and \cref{ass:hard_margin} hold with respect to some $\gamma \in [0,1)$, with parameters
        $\Cconc$, $\Closs$, and $\gammargin>0$. For any $\delta,\deltafail\in(0,1)$,
        the DPO algorithm in \cref{eq:dpo} ensures that with
        probability at least $1-\deltafail$,
        \begin{align}
          \bbP_{x\sim\cdist}\brk*{\pihat(\Ygamma(x)\mid{}x)\leq{}1-\delta}
          \approxleq{}
          \frac{1}{\gammargin\delta}\cdot{}\bigoht\prn*{\sqrt{\frac{\Cconc\log^{3}(\Closs\abs{\Pi}\deltafail^{-1})}{n}}
          +\beta\log(\Cconc) + \gamma}
        \end{align}
        where $\bigoht\prn*{\cdot}$ hides factors logarithmic in $n$
        and $\Cconc$ and doubly logarithmic in $\Pi$, $\Closs$, and
        $\deltafail^{-1}$.
      \end{thmmod}

We first state and prove some supporting technical lemmas, then
proceed to the proof of \cref{thm:dpo_general}.

\subsubsection{Technical lemmas}

The following result is a generalization of \cref{lem:conc_bound}.

\begin{lemmod}{lem:conc_bound}{$'$}
  \label{lem:conc_bound_general}
  For all $\gamma\in(0,1)$, the model $\pistarb$ satisfies $\cC_{\pistarb} \leq (1-\gamma)^{-1}\Cstarg$ and $\Cpp{\piref}{\pistarb} \leq{} \abs*{\cY}$.
\end{lemmod}
\begin{proof}[\pfref{lem:conc_bound_general}]%
  \newcommand{\ystarg}{\ystar_{\gamma}}
  For any fixed $x \in \cX$, we have 
    \begin{align}
      \En_{y \sim \pistarb(\cdot\mid x)}\brk*{\frac{\pistarb(y\mid x)}{\piref(y\mid x)}}
      &= \En_{y \sim \pistarb(\cdot\mid x)}\brk*{
      \frac{\piref^{1+\beta^{-1}}(y\mid{}x)}{\piref(y\mid{}x)}
        }\cdot{}\prn*{\sum_{y'\in\cY}\piref^{1+\beta^{-1}}(y'\mid{}x)}^{-1}\\
      &\leq{} 
        \max_{y\in\cY}\piref^{\beta^{-1}}(y\mid{}x)
        \cdot{}\prn*{\sum_{y'\in\cY}\piref^{1+\beta^{-1}}(y'\mid{}x)}^{-1}\\
            &\leq{} 
              (1-\gamma)^{-1}\piref^{\beta^{-1}}(\Ygamma(x)\mid{}x)
              \cdot{}\prn*{\sum_{y'\in\cY}\piref^{1+\beta^{-1}}(y'\mid{}x)}^{-1}\\
                  &=
              (1-\gamma)^{-1}\frac{\piref^{1+\beta^{-1}}(\Ygamma(x)\mid{}x)}{\piref(\Ygamma(x)\mid{}x)}
        \cdot{}\prn*{\sum_{y'\in\cY}\piref^{1+\beta^{-1}}(y'\mid{}x)}^{-1}\\
      &= 
        (1-\gamma)^{-1}\frac{\sum_{y\in\Ygamma(x)}\piref^{1+\beta^{-1}}(y\mid{}x)}{\piref(\Ygamma(x)\mid{}x)}
        \cdot{}\prn*{\sum_{y'\in\cY}\piref^{1+\beta^{-1}}(y'\mid{}x)}^{-1}\\
      &\leq
        (1-\gamma)^{-1}\frac{1}{\piref(\Ygamma(x)\mid{}x)}.
    \end{align}
    It follows that
    $\cC_{\pistarb} \leq (1-\gamma)^{-1}\Cstarg$
        as claimed.
    
For the second result, we have
  \begin{align}
    \Cpp{\piref}{\pistarb}
    =\En_{\piref}\brk*{\frac{1}{\piref(y\mid{}x)}\cdot\prn*{\sum_{y'\in\cY}\piref^{1+\beta^{-1}}(y'\mid{}x)}^{\beta}}
    \leq{}\En_{\piref}\brk*{\frac{1}{\piref(y\mid{}x)}}=\abs{\cY}.
  \end{align}
\end{proof}

The next lemmas provide bounds on the tails of the self-rewards used
in the algorithm.

\begin{lemma}
  \label{lem:logprob_bound}
  Suppose $\beta\in\brk{0,1}$. For any model $\pi$, with probability at least
  $1-\delta$ over the draw of $x\sim\cdist$, $y,y'\sim\piref(\cdot\mid x)$, we
  have that
  for all $s>0$, 
  \begin{align}
    \bbP\brk*{
    \abs*{\beta\log\left(\frac{\pi(y\mid{}x)}{\piref(y\mid{}x)}\right)-\beta\log\left(\frac{\pi(y'\mid{}x)}{\piref(y'\mid{}x)}\right)}
    > \log(2\Cpp{\piref}{\pi}) + s} \leq \exp\prn*{-s}.
  \end{align}
\end{lemma}
\begin{proof}[\pfref{lem:logprob_bound}]
Define 
\[X\ldef\abs*{\beta\log\left(\frac{\pi(y\mid{}x)}{\piref(y\mid{}x)}\right)-\beta\log\left(\frac{\pi(y'\mid{}x)}{\piref(y'\mid{}x)}\right)}.\] 
By the
Chernoff method, we have that with probability at least $1-\delta$,
\begin{align}
  X
  &\leq{} \log(\En\brk*{\exp(X)}) + \log(\delta^{-1})\\
  &= \log\left(\En_{x\sim\cdist,y,y'\sim\piref(x)}\brk*{\exp\prn*{\abs*{\beta\log\left(\frac{\pi(y\mid{}x)}{\piref(y\mid{}x)}\right)-\beta\log\left(\frac{\pi(y'\mid{}x)}{\piref(y'\mid{}x)}\right)}}}\right)
    + \log(\delta^{-1})\\
    &\leq{}     \log\Bigg(\En_{x\sim\cdist,y,y'\sim\piref(x)}\brk*{\exp\prn*{\beta\log\left(\frac{\pi(y\mid{}x)}{\piref(y\mid{}x)}\right)-\beta\log\left(\frac{\pi(y'\mid{}x)}{\piref(y'\mid{}x)}\right)}}
      \\&\qquad+
      \En_{x\sim\cdist,y,y'\sim\piref(x)}\brk*{\exp\prn*{\beta\log\left(\frac{\pi(y'\mid{}x)}{\piref(y'\mid{}x)}\right)-\beta\log\left(\frac{\pi(y\mid{}x)}{\piref(y\mid{}x)}\right)}}\Bigg)
      + \log(\delta^{-1})\\
      &=
      \log\left(2\En_{x\sim\cdist,y,y'\sim\piref(x)}\brk*{\exp\prn*{\beta\log\left(\frac{\pi(y\mid{}x)}{\piref(y\mid{}x)}\right)-\beta\log\left(\frac{\pi(y'\mid{}x)}{\piref(y'\mid{}x)}\right)}}\right)
        + \log(\delta^{-1})\\
        &=
          \log\prn*{\En_{x\sim\cdist,y,y'\sim\piref(x)}\brk*{\prn*{\frac{\pi(y\mid{}x)}{\piref(y\mid{}x)}\cdot\frac{\piref(y'\mid{}x)}{\pi(y'\mid{}x)}}^\beta
          }}
        + \log(2\delta^{-1}).
\end{align}
As long as $\beta\leq{}1$, by Jensen's inequality, we can bound
\begin{align}
&\En_{x\sim\cdist,y,y'\sim\piref(x)}\brk*{\prn*{\frac{\pi(y\mid{}x)}{\piref(y\mid{}x)}\cdot\frac{\piref(y'\mid{}x)}{\pi(y'\mid{}x)}}^\beta
          }\\
  &\leq{}\En_{x\sim\cdist,y'\sim\piref(x)}\brk*{\prn*{\En_{y\sim\piref(x)}\brk*{\frac{\pi(y\mid{}x)}{\piref(y\mid{}x)}}\cdot\frac{\piref(y'\mid{}x)}{\pi(y'\mid{}x)}}^\beta}\\
&=\En_{x\sim\cdist,y'\sim\piref(x)}\brk*{\prn*{\frac{\piref(y'\mid{}x)}{\pi(y'\mid{}x)}}^\beta} \\
&= \Cpp{\piref}{\pi},
\end{align}
which proves the result.
\end{proof}

\begin{lemma}
  \label{lem:fourth_moment}
  Let $\beta\in\brk{0,1}$. For all models $\pi$, we have
  \begin{align}
    \En_{x\sim\mu,y,y'\sim\piref(\cdot\mid x)}\brk*{\abs*{\beta\log\left(\frac{\pi(y\mid{}x)}{\piref(y\mid{}x)}\right)-\beta\log\left(\frac{\pi(y'\mid{}x)}{\piref(y'\mid{}x)}\right)}^4}
    \leq{} O(\log^4(\Cpp{\piref}{\pi}) +
                         1).
  \end{align}
\end{lemma}
\begin{proof}[\pfref{lem:fourth_moment}]
  Define 
\[X\ldef\abs*{\beta\log\left(\frac{\pi(y\mid{}x)}{\piref(y\mid{}x)}\right)-\beta\log\left(\frac{\pi(y'\mid{}x)}{\piref(y'\mid{}x)}\right)}.\]
Set $k = \log(2\Cpp{\piref}{\pi})$. 
We can bound
\begin{align}
\En\brk*{X^4}
&=
\En\brk*{\int_{0}^{\infty}\indic\crl*{X^4>t}dt}\\
&=
4\En\brk*{\int_{0}^{\infty}\indic\crl*{X>t}t^3dt}\\
&=
4\int_{0}^{\infty}\bbP\brk*{X>t}t^3dt\\
&\leq{} k^4 +
4\int_{k}^{\infty}\bbP\brk*{X>t}t^3dt\\
&\leq{} 
k^4 + 4\int_{k}^{\infty}e^{k-t}t^3dt\\
&=
k^4 +
4(k^3 + 3k^2 + 6k + 6) \\
&= O(k^4 + 1),
  \end{align}
  where the third-to-last line uses \cref{lem:logprob_bound}.
\end{proof}

\subsubsection{Proof of \protect\creftitle{thm:dpo_general}}
\begin{proof}[\pfref{thm:dpo_general}]
For any model $\pi\in\Pi$, define $J(\pi)\ldef{} \En_{\pi}\brk*{\log\piref(y\mid{}x)}$. Let
$\pihat\in\Pi$ denote the model returned by the DPO algorithm in \cref{eq:dpo_ml}.
Let
$\En_{\pi,\pi'}\brk*{\cdot}$ denote shorthand for
$\En_{x\sim\cdist,y\sim\pi(x),y'\sim\pi'(x)}\brk{\cdot}$, and for any $r:\cX\times\cY\to\mathbb{R}$ define
$\Delta^{r}(x,y,y'):=r(x,y) - r(x,y')$. Define
\[\rstar(x,y):=\log\piref(y\mid{}x)=\beta\log\prn*{\frac{\pistarb(y\mid{}x)}{\piref(y\mid{}x)}}+Z(x),\]
and let $\rhat(x,y) \ldef{}
\beta\log\prn*{\frac{\pihat(y\mid{}x)}{\piref(y\mid{}x)}}$. By a
standard argument \citep{huang2024correcting}, we have
\begin{align}\pihat \in \argmax_{\pi:\cX\to\Delta(\cY)} \En_{\pi}[\rhat(x,y)] - \beta \Dkl{\pi}{\piref}.\label{eq:pihat-rhat-opt}\end{align}
Therefore for any comparator model $\pistar: \cX \to \Delta(\cY)$ (not
necessarily in the model class $\Pi$), we have
\begin{align}
J(\pistar) - J(\pihat)
&= \En_{\pistar}[\rstar(x,y)] - \En_{\pihat}[\rstar(x,y)] \\ 
&= \En_{\pistar}[\rhat(x,y)] - \beta\Dkl{\pistar}{\piref} - \En_{\pihat}[\rhat(x,y)] + \beta\Dkl{\pihat}{\piref} \\
&\qquad+ \En_{\pistar}[\rstar(x,y)-\rhat(x,y)] + \beta\Dkl{\pistar}{\piref} + \En_{\pihat}[\rhat(x,y)-\rstar(x,y)] - \beta\Dkl{\pihat}{\piref} \\ 
&\leq \En_{\pistar}[\rstar(x,y)-\rhat(x,y)] + \beta\Dkl{\pistar}{\piref} + \En_{\pihat}[\rhat(x,y)-\rstar(x,y)] - \beta\Dkl{\pihat}{\piref} \\ 
&= \En_{\pistar,\piref}\brk*{\Delta^{\rstar}(x,y,y') -
  \Delta^{\rhat}(x,y,y')}
  + \En_{\pihat,\piref}\brk*{\Delta^{\rhat}(x,y,y') -
  \Delta^{\rstar}(x,y,y')}\\
  &\qquad+ \beta\Dkl{\pistar}{\piref} - \beta\Dkl{\pihat}{\piref} \label{eq:j-regret-decomp}
\end{align}
where the inequality uses \cref{eq:pihat-rhat-opt}. To bound the \rhs
above, we will use the following lemma, which is proven in the sequel.
\begin{lemma}
\label{lem:com}
For any model $\pi$ and any $\eta>0$, we have that
\begin{align}
&\En_{\pi,\piref}\brk*{\abs*{\Delta^{\rstar}(x,y,y') -
\Delta^{\rhat}(x,y,y')}}\\
&\approxleq{}           \cC_{\pi}^{1/2}\cdot\prn*{\En_{\piref,\piref}\brk*{\abs*{\Delta^{\rstar}(x,y,y') -
    \Delta^{\rhat}(x,y,y')}^2\indic\crl*{\abs[\big]{\Delta^{\rstar}}\leq\eta,
    \abs[\big]{\Delta^{\rhat}}\leq\eta}}}^{1/2}\\
&~~~~
+ \cC_{\pi}^{1/2}(\log(\Cpp{\piref}{\pihat})+\log(\Cpp{\piref}{\pistarb}))\cdot\prn*{\bbP_{\piref,\piref}\brk*{\abs[\big]{\Delta^{\rstar}}>\eta}
  +
  \bbP_{\piref,\piref}\brk*{\abs[\big]{\Delta^{\rhat}}>\eta}}^{1/4}.
\end{align}

\end{lemma}
Using \cref{lem:com} to bound the first two terms of \cref{eq:j-regret-decomp}, and using the fact that all $\pi\in\Pi$ have
$\cC_{\pi}\leq\Cconc$ and $\Cpp{\piref}{\pi}\leq\Closs$, we
have that
\begin{align}
&J(\pistar) - J(\pihat) \\
&\approxleq{} (\cC_{\pistar}+\Cconc)^{1/2}\cdot\prn*{\En_{\piref,\piref}\brk*{\abs*{\Delta^{\rstar}(x,y,y') -
    \Delta^{\rhat}(x,y,y')}^2\indic\crl*{\abs[\big]{\Delta^{\rstar}}\leq\eta,
\abs[\big]{\Delta^{\rhat}}\leq\eta}}}^{1/2}\\
&+(\cC_{\pistar}+\Cconc)^{1/2}\log(\Closs)\cdot\prn*{\bbP_{\piref,\piref}\brk*{\abs[\big]{\Delta^{\rstar}}>\eta}
+
\bbP_{\piref,\piref}\brk*{\abs[\big]{\Delta^{\rhat}}>\eta}}^{1/4}
+
\beta\Dkl{\pistar}{\piref}.
\label{eq:dpo1}
\end{align}
Let us overload notation and write $\Delta^{\pi}(x,y,y') =
\beta\log\prn*{\frac{\pi(y\mid{}x)}{\piref(y\mid{}x)}}-\beta\log\prn*{\frac{\pi(y'\mid{}x)}{\piref(y'\mid{}x)}}$,
so that $\Delta^{\pihat}=\Delta^{\rhat}$ and
$\Delta^{\pistarb}=\Delta^{\rstar}$. Since $\pistarb\in\Pi$, the
definition of $\pihat$ in \cref{eq:dpo} implies that
\begin{align}
  \sum_{(x,y,y')\in\cDpref}\prn*{\Delta^{\pihat}(x,y,y')-\Delta^{\pistarb}(x,y,y')
  }^2
  &\leq{}
  \min_{\pi\in\Pi}\sum_{(x,y,y')\in\cDpref}\prn*{\Delta^{\pi}(x,y,y')-\Delta^{\pistarb}(x,y,y')}^2\\
  &\leq{}
  \sum_{(x,y,y')\in\cDpref}\prn*{\Delta^{\pistarb}(x,y,y')-\Delta^{\pistarb}(x,y,y')}^2 \\
  &= 0. \label{eq:emp_loss}
\end{align}
\newcommand{\Brange}{B_{n,\deltafail}}
Define $\Brange := \log(2n\Closs\abs{\Pi}\deltafail^{-1})$. It is immediate that
\begin{align}
    \sum_{(x,y,y')\in\cDpref}\prn*{\Delta^{\pihat}(x,y,y')-\Delta^{\pistarb}(x,y,y')
  }^2 \indic\crl*{\abs[\big]{\Delta^{\pihat}}\leq\Brange,
  \abs[\big]{\Delta^{\pistarb}}\leq\Brange}
  \leq{} 0.
\end{align}
From here, Bernstein's inequality and a union bound implies that with probability at least $1-\deltafail$,
\begin{align}
&\En_{\piref,\piref}\brk*{\abs*{\Delta^{\pihat}(x,y,y') -
\Delta^{\pistarb}(x,y,y')}^2\indic\crl*{\abs[\big]{\Delta^{\pihat}}\leq\Brange,
\abs[\big]{\Delta^{\pistarb}}\leq\Brange}} \\
&\approxleq{} \frac{\Brange^2\log(\abs{\Pi}\deltafail^{-1})}{n}\rdef\vepsstat^2.
\end{align}
In particular, if we combine this with \cref{eq:dpo1} and set
$\eta=\Brange$, then \cref{lem:logprob_bound} implies that
\begin{align}
J(\pistar) - J(\pihat)
&\approxleq{} (\cC_{\pistar}+\Cconc)^{1/2}\cdot\vepsstat
+(\cC_{\pistar}+\Cconc)^{1/2}\log(\Closs)\cdot\deltafail^{1/4}
+
\beta\Dkl{\pistar}{\piref}.
\end{align}
Note that the above bound holds for any
$\pistar:\cX\to\Delta(\cY)$. We define $\pistar$ by
\[
  \pistar(y\mid
  x):= \frac{\piref(y\mid x) \indic[y \in
    \Ygamma(x)]}{\piref(\Ygamma(x)\mid x)},
\]
which can be seen to satisfy $\cC_{\pistar} \leq{} \Cstarg \leq
\Cconc$ and $\Dkl{\pistar}{\piref} \leq \log(\cC_{\pistar})
\leq \log(\Cconc)$. With this choice, we can further bound the
expression above by
\begin{align}
J(\pistar) - J(\pihat)
&\approxleq{} (\Cconc)^{1/2}\cdot\vepsstat
+(\Cconc)^{1/2}\log(\Closs)\cdot\deltafail^{1/4}
+
\beta\log(\Cconc)\label{eq:dpo2}
\end{align}
Given a desired failure probability $\deltafail$, applying the
bound above with
$\deltafail'\ldef{}\deltafail\wedge{}(\vepsstat/\log(\Closs))^4$
then gives
\begin{align}
J(\pistar) - J(\pihat)
&\approxleq{} (\Cconc)^{1/2}\cdot\vepsstat + \beta\log(\Cconc).
\end{align}
Finally, we observe that for our choice of $\pistar$,
under the margin condition with parameter $\gamma$, we
have
\begin{align}
J(\pistar) - J(\pihat)
&=\En_{x\sim\cdist}\En_{y,y'\sim\pistar,\pihat}\brk*{\log\prn*{\frac{\piref(y\mid{}x)}{\piref(y'\mid{}x)}}}\\
&\approxgeq{}\gammargin\cdot\En_{x\sim\cdist}\En_{y'\sim\pihat}\brk*{\indic\crl{y'\not \in \Ygamma(x)}} - \gamma\\
&\approxgeq{}\gammargin\delta\cdot\En_{x\sim\cdist}\brk*{\indic\crl{\pihat(\Ygamma(x)\mid{}x)\leq{}1-\delta}} - \gamma
\end{align}
where the first inequality uses \cref{ass:hard_margin-apx} together with the fact that $y \in \Ygamma(x)$ with probability $1$ over $x\sim \mu$ and $y \sim \pistar(\cdot\mid x)$. This proves the result.

\end{proof}

               \begin{proof}[\pfref{lem:com}]
          For any $\eta>0$, we can bound
          \begin{align}
            \En_{\pi,\piref}\brk*{\abs*{\Delta^{\rstar}(x,y,y') -
            \Delta^{\rhat}(x,y,y')}}
            &\leq{}             \En_{\pi,\piref}\brk*{\abs*{\Delta^{\rstar}(x,y,y') -
            \Delta^{\rhat}(x,y,y')}\indic\crl*{\abs[\big]{\Delta^{\rstar}}\leq\eta,
              \abs[\big]{\Delta^{\rhat}}\leq\eta}}\\
            &~~~~+  \En_{\pi,\piref}\brk*{\abs*{\Delta^{\rstar}(x,y,y') -
            \Delta^{\rhat}(x,y,y')}\indic\crl*{\abs[\big]{\Delta^{\rstar}}>\eta
                                \vee \abs[\big]{\Delta^{\rhat}}>\eta}}.
          \end{align}
          For the second term above, we can use Cauchy-Schwarz to bound
          \begin{align}
            &\En_{\pi,\piref}\brk*{\abs*{\Delta^{\rstar}(x,y,y') -
            \Delta^{\rhat}(x,y,y')}\indic\crl*{\abs[\big]{\Delta^{\rstar}}>\eta
              \vee \abs[\big]{\Delta^{\rhat}}>\eta}}\\
            &\leq{}\cC_{\pi}^{1/2}\cdot\prn*{\En_{\piref,\piref}\brk*{\abs*{\Delta^{\rstar}(x,y,y') -
            \Delta^{\rhat}(x,y,y')}^2\indic\crl*{\abs[\big]{\Delta^{\rstar}}>\eta
              \vee \abs[\big]{\Delta^{\rhat}}>\eta}}}^{1/2}\\
&\approxleq{}
              \cC_{\pi}^{1/2}\cdot\prn*{\bbP_{\piref,\piref}\brk*{\abs[\big]{\Delta^{\rstar}}>\eta}
              +
              \bbP_{\piref,\piref}\brk*{\abs[\big]{\Delta^{\rhat}}>\eta}}^{1/4}\\
            &\qquad\qquad\cdot{}\prn*{
              \En_{\piref,\piref}\brk*{\abs*{\Delta^{\rstar}(x,y,y')}^{4}}
              +               \En_{\piref,\piref}\brk*{\abs*{\Delta^{\rhat}(x,y,y')}^{4}}
                                                        }^{1/4}\\
            &\approxleq{}
              \cC_{\pi}^{1/2}\cdot\prn*{\bbP_{\piref,\piref}\brk*{\abs[\big]{\Delta^{\rstar}}>\eta}
              +
              \bbP_{\piref,\piref}\brk*{\abs[\big]{\Delta^{\rhat}}>\eta}}^{1/4}
              \cdot{}(\log(\Cpp{\piref}{\pihat})+\log(\Cpp{\piref}{\pistarb})),
          \end{align}
          where the last inequality follows from \cref{lem:fourth_moment}.

            Meanwhile, for the first term, for any $\lambda>0$ we can
            bound
            \begin{align}
              &\En_{\pi,\piref}\brk*{\abs*{\Delta^{\rstar}(x,y,y') -
                \Delta^{\rhat}(x,y,y')}\indic\crl*{\abs[\big]{\Delta^{\rstar}}\leq\eta,
                \abs[\big]{\Delta^{\rhat}}\leq\eta}}\\
              &\leq{}\cC_{\pi}^{1/2}\prn*{\En_{\piref,\piref}\brk*{\abs*{\Delta^{\rstar}(x,y,y') -
                \Delta^{\rhat}(x,y,y')}^2\indic\crl*{\abs[\big]{\Delta^{\rstar}}\leq\eta,
                \abs[\big]{\Delta^{\rhat}}\leq\eta}}}^{1/2}.
            \end{align}

        \end{proof}        

\subsection{Proof of \creftitle{thm:xpo-sharpening} and \creftitle{thm:xpo-softmax}}
\label{sec:proofs_exploration}

In this section we prove \cref{thm:xpo-sharpening} as well as
\cref{thm:xpo-softmax}, the application to linear softmax models. For
the formal theorem statements, see \cref{thm:xpo-sharpening-apx} and
\cref{thm:xpo-softmax-apx} respectively. The section is organized as
follows.
\begin{itemize}
\item \cref{sec:xpo_background} gives necessary background on
  KL-regularized policy optimization, as well as the Sequential
  Extrapolation Coefficient.
\item \cref{sec:xpo_reward} presents a generic guarantee for \xpo
  under a general choice of reward function.
\item \cref{sec:xpo_sharpening} instantiates the result above with the
  self-reward function $r(x,y) := \log \piref(y \mid x)$ to prove
  \cref{thm:xpo-sharpening}.
\item Finally, \cref{sec:xpo_linear} applies the preceding results to
  prove \cref{thm:xpo-softmax}.
\end{itemize}

\subsubsection{Background}
\label{sec:xpo_background}
To begin, we give background on KL-regularized policy optimization and the Sequential Extrapolation Coefficient. %

\paragraph{KL-regularized policy optimization} Let $\beta>0$ be given,
and let $r: \MX \times \MY \to [-\Rmax,\Rmax]$ be an unknown reward
function on prompt/action pairs. Define a value function $J_\beta$
over model class $\Pi$ by:
\begin{align}
  J_\beta(\pi) := \En_\pi[r(x,y)] - \beta \cdot
  \Dkl{\BP^\pi}{\BP^{\piref}}.
\end{align}
We refer to this as a \emph{KL-regularized policy optimization}
objective (we use the term ``policy'' following the reinforcement learning
literature; for our setting, policies correspond to models). Given query access to $r$, the goal is to find $\pihat \in\Pi$ such that 
\begin{align}
  J_\beta(\pistarb) - J_\beta(\pihat) \leq \epsilon
\end{align}
where $\pistarb(y\mid x) \propto \piref(y\mid x) \exp(\beta^{-1}
r(x,y))$ is the model that maximizes $J_\beta$ over all models $\pi: \MX \to \Delta(\MY)$.

We make use of the following assumptions, as in \iftoggle{workshop}{\cite{xie2024exploratory}}{\citet{xie2024exploratory}}.

\begin{assumption}[Realizability]\label{assumption:xpo-realizability}
It holds that $\pistarb \in \Pi$.
\end{assumption}

\begin{assumption}[Bounded density ratios]\label{assumption:xpo-bdr-apx}
For all $\pi \in \Pi$, $(x,y) \in \MX\times \MY$, $\abs[\big]{\beta \log \frac{\pi(y\mid x)}{\piref(y\mid x)}} \leq \Vmax$.
\end{assumption}
\dfc{Under realizability, we have $\Rmax\leq\Vmax$, so maybe we should
  just state the guarantee in terms of $\Vmax$ only?}\dhruv{i personally think it's a bit more opaque to have no explicit bound on reward range and derive it implicitly from realizability}

Finally, we require two definitions.

\begin{definition}[Sequential Extrapolation Coefficient for RLHF, \citep{xie2024exploratory}]\label{def:sec}
For a model class $\Pi$, reward function $r$, reference model $\piref$, and parameters $T \in \NN$ and $\beta,\lambda > 0$, the Sequential Extrapolation Coefficient is defined as
\begin{align}
&\SEC(\Pi,r,T,\beta,\lambda;\piref)  \\ 
&:= \sup_{\pi^{(1)},\dots,\pi^{(T)} \in \Pi}\left\{\sum_{t=1}^T\frac{\EE^{(t)}\left[\beta\log\frac{\pi^{(t)}(y\mid x)}{\piref(y\mid x)} - r(x,y) - \beta\log\frac{\pi^{(t)}(y'\mid x)}{\piref(y'\mid x)} + r(x,y')\right]^2}{\lambda \lor \sum_{i=1}^{t-1} \EE^{(i)}\left[\left(\beta\log\frac{\pi^{(t)}(y\mid x)}{\piref(y\mid x)} - r(x,y) - \beta\log\frac{\pi^{(t)}(y'\mid x)}{\piref(y'\mid x)} + r(x,y')\right)^2\right]}\right\}
\end{align}
where $\EE^{(t)}$ denotes expectation over $x \sim \mu$, $y \sim \pi^{(t)}(\cdot\mid x)$, and $y' \sim \piref(\cdot\mid x)$.
\end{definition}

\begin{definition}
Let $\epsilon>0$. We say that $\Psi \subseteq \Pi$ is a $\epsilon$-net for model class $\Pi$ if for every $\pi \in \Pi$ there exists $\pi' \in \Psi$ such that
\[\max_{x \in \MX} \max_{y \in \MY} \left| \log \frac{\pi(y\mid x)}{\pi'(y\mid x)}\right| \leq \epsilon.\]
We write $\cN(\Pi,\epsilon)$ to denote the size of the smallest $\epsilon$-net for $\Pi$.
\end{definition}

\subsubsection{Guarantees for KL-regularized policy optimization with \xpo}
\label{sec:xpo_reward}

\begin{algorithm}[t]
\caption{Reward-based variant of Exploratory Preference Optimization \citep{xie2024exploratory}}\label{alg:xpo}
\begin{algorithmic}
\State \textbf{input:} Base model $\piref:\MX\to\Delta(\MY)$, reward function $r:\MX\times\MY\to\RR$, number of iterations $T \in \NN$, KL regularization coefficient $\beta>0$, optimism coefficient $\alpha>0$.
\State Initialize: $\pi^{(1)} \gets \piref$, $\MD^{(0)} \gets \emptyset$.
\For{iteration $t=1,\dots,T$}
\State \textbf{Generate sample:} $(x^{(t)},y^{(t)}, \ty^{(t)})$ via $x^{(t)} \sim \mu$, $y^{(t)} \sim \pi^{(t)}(\cdot\mid x^{(t)})$, $\ty^{(t)} \sim \piref(\cdot\mid x^{(t)})$.
\State \textbf{Update dataset:} $\MD^{(t)} \gets \MD^{(t-1)} \cup \{(x^{(t)},y^{(t)},\ty^{(t)})\}$.
\State \textbf{Model optimization with global optimism:} 
\begin{align}
\pi^{(t+1)} &\gets \argmin_{\pi\in\Pi} \Bigg\{\alpha\sum_{(x,y,y') \in \MD^{(t)}} \log(\pi(y'\mid x)) \\
&- \sum_{(x,y,y') \in \MD^{(t)}} \left(\beta \log\frac{\pi(y\mid x)}{\piref(y\mid x)} - \beta \log\frac{\pi(y'\mid x)}{\piref(y'\mid x)} - (r(x,y) - r(x,y'))\right)^2\Bigg\}.
\end{align}
\EndFor
\State \textbf{return:} $\pihat \gets \argmax_{t \in [T+1]} J_\beta(\pi^{(t)})$. \Comment{Can estimate $J_\beta(\pi^{(t)})$ using validation data.}
\end{algorithmic}
\end{algorithm}

In this section, we give self-contained guarantees for the \xpo
algorithm (\cref{alg:xpo}). \xpo was introduced in
\cite{xie2024exploratory} for KL-regularized policy optimization in
the related setting where the learner only has indirect access to the
reward function $r$ through \emph{preference data} (specifically,
pairs of actions labeled via a Bradley-Terry model). Standard offline
algorithms for this problem, such as $\algofont{DPO}$, require bounds
on concentrability of the model class (see
e.g. \cref{eq:conc}). \iftoggle{workshop}{\cite{xie2024exploratory}}{\citet{xie2024exploratory}} show that the \xpo algorithm avoids this dependence, and instead requires bounded Sequential Extrapolation Coefficient. 

\cref{alg:xpo} is a variant of the \xpo algorithm which is adapted
to reward-based feedback (as opposed to preference-based feedback),
and \cref{thm:xpo-modified-apx} shows that this algorithm enjoys
guarantees similar to those of \cite{xie2024exploratory} for this
setting. Note that this is not an immediate corollary of the results in \iftoggle{workshop}{\cite{xie2024exploratory}}{\citet{xie2024exploratory}}, since the sample complexity in the preference-based setting scales with $e^{O(\Rmax)}$, and for our application to sharpening it is important to avoid this dependence. However, our algorithm and analysis only diverge from \iftoggle{workshop}{\cite{xie2024exploratory}}{\citet{xie2024exploratory}} in a few places.

\begin{theorem}[Variant of \iftoggle{workshop}{Theorem~3.1 in \cite{xie2024exploratory}}{\citet[Theorem~3.1]{xie2024exploratory}}]\label{thm:xpo-modified-apx}
Suppose that
\Cref{assumption:xpo-realizability,assumption:xpo-bdr-apx} hold. For
any $T \in \NN$, $\epdisc,\deltaf \in (0,1)$, by setting $\alpha := \frac{\beta}{\Rmax+\Vmax} \sqrt{\frac{\log(2\cN(\Pi,\epdisc)T/\deltaf)}{\SEC(\Pi) T}}$, \cref{alg:xpo} produces a model $\pihat \in \Pi$ such that with probability at least $1-\deltaf$,
\begin{align}
\beta\Dkl{\pihat}{\pistarb} = J_\beta(\pistarb) - J_\beta(\pihat) &\lesssim 
(\Rmax+\Vmax) \sqrt{\frac{\SEC(\Pi)\log(2\cN(\Pi,\epdisc)T/\deltaf)}{T}} \\ 
&+ \beta\epdisc \sqrt{\SEC(\Pi)T}
\end{align}
where $\SEC(\Pi) := \SEC(\Pi, r, T, \beta, \Vmax^2; \piref)$.
\end{theorem}
\begin{proof}[\pfref{thm:xpo-modified-apx}]
For compactness, we abbreviate $\SEC(\Pi) := \SEC(\Pi,r,T,\beta,\Vmax^2;\piref)$. From Equation~(37) of \cite{xie2024exploratory}, we have
\begin{align}
&\frac{1}{T}\sum_{t=1}^T J_\beta(\pistarb) - J_\beta(\pi^{(t)}) \\ 
&\approxleq \frac{\alpha}{\beta} (\Rmax+\Vmax)^2 \cdot \SEC(\Pi) + \frac{\beta}{\alpha T} + \frac{\Vmax}{T} + \frac{1}{T}\sum_{t=2}^T \E_{(x,y) \sim \piref}[\beta\log \pi^{(t)}(y\mid x) - \beta\log\pistarb(y\mid x)] \\ 
&+ \frac{\beta}{\alpha (\Rmax+\Vmax)^2 T} \sum_{t=2}^T \E_{\substack{x \sim \mu \\ y,y' \sim \pibar^{(t)}\mid x}}\left[\left(\beta\log\frac{\pi^{(t)}(y\mid x)}{\piref(y\mid x)} - r(x,y) - \beta \log \frac{\pi^{(t)}(y'\mid x)}{\piref(y'\mid x)} + r(x,y')\right)^2\right]
\end{align} 
where $\pibar^{(t)} := \frac{1}{t-1} \sum_{i < t} \pi^{(i)} \otimes \piref$ denotes the model that, given $x \in \MX$, samples $i \sim \Unif([t-1])$ and then samples $y \sim \pi^{(i)}(\cdot\mid x)$ and $y' \sim \piref(\cdot\mid x)$. For any $2 \leq t \leq T$, define $L^{(t)}:\Pi \to [0,\infty)$ by
\begin{align}
L^{(t)}(\pi) 
&:= \E_{(x,y) \sim \piref}[\beta\log \pi(y\mid x) - \beta\log\pistarb(y\mid x)] \\ 
&+ \frac{\beta}{\alpha (\Vmax+\Rmax)^2} \E_{\substack{x \sim \mu \\ y,y' \sim \pibar^{(t)}\mid x}}\left[\left(\beta\log\frac{\pi(y\mid x)}{\piref(y\mid x)} - r(x,y) - \beta \log \frac{\pi(y'\mid x)}{\piref(y'\mid x)} + r(x,y')\right)^2\right].
\end{align}
Similarly, define
\begin{align}
\wh L^{(t)}(\pi) 
&:= \sum_{(x,y,y') \in \MD^{(t)}}[\beta\log \pi(y'\mid x) - \beta\log\pistarb(y'\mid x)] \\ 
&+ \frac{\beta}{\alpha (\Vmax+\Rmax)^2} \sum_{(x,y,y') \in \MD^{(t)}}\left[\left(\beta\log\frac{\pi(y\mid x)}{\piref(y\mid x)} - r(x,y) - \beta \log \frac{\pi(y'\mid x)}{\piref(y'\mid x)} + r(x,y')\right)^2\right]
\end{align}
where $\MD^{(t)}$ is the dataset defined in iteration $t$ of \cref{alg:xpo}. By \cref{assumption:xpo-realizability} we have $\pistarb \in \Pi$, so $\inf_{\pi\in\Pi} \wh L^{(t)}(\pi) \leq 0$. Moreover by definition, $\pi^{(t)} \in \argmin_{\pi\in\Pi} \wh L^{(t)}$. 

Let $\Psi$ be an $\epdisc$-net over $\Pi$, of size $\cN(\Pi,\epdisc)$. Fix any $\pi \in \Psi$ and $2 \leq t \leq T$, and define increments $X_i := \wh L^{(i)}(\pi) - \wh L^{(i-1)}(\pi)$ for $2 \leq i \leq t$, with the notation $\wh L^{(1)}(\pi) := 0$ so that $\wh L^{(t)}(\pi) = \sum_{i=2}^t X_i$. Let $\MF_i$ be the filtration induced by $\MD^{(i)}$ and define $\gamma_i := \EE[X_i\mid \MF_{i-1}]$. Observe that $(t-1)L^{(t)}(\pi) = \sum_{i=2}^t \gamma_i$. For any $i$, note that we can write $X_i = Y_i + Z_i$ where $Y_i \in [-\Vmax,\Vmax]$ and $Z_i \in [0,\beta/\alpha]$. By \cref{cor:filtration-tail}, it holds with probability at least $1-\deltaf/(2|\Pi|T)$
\[\sum_{i=2}^t \EE[Z_i\mid \MF_{i-1}] \lesssim \frac{\beta}{\alpha} \log(2|\Psi|T/\deltaf) + \sum_{i=2}^t Z_i.\]
By Azuma-Hoeffding, it holds with probability at least $1-\deltaf/(2|\Pi|T)$ that
\[\sum_{i=2}^t \EE[Y_i\mid \MF_{i-1}] \lesssim \Vmax\sqrt{T\log(2|\Psi|T/\deltaf)} + \sum_{i=2}^t Y_i.\]
Hence, with probability at least $1-\deltaf/(|\Psi|T)$ we have
\[(t-1)L^{(t)}(\pi) \lesssim \frac{\beta}{\alpha} \log(2|\Psi|T/\deltaf) + \Vmax\sqrt{T\log(2|\Psi|T/\deltaf)} + \wh L^{(t)}(\pi).\]
With probability at least $1-\deltaf$ this bound holds for all $\pi \in\Psi$ and $2 \leq t \leq T$. Henceforth condition on this event. Fix any $\pi \in \Pi$ and $2 \leq t \leq T$. Since $\Psi$ is an $\epsilon$-net for $\Pi$, we see by definition of $L^{(t)}$ that there is some $\pi' \in \Psi$ such that 
\[|L^{(t)}(\pi) - L^{(t)}(\pi')| \lesssim \beta \epdisc + \frac{\beta}{\alpha(\Vmax+\Rmax)^2} \cdot \beta\epdisc (\Vmax+\Rmax) \leq \beta\epdisc\left(1 + \frac{\beta}{\alpha(\Vmax+\Rmax)}\right)\]
and similarly
\[|\wh L^{(t)}(\pi) - \wh L^{(t)}(\pi')| \lesssim (t-1)\beta\epdisc\left(1 + \frac{\beta}{\alpha(\Vmax+\Rmax)}\right).\]

It follows that, for all $2 \leq t \leq T$, since $\wh L^{(t)}(\pi^{(t)}) \leq 0$, we get
\[(t-1)L^{(t)}(\pi^{(t)}) \lesssim \frac{\beta}{\alpha} \log(2|\Psi|T/\deltaf) + \Vmax\sqrt{T\log(2|\Psi|T/\deltaf)} + \beta\epdisc T\left(1 + \frac{\beta}{\alpha(\Vmax+\Rmax)}\right).\]
Hence,
\begin{align}
&\frac{1}{T}\sum_{t=1}^T J_\beta(\pistarb) - J_\beta(\pi^{(t)}) \\ 
&\approxleq \frac{\alpha}{\beta} (\Rmax+\Vmax)^2 \cdot \SEC(\Pi) + \frac{\beta}{\alpha T} + \frac{\Vmax}{T} + \frac{1}{T}\sum_{t=2}^T L^{(t)}(\pi^{(t)}) \\ 
&\approxleq (\Rmax+\Vmax) \sqrt{\frac{\SEC(\Pi)\log(2|\Psi|T/\deltaf)}{T}} + \beta\epdisc \sqrt{\SEC(\Pi)T}
\end{align}
by taking
\[\alpha := \frac{\beta}{\Rmax+\Vmax} \sqrt{\frac{\log(2|\Psi|T/\deltaf)}{\SEC(\Pi) T}}.\]
Since the output $\pihat$ of \cref{alg:xpo} satisfies $\pihat \in \argmax_{t \in [T]} J_\beta(\pi^{(t)})$, the claimed bound on $J_\beta(\pistarb) - J_\beta(\pihat)$ is immediate. Finally, observe that by definition of $\pistarb$,
\begin{align}
J_\beta(\pistarb) - J_\beta(\pihat)
&= \E_{(x,y) \sim \pistarb}\left[r(x,y) - \beta \log \frac{\pistarb(y\mid x)}{\piref(y\mid x)}\right] - \E_{(x,y) \sim \pihat}\left[r(x,y) - \beta \log \frac{\pihat(y\mid x)}{\piref(y\mid x)}\right] \\
&= \E_{(x,y) \sim \pistarb}\left[r(x,y) - \beta \log \frac{\pistarb(y\mid x)}{\piref(y\mid x)}\right] - \E_{(x,y) \sim \pihat}\left[r(x,y) - \beta \log \frac{\pistarb(y\mid x)}{\piref(y\mid x)}\right] \\
&\qquad+ \E_{(x,y) \sim \pihat}\left[\beta \log \frac{\pihat(y\mid x)}{\pistarb(y\mid x)}\right] \\
&= \beta \log \E_{(x,y) \sim \piref} [\exp(r(x,y))] - \beta \log \E_{(x,y) \sim \piref} [\exp(r(x,y))] + \beta\Dkl{\pihat}{\pistarb} \\ 
&= \beta\Dkl{\pihat}{\pistarb}.
\end{align}
This completes the proof.
\end{proof}

\subsubsection{Applying \xpo to \mlsharp}
\label{sec:xpo_sharpening}

We now prove \cref{thm:xpo-sharpening-apx}, the formal statement of
\cref{thm:xpo-sharpening}, which applies \xpo to \mlsharp. This result
is a straightforward corollary of \cref{thm:xpo-modified-apx} with the reward function $\rself(x,y) := \log \piref(y \mid x)$, together with the observation that low KL-regularized regret implies sharpness (under \cref{ass:hard_margin}).

  \begin{theorem}[Sharpening via active exploration]
    \label{thm:xpo-sharpening-apx}
There are absolute constants $c_{\mathrm{\ref{thm:xpo-sharpening-apx}}},
C_{\mathrm{\ref{thm:xpo-sharpening-apx}}}>0$ so that the following
holds. Let $\epsilon,\delta,\gammargin,\deltafail,\beta \in (0,1)$ and
$T \in \NN$ be given. For base model $\piref$, define reward function $r(x,y) := \log \piref(y\mid x)$. Let $\Rmax \geq 1+\max_{x,y}\log\frac{1}{\piref(y\mid x)}$. Suppose that $\piref$ satisfies \cref{ass:hard_margin} with parameter $\gammargin$, that $\beta^{-1} \geq 2\gammargin^{-1}\log(2|\MY|/\delta)$, and that there is $\epdisc \in (0,1)$ so that 
\[T \geq C_{\mathrm{\ref{thm:xpo-sharpening-apx}}} \frac{\Rmax^2 \SEC(\Pi) \log(2\cN(\Pi,\epdisc)T/\deltafail)}{\epsilon^2\delta^2 \beta^2}\]
and 
\[\epdisc \leq c_{\mathrm{\ref{thm:xpo-sharpening-apx}}} \frac{\epsilon\delta}{\sqrt{\SEC(\Pi) T}}\]
where $\SEC(\Pi) := \SEC(\Pi,r,T,\beta,\Rmax^2;\piref)$. Also suppose that $\pistarb \in \Pi$ where $\pistarb(y\mid x) \propto \piref^{1+\beta^{-1}} (y\mid x)$.

Then applying \cref{alg:xpo} with base model $\piref$, reward function $r$, iteration count $T$, regularization $\beta$, and optimism parameter $\alpha := \frac{\beta}{\Rmax} \sqrt{\frac{\log(2\cN(\Pi,\epdisc)T/\delta)}{\SEC(\Pi) T}}$ yields a model $\pihat \in \Pi$ such that with probability at least $1-\deltafail$,
\[\BP_{x \sim \mu}[\pihat(\Ystar(x)\mid x) < 1-\delta] \leq
  \epsilon.\]
The total sample complexity is
\[
  m = \bigoht\prn*{
    \frac{\Rmax^2 \SEC(\Pi) \log(\cN(\Pi,\epdisc)/\deltafail)\log^2(\abs{\cY}\delta^{-1})}{ \gammargin^2\epsilon^2\delta^2}
    }.
  \]
  \end{theorem}

\begin{proof}[\pfref{thm:xpo-sharpening-apx}]
By definition of $r$, we have $|r(x,y)| \leq \Rmax$ for all $x,y$. By assumption, \cref{assumption:xpo-realizability} is satisfied, and by definition of $\Rmax$, \cref{assumption:xpo-bdr} is satisfied with parameter $\Vmax := \beta \Rmax \leq \Rmax$. It follows from \cref{thm:xpo-modified-apx} that with probability at least $1-\deltafail$, the output $\pihat$ of \cref{alg:xpo} satisfies
\begin{align} 
\beta\Dkl{\pihat}{\pistarb} &\lesssim (\Rmax+\Vmax)\sqrt{\frac{\SEC(\Pi) \log(2\cN(\Pi,\epdisc)T/\deltafail)}{T}} \\
&\qquad+ \beta\epdisc \sqrt{\SEC(\Pi) T}.
\end{align}
By choice of $T$ and $\epdisc$, so long as $C_{\mathrm{\ref{thm:xpo-sharpening-apx}}}>0$ is chosen to be a sufficiently large constant and $c_{\mathrm{\ref{thm:xpo-sharpening-apx}}}>0$ is chosen to be a sufficiently small constant, we have $\beta\Dkl{\pihat}{\pistarb} \leq \frac{1}{12}\beta\epsilon\delta$, so by e.g. Equation~(16) of \cite{sason2016f}, $\Dhels{\pihat}{\pistarb} \leq \epsilon\delta/(12)$.

For any $x \in \MX$ and $y' \in \MY \setminus \Ystar(x)$, by \cref{ass:hard_margin} and definition of $\pistarb$ we have
\begin{align}
\frac{1}{\pistarb(y'\mid x)} \geq \frac{\max_{y \in \MY} \pistarb(y\mid x)}{\pistarb(y'\mid x)} &= \left(\frac{\max_{y \in \MY} \piref(y\mid x)}{\piref(y'\mid x)}\right)^{1+\beta^{-1}} \\
&\geq (1+\gammargin)^{1+\beta^{-1}} \geq e^{\gammargin/(2\beta)} \geq \frac{2|\MY|}{\delta}
\end{align}
where the final inequality is by the assumption on $\beta$ in the
theorem statement. Therefore
\[\pistarb(\Ystar(x)\mid x) \geq 1 - \sum_{y' \in \MY \setminus \Ystar(x)} \pistarb(y'\mid x) \geq 1-\frac{\delta}{2}.\]
Now for any $x$, we can lower bound
\begin{align}
\Dhels{\pihat(\cdot\mid x)}{\pistarb(\cdot\mid x)}
&\geq \left(\sqrt{1 - \pihat(\Ystar(x)\mid x)} - \sqrt{1 - \pistarb(\Ystar(x)\mid x)}\right)^2 \\ 
&\geq \frac{\delta}{12} \cdot \indic\{\pihat(\ystar(x)\mid x) \leq 1-\delta\}.
\end{align}

Hence,
\begin{align}
\BP_{x\sim\mu}[\pihat(\Ystar(x)\mid x)<1-\delta]
&\leq \frac{12}{\delta} \EE_{x \sim \mu} \Dhels{\pihat(\cdot\mid x)}{\pistarb(\cdot\mid x)} \\ 
&= \frac{12}{\delta} \Dhels{\pihat}{\pistarb} \\ 
&\leq \epsilon.
\end{align}
as claimed.
\end{proof}

\subsubsection{Application: linear softmax models}
\label{sec:xpo_linear}

In this section we apply \cref{thm:xpo-sharpening} to the class of linear softmax models, proving \cref{thm:xpo-softmax}. This demonstrates that \cref{alg:xpo} can achieve an exponential improvement in sample complexity compared to \bestofnalg.

\begin{definition}[Linear softmax model]
Let $d \in \NN$ be given, and let $\phi: \MX \times \MY \to \RR^d$ be
a feature map with $\norm{\phi(x,y)}_2 \leq 1$ for all $x,y$. Let
$\piz: \MX \to \Delta(\MY)$ be the uniform model $\piz(y\mid x) :=
\frac{1}{|\MY|}$, and let $B \geq 1$.\footnote{We use the notation
  $\piz$ to highlight the fact that $\piz=\pi_\theta$ for $\theta=0$.}
  We consider the linear softmax model class $\Pi_{\phi,B} := \{\pi_\theta: \theta \in \RR^d, \norm{\theta}_2 \leq B\}$ where $\pi_\theta:\MX \to \Delta(\MY)$ is defined by
\[\pi_\theta(y\mid x) \propto \piz(y\mid x) \exp(\langle \phi(x,y), \theta\rangle).\]
\end{definition}

\begin{theorem}[Restatement of \cref{thm:xpo-softmax}]\label{thm:xpo-softmax-apx}
Let $\epsilon,\delta,\gammargin,\deltafail \in (0,1)$ be
given. Suppose that $\piref = \pi_{\theta^\star} \in \Pi_{\phi,B}$ for
some $\theta^\star \in \RR^d$ with $\norm{\theta^\star}_2 \leq
\frac{\gammargin B}{3\log(2|\MY|/\delta)}$. Also, suppose that
$\piref$ satisfies \cref{ass:hard_margin} with parameter
$\gammargin$. Then \cref{alg:xpo} with base model $\piref$, reward
function $r(x,y) := \log \piref(x,y)$, regularization parameter $\beta
:= \gammargin/(2\log(2|\MY|/\delta))$, and optimism parameter
$\alpha(T) \propto
\frac{\beta}{B+\log(|\MY|)}\sqrt{\frac{d\log(BdT/(\epsilon\delta))+\log(T/\deltafail)}{dT\log(T)}}$
returns an $(\epsilon,\delta)$-sharpened model with probability at
least $1-\rho$, and has sample complexity 
\[m = \poly(\epsilon^{-1},\delta^{-1},\gammargin^{-1}, d, B, \log(|\MY|/\deltafail)).\]
\end{theorem}

Before proving the result, we unpack the
conditions. \cref{thm:xpo-softmax-apx} requires the base model
$\piref$ to lie in the model class and also satisfy the margin
condition (\cref{ass:hard_margin}). For any constant
$\epsilon,\delta>0$, the sharpening algorithm then succeeds with
sample complexity $\poly(d, \gammargin^{-1}, B, \log(|\MY|))$. These
conditions are non-vacuous; in fact, there are fairly natural examples
for which
non-exploratory algorithm such as \bestofnalg require sample
complexity $\exp(\Omega(d))$, whereas all of the above parameters are
$\poly(d)$. The following is one such example.

\begin{example}[Separation between \rlhfalg and \sftalg]\label{ex:softmax}
Set $\MX = \{x\}$ and let $\MY\subset\bbR^{d}$ be a $1/4$-packing of
the unit sphere in $\RR^d$ of cardinality $\exp(\Theta(d))$. Define
$\phi: \MX \times \MY \to \RR^d$ by $\phi(x,y) := y$, and let $B =
Cd\log d$ for an absolute constant $C>0$. Fix any $y^\star \in \MY$ and define $\piref := \pi_{\theta^\star} \in \Pi_{\phi,B}$ by $\theta^\star := y^\star$. Then for any $y \neq y^\star$, we have $\langle y, y^\star\rangle \leq 1 - \Omega(1)$, so 
\[\frac{\piref(y^\star\mid x)}{\piref(y\mid x)} = \exp(\langle y^\star - y, y^\star\rangle) = \exp(\Omega(1)) = 1+\Omega(1).\]
Thus, $\piref$ satisfies \cref{ass:hard_margin} with $\gammargin =
\Omega(1)$. Moreover, $\norm{\theta^\star}_2 = 1 \leq \frac{\gammargin
  B}{3\log(2|\MY|/\delta)}$ for any $\delta = 1/\poly(d)$, so long as
$C$ is a sufficiently large constant. It follows from
\cref{thm:xpo-softmax} that \cref{alg:xpo} computes an
$(\epsilon,\delta)$-sharpened model with sample complexity
$\poly(\epsilon^{-1},\delta^{-1},d)$. However, since
$\piref(y^\star\mid x) \leq \piref(y\mid x) \cdot \exp(2)$ for all $y
\in \MY$, it is clear that 
\[\Cstar = \EE\brk*{\frac{1}{\piref(\Ystar(x)\mid x)}} = \frac{1}{\piref(y^\star\mid x)} = \Omega(|\MY|) = \exp(\Omega(d)).\]
Thus, the sample complexity guarantee for \bestofnalg in
\cref{thm:bestofn} will incur \emph{exponential} dependence on $d$ in
the sample complexity. It is straightforward to check that this
dependence is real for \bestofnalg, and not just an artifact of the analysis, since the model that \bestofnalg is trying to learn (via MLE) will itself not be sharp in this example, unless $\exp(\Omega(d))$ samples are drawn per prompt.
\end{example}

We now proceed to the proof of \cref{thm:xpo-softmax-apx}, which requires the following bounds on the covering number and the Sequential Extrapolation Coefficient of $\Pi_{\phi,B}$.

\begin{lemma}\label{lemma:softmax-net}
Let $\epdisc>0$. Then $\Pi_{\phi,B}$ has an $\epdisc$-net of size $(6B/\epdisc)^d$.
\end{lemma}

\begin{proof}[\pfref{lemma:softmax-net}]
By a standard packing argument, there is a set $\{\theta_1,\dots,\theta_N\}$ of size $(6B/\epdisc)^d$ such that for every $\theta \in \RR^d$ with $\norm{\theta}_2 \leq B$ there is some $i \in [N]$ with $\norm{\theta_i - \theta}_2 \leq \epdisc/2$. Now for any $x\in\MX$ and $y \in \MY$,
\begin{align}
\log \frac{\pi_{\theta}(y\mid x)}{\pi_{\theta_i}(y\mid x)} 
&= \log \frac{\exp(\langle \phi(x,y),\theta\rangle)}{\exp(\langle \phi(x,y),\theta_i\rangle)} + \log \frac{\E_{(x',y')\sim \piz} \exp(\langle \phi(x',y'), \theta_i\rangle)}{\E_{(x',y')\sim \piz} \exp(\langle \phi(x',y'), \theta\rangle)} \\ 
&= \langle \phi(x,y),\theta-\theta_i\rangle + \log \frac{\E_{(x',y')\sim \piz} \left[\exp(\langle \phi(x',y'), \theta\rangle)\exp(\langle \phi(x',y'), \theta_i-\theta\rangle)\right]}{\E_{(x',y')\sim \piz} \exp(\langle \phi(x',y'), \theta\rangle)}.
\end{align}
The first term is bounded by $\epdisc/2$ in magnitude. In the second term, we have $\exp(\langle \phi(x',y'),\theta_i-\theta\rangle) \in [\exp(-\epdisc/2),\exp(\epdisc/2)]$, so the ratio of expectations lies in $[\exp(-\epdisc/2),\exp(\epdisc/2)]$ as well, and so the log-ratio lies in $[-\epdisc/2,\epdisc/2]$. In all, we get $\left|\log \frac{\pi_{\theta}(y|x)}{\pi_{\theta_i}(y\mid x)}\right| \leq \epdisc$. Thus, $\{\pi_{\theta_1},\dots,\pi_{\theta_N}\}$ is an $\epdisc$-net for $\Pi$.
\end{proof}

\begin{lemma}\label{lemma:sec-softmax-bound}
Let $r:\MX\times\MY \to [-\Rmax,\Rmax]$ be a reward function and let $T \in \NN$ and $\beta>0$. If $\lambda \geq 4\beta^2 B^2 + \Rmax^2$ then for any $\pistar \in \Pi_{\phi,B}$,
\[\SEC(\Pi_{\phi,B}, r,T, \beta,\lambda; \pistar) \lesssim d\log(T+1).\]
\end{lemma}

\begin{proof}[\pfref{lemma:sec-softmax-bound}]
Fix $\pi^{(1)},\dots,\pi^{(T)} \in \Pi_{\phi,B}.$ By definition, there are some $\theta^{(1)},\dots,\theta^{(T)} \in \RR^d$ with $\norm{\theta^{(t)}}_2 \leq B$ and \[\pi^{(t)}(y\mid x)\propto \piz(y\mid x)\exp(\langle \phi(x,y),\theta^{(t)}\rangle)\] for all $t \in [T]$ and $(x,y)\in\MX\times\MY$. Similarly, there is some $\theta^\star \in \RR^d$ with $\norm{\theta^\star}_2 \leq B$ and $\pistar(y\mid x) \propto \piz(y\mid x)\exp(\langle \phi(x,y),\theta^\star\rangle)$. 

Define $\phit:\MX\times\MY \to \RR^{d+1}$ by $\phit(x,y) := [\phi(x,y), \frac{r(x,y)}{\Rmax}]$ and define $\thetat^{(t)} := [\beta(\theta^{(t)} - \theta^\star), -\Rmax]$. Then for any $t \in [T]$ we have
\begin{align}
&\frac{\EE^{(t)}\left[\beta\log\frac{\pi^{(t)}(y\mid x)}{\pistar(y\mid x)} - r(x,y) - \beta\log\frac{\pi^{(t)}(y'\mid x)}{\pistar(y'\mid x)} + r(x,y')\right]^2}{\lambda \lor \sum_{i=1}^{t-1} \EE^{(i)}\left[\left(\beta\log\frac{\pi^{(t)}(y\mid x)}{\pistar(y\mid x)} - r(x,y) - \beta\log\frac{\pi^{(t)}(y'\mid x)}{\pistar(y'\mid x)} + r(x,y')\right)^2\right]} \\ 
&= \frac{\EE^{(t)}\left[\langle \phit(x,y) - \phit(x,y'), \thetat^{(t)}\rangle\right]^2}{\lambda \lor \sum_{i=1}^{t-1} \EE^{(i)}\left[\left(\langle \phit(x,y) - \phit(x,y'), \thetat^{(t)}\rangle\right)^2\right]} \\ 
&\leq \frac{(\thetat^{(t)})^\top \Sigma^{(t)} \thetat^{(t)}}{\lambda \lor \sum_{i=1}^{t-1} (\thetat^{(t)})^\top \Sigma^{(i)} \thetat^{(t)}}
\end{align}
where for each $i \in [T]$ we have defined $\Sigma^{(i)} := \EE^{(i)}\brk*{ (\phit(x,y) - \phit(x,y'))(\phit(x,y) - \phit(x,y'))^\top}$. Observe that $\norm{\thetat^{(t)}}_2^2 \leq 4\beta^2 B^2 + \Rmax^2 \leq \lambda$ by assumption on $\lambda$. Therefore,
\begin{align}
\frac{(\thetat^{(t)})^\top \Sigma^{(t)} \thetat^{(t)}}{\lambda \lor \sum_{i=1}^{t-1} (\thetat^{(t)})^\top \Sigma^{(i)} \thetat^{(t)}}
&\lesssim \frac{(\thetat^{(t)})^\top \Sigma^{(t)} \thetat^{(t)}}{\lambda + \sum_{i=1}^{t-1} (\thetat^{(t)})^\top \Sigma^{(i)} \thetat^{(t)}} \\ 
&\leq \frac{(\thetat^{(t)})^\top \Sigma^{(t)} \thetat^{(t)}}{(\thetat^{(t)})^\top \left(I_d + \sum_{i=1}^{t-1}\Sigma^{(i)}\right) \thetat^{(t)}} \\ 
&\leq \lambda_{\max}\left(\left(I_d + \sum_{i=1}^{t-1}\Sigma^{(i)}\right)^{-1/2} \Sigma^{(t)} \left(I_d + \sum_{i=1}^{t-1}\Sigma^{(i)}\right)^{-1/2}\right) \\ 
&\leq \Tr \left(\left(I_d + \sum_{i=1}^{t-1}\Sigma^{(i)}\right)^{-1/2} \Sigma^{(t)} \left(I_d + \sum_{i=1}^{t-1}\Sigma^{(i)}\right)^{-1/2}\right) \\ 
&= \Tr\left(\left(I_d + \sum_{i=1}^{t-1}\Sigma^{(i)}\right)^{-1}\Sigma^{(t)}\right).
\end{align}
Observe that $\Tr(\Sigma^{(t)}) \leq \max_{x,y} \norm{\phit(x,y)}_2^2 \lesssim 1$. Hence by \cref{lemma:elliptic}, we have
\begin{align}
&\sum_{t=1}^T\frac{\EE^{(t)}\left[\beta\log\frac{\pi^{(t)}(y\mid x)}{\pistar(y\mid x)} - r(x,y) - \beta\log\frac{\pi^{(t)}(y'\mid x)}{\pistar(y'\mid x)} + r(x,y')\right]^2}{\lambda \lor \sum_{i=1}^{t-1} \EE^{(i)}\left[\left(\beta\log\frac{\pi^{(t)}(y\mid x)}{\pistar(y\mid x)} - r(x,y) - \beta\log\frac{\pi^{(t)}(y'\mid x)}{\pistar(y'\mid x)} + r(x,y')\right)^2\right]} \\ 
&\lesssim \sum_{t=1}^T \Tr\left(\left(I_d + \sum_{i=1}^{t-1}\Sigma^{(i)}\right)^{-1}\Sigma^{(t)}\right) \\ 
&\lesssim d\log(T+1).
\end{align}
Since $\pi^{(1)},\dots,\pi^{(T)} \in \Pi$ were arbitrary, this completes the proof.
\end{proof}

The proof is now immediate from \cref{thm:xpo-sharpening-apx} and the above lemmas.

\begin{proof}[Proof of \cref{thm:xpo-softmax-apx}]
By the assumption on $\theta^\star$ and choice of $\beta$, the model $\pistarb$ defined by $\pistarb(y\mid x) \propto \piref(y\mid x)^{1+\beta^{-1}}$ satisfies $\pistarb = \pi_{(1+\beta^{-1})\theta^\star} \in \Pi_{\phi,B}$. By \cref{lemma:softmax-net}, we have $\cN(\Pi_{\phi,B},\epdisc) \leq (6B/\epdisc)^d$. Take $\Rmax := \sqrt{4\beta^2 B^2 + (2B+\log |\MY|)^2}$. We know that $r(x,y) := \log\piref(y\mid x)$ satisfies $|r(x,y)| \leq 2B+\log|\MY|$ for all $x,y$. By \cref{lemma:sec-softmax-bound}, we therefore get that $\SEC(\Pi_{\phi,B},r,T,\beta,\Rmax^2;\piref) \lesssim d\log(T+1)$. Substituting these bounds into \cref{thm:xpo-sharpening-apx} yields the claimed result.
\end{proof}

\end{document}